%% file: main.tex
\documentclass{article}

 \usepackage[nonatbib,preprint]{neurips_2020}

\usepackage[utf8]{inputenc} 
\usepackage[T1]{fontenc}    
\usepackage{hyperref}       
\usepackage{url}            
\usepackage{booktabs}       
\usepackage{amsfonts}       
\usepackage{nicefrac}       
\usepackage{microtype}      

\usepackage{subcaption}
\usepackage{pgfplots} 
\pgfplotsset{compat=1.14} 
\usepgfplotslibrary{groupplots}
\usepgfplotslibrary{external}
\newcommand{\quotes}[1]{``#1''}
\usepackage{amsthm}
\newtheorem{lemma}{Lemma}
\usepackage{floatrow}
\usepackage{amsmath,amsfonts,bm}

\def\va{{\bm{a}}}
\def\vb{{\bm{b}}}

\def\ve{{\bm{e}}}

\def\vk{{\bm{k}}}
\def\vl{{\bm{l}}}

\def\vo{{\bm{o}}}

\def\vq{{\bm{q}}}

\def\vv{{\bm{v}}}

\def\vx{{\bm{x}}}

\def\mI{{\bm{I}}}

\def\mW{{\bm{W}}}

\newcommand{\softmax}{\mathrm{softmax}}
\newcommand{\normalize}{\mathrm{normalize}}

\def\1{\bm{1}}
\newcommand{\R}{\mathbb{R}}

\title{Normalized Attention Without Probability Cage}

\author{%
  Oliver Richter and Roger Wattenhofer\\
  Department of Electrical Engineering and Information Technology\\
  ETH Zurich, Switzerland \\
  \texttt{\{richtero,wattenhofer\}@ethz.ch} \\
}

\begin{document}

\maketitle

\begin{abstract}
  Attention architectures are widely used; they recently gained renewed popularity with Transformers
  yielding a streak of state of the art results. Yet, the geometrical
  implications of $\softmax$-attention remain largely unexplored. 
  In this work we highlight the limitations of constraining attention weights to the probability simplex and the resulting convex hull of value vectors.
  We show that Transformers are sequence length dependent biased towards token isolation at initialization and contrast Transformers to simple max- and sum-pooling -- two strong baselines rarely reported. We 
  propose to replace the $\softmax$ in self-attention with normalization, yielding a hyperparameter and data-bias robust, generally applicable architecture.
  We support our insights with empirical results from more than 25,000 trained models. All results and implementations are made available.\footnote{\url{https://github.com/OliverRichter/normalized-attention}}
\end{abstract}

\section{Introduction}
The concept of neural attention~\cite{Attention,AttentionNLP} has sparked a number of architectural breakthroughs. The Transformer architecture~\cite{Transformer} successfully deploys multi-headed self-attention in several consecutive layers for natural language processing (NLP) -- an architecture choice that has become popular~\cite{Transformer,GPT,GPT2,BERT,XLNet,T5,RoBERTa,ELECTRA}.
Apart from NLP, self-attention has shown success in applications ranging from image classification~\cite{AttentionCNN} to generative adversarial networks~\cite{self-att-GAN} to reinforcement learning~\cite{AttentiveMultitask,TransformerRL}.
The attention architecture choice is thereby often based on one, if not both, of the following arguments:
(1) Attention helps with credit assignment by providing more direct, dynamic links between inputs and outputs. (2) Attention is directly interpretable as one can investigate the percentages to which different inputs are \quotes{attended} to.
However, this second argument has been challenged recently, as several works show that attention weights do not directly correlate with predictions~\cite{AttentionIsNotExplanation,AttentionIsNotNotExplanation,OnIdentifiabilityInTransformers,BERTsStory} in NLP models.
With interpretability in dispute, we are left with an open question: 
Can we improve the credit assignment ability by removing the constraint on attention weights to represent a 
distribution?

In this work, we show the theoretical implications of constraining the attention weights to the probability simplex, and propose an unconstrained alternative based on normalization. We show that the popular Transformer architecture has an innate bias towards token isolation at initialization and showcase implications thereof on biases in the data. Our experimental results demonstrate the advantage of unconstrained attention. In particular, we improve robustness to hyperparameters and show the general applicability of attention based architectures as compared to other architectures such as sum and max pooling.
To summarize, our contributions include:
\begin{itemize}
    \item a theoretical investigation of the probability simplex constraint in self-attention
    \item a robust, general purpose alternative based on normalization
    \item a large scale experimental comparison of the performance implications that an architecture choice entails 
    with respect to the task type, hyperparameters as well as biases in the data
\end{itemize}

\section{Background and Related Work}
\label{sec:background}

Many data processing tasks can be addressed by representing the input as a set or sequence of discrete tokens, e.g., the words in a sentence or the frames in a video. As a general formulation, we represent each input token through a vector $\vx^i\in\R^d$ for $i\in\{1,\dots, N\}$, where $N$ is the sequence length and $d$ is the dimensionality of each token.
For ease of notation we use the word \quotes{sequence} throughout, but note that all architectures discussed are also applicable to unordered sequences, i.e., sets of tokens.
Multi-headed dot-product self-attention is a fundamental building block of the Transformer architecture~\cite{Transformer}. It allows for information exchange between different tokens of the input sequence. More formally, for each attention head $m$ the input vectors $\vx^i$ are projected through an affine transformation to a query $\vq_m^i$, key $\vk_m^i$ and value vector $\vv_m^i$. The dimensionality of these vectors is chosen as $d_h = \frac{d}{M}$, where $M$ is the number of attention heads. The query and key vectors are used for a pairwise dot product, scaled by the square root of the head dimension $d_h$, to form the attention logits $l_m^{i,j}$ and attention vectors $\va_m^{i}$ as
\[l_m^{i,j} = \frac{<\vq_m^i, \vk_m^j>}{\sqrt{d_h}}\qquad \va_m^{i} = \softmax([l_m^{i,1},\dots,l_m^{i,N}])\]
where $\softmax$ refers to the normalized exponential function
$\softmax(\vx)^j = \frac{\exp(x^{j})}{\sum_k \exp(x^{k})}$
commonly used to project vectors to the probability simplex
$\mathcal{S}_P =\{\va_m^i|a_m^{i,j}\geq 0\text{ } \forall j\text{ and } \sum_j a_m^{i,j}=1\}$.
The output $\vo_m^i$ of each attention head $m$ is then given by a weighted sum of all value vectors $\vo_m^i=\sum_{j}a_m^{i,j}\cdot\vv_m^j$. These attention head outputs are concatenated and mixed trough an additional affine transformation to form the attention layer output in the Transformer architecture~\cite{Transformer}.

In this work, we investigate whether constraining the attention vectors $\va_m^i$ into the probability simplex 
through the $\softmax$ function is the best we can do. We contrast the multi-head self-attention architecture to attention-inspired architectures without $\softmax$ (discussed in Section~\ref{sec:norm_att}) as well as simpler aggregation methods commonly used. Specifically, while~Yun~et~al.~\cite{TransformerUniversalFunctionApproximator} show that Transformers are universal sequence-to-sequence function approximators, we question the practical necessity of an attention architecture, when sum pooling~\cite{DeepSets} already provides general function approximation capabilities~\cite{DeepSets,DeepSetsGNN,UniversalEquivariantSetNetworks}. Further, we compare to max pooling, a common aggregator choice that has shown good empirical success~\cite{MaxPoolingCNN,DeepSets,MaxPoolingGraphAlgorithms}.
Several recent works have proposed architectural changes to the Transformer~\cite{TransformerUniversalFunctionApproximator,PayLessAttention,ALBERT,UniversalTransformers,LayerDrop,LiteTransformer,LayerReorder,ReZero}. However, to the best of our knowledge, we are the first to explicitly question the $\softmax$ in self-attention. 

\section{Limitations and Implications of Softmax Attention}
\label{sec:limitations}

To start our discussion, we highlight an observation that follows directly from attention vectors $\va^i$ being constrained to the probability simplex $\mathcal{S}_P$:

\textbf{Attention head outputs $\vo^i_m$ are convex combinations of value vectors $\vv^i_m$}\\
This in itself has drastic implications. First and foremost, we note that a convex combination of vectors $\vv^i_m$ cannot yield any vector outside the convex hull spanned by the value vectors $\vv^i_m$. An illustration of this output cage is given in Figure~\ref{fig:convex_hull}~(left).
\begin{figure}
    \centering
    \includegraphics[width=2.9cm]{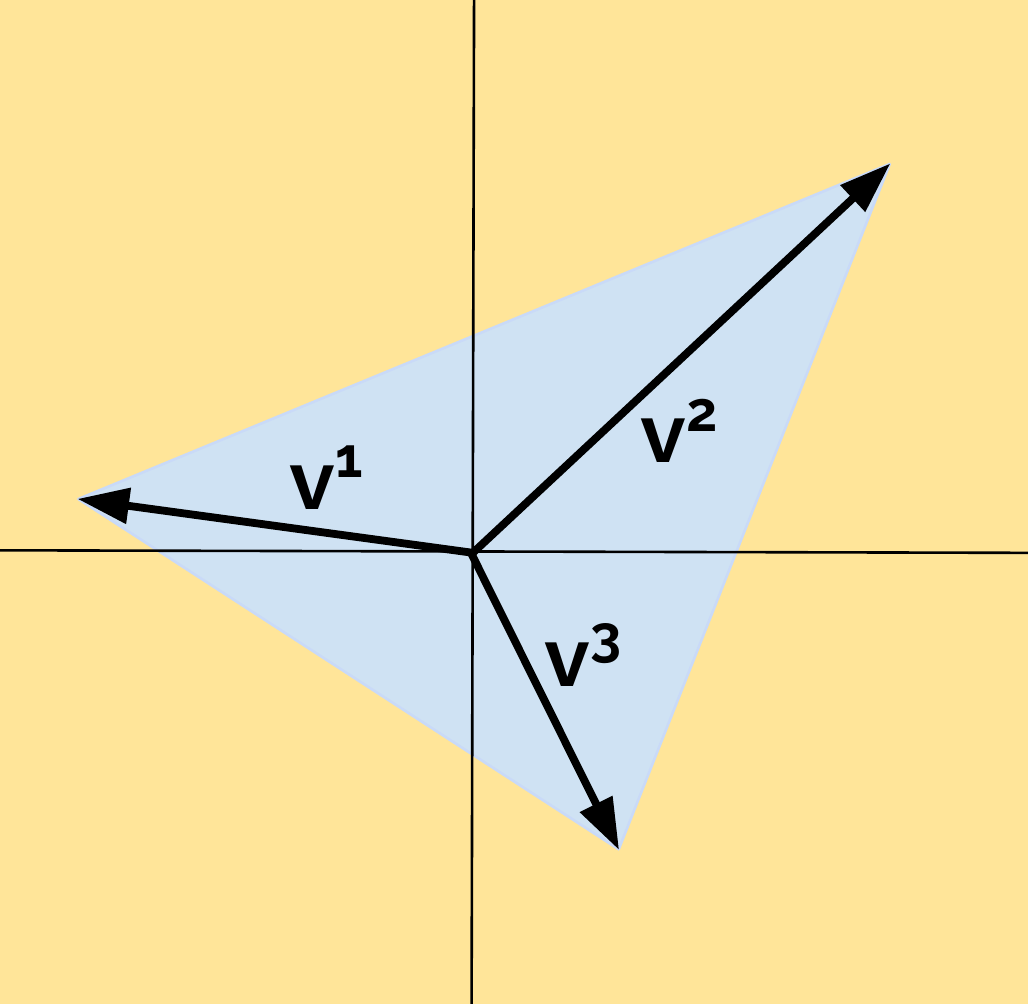}
    \input{plots/std_scaling}
    \input{plots/norm_scaling}
    \caption{\textbf{Left: }Softmax attention outputs can only lie within the convex hull spanned by the value vectors $\vv^i$ (blue region). \textbf{Middle/Right: }The standard deviation ($\sigma$) and norm of a pooling output is dependent on the sequence length $N$ (x-axis) and the pooling method, if the output is not normalized. Softmax attention outputs scale similar to mean pooling at initialization, i.e., Transformers focus more on local information in longer sequences.}
    \label{fig:convex_hull}
\end{figure}
We conjecture that this constraint limits flexibility -- and thereby ease of adjustment -- of the functions expressed by the neural network throughout the training process. This conjecture is supported by our experimental results
showing an increased robustness to hyperparameter choices when the constraint is removed.
Exploring the observation above further, we note the following from a theoretical perspective:

\textbf{No convex combination can represent the binary exclusive OR (XOR) function}\\
A formal proof is given in Appendix~\ref{app:proofs}. Note that this implication highlights an inability to represent non-linearity. While XOR can be represented in architectures with multiple heads and layers, the insight further underlines our argument: An aggregation with weights constrained to the probability simplex is restrictive. Especially if we compare it to other aggregation methods that can represent XOR (cf. Section~\ref{sec:norm_att}).
Finally, we want to highlight an additional insight that, to the best of our knowledge, has not been discussed in the literature so far:

\textbf{Transformers have an aggregation size dependent focus on local information at initialization.}
To see this, consider the embeddings after the first residual connection, given by 
\[\ve^i = \vx^i + \mW\left[\vo_1^i,\dots,\vo_M^i\right] +\vb = \vx^i + \mW\left[\sum_{j}a_1^{i,j}\cdot\vv_1^j,\dots,\sum_{j}a_M^{i,j}\cdot\vv_M^j\right] +\vb\]
where $[\cdot]$ denotes concatenation and $\mW$ and $\vb$ represent the parameters of the affine transformation that mixes the attention $M$ head outputs. 
Our aim is to show how much this embedding $\ve^i$ is influenced by the local information $\vx^i$ relative to the context information $\{\vx^j|j\neq i\}$.
We first note that the contribution of context information depends on the initialization of $\mW$ and $\vb$, where a typical initialization in language models favors the residual connection, i.e., local information.\footnote{As an example, BERT~\cite{BERT} initializes $\mW$ with parameters drawn from a truncated normal distribution with standard deviation 0.02 and $\vb$ to $\bm{0}$.} However, even if we consider $\mW$ as scale preserving, we note that the magnitudes of the attention head outputs $\vo_m$ are upper bounded by the magnitudes of the value vectors $\vv_m$ as a result of the convex hull. Moreover, attention logits are normally close to 0 at initialization (to have the softmax in the unsaturated region). This yields attention to be close to mean aggregation as $\sum_{j}a_m^{i,j}\cdot\vv_m^j\approx \frac{1}{N}\sum_{j}\vv_m^j$. We note that taking the mean effectively scales the standard deviation of a random variable by the square root of the aggregation size. This means that the fraction of context information in $\ve^i$ is dependent on the sequence length and is smaller for longer sequences! Specifically, at initialization, Transformers focus more on local information in longer sequence than in shorter sequences. For reference, we visualize the dependence of $\vo_m$ on aggregation size at initialization for different aggregators in Figure~\ref{fig:convex_hull}~(right). Details on the corresponding experiment can be found in Appendix~\ref{app:seq_dependent_aggregation}. We note that while an architectural bias towards local information might be beneficial in some applications, the implicit dependence on aggregation size is questionable.

\section{Normalized Attention Pooling}
\label{sec:norm_att}
Given the implications that a self-attention based architecture brings along, a few natural questions to ask are: What happens if we remove the softmax? 
Is some form of online logit normalization necessary at all? And how do these architectures compare to simpler pooling methods like sum- or max-pooling?
To investigate these, we contrast the following architectures in our experiments. We provide a schematic figure of each architecture in Appendix~\ref{app:arch}.

\textbf{Transformer Encoder (BERT):} As a starting point, we replicate the encoder architecture presented by~\cite{Transformer} as described in the code release of~\cite{BERT}.\footnote{\url{https://github.com/google-research/bert}} This architecture is among others also used by~\cite{GPT,GPT2,XLNet,T5,RoBERTa,ELECTRA}.
Each Transformer-layer consists of two sub-modules: a multi-head self-attention \quotes{layer} and a feed forward network. Both modules have residual connections around them.
The multi-head self-attention \quotes{layer} consists of a projection to queries, keys and values, the attention mechanism as well as a mixing layer as described in Section~\ref{sec:background}.
The feed forward network consists of two layers with a GELU~\cite{gelu} non-linearity on the hidden layer.
Layer normalization~\cite{layerNorm} is applied \emph{between} incoming and outgoing residual connections. Note that this gives a crucial distinction of this architecture: Embeddings are normalized \emph{after} they are summed with the residual connection. This yields the implicit dependence on the sequence length as discussed in the end of Section~\ref{sec:limitations}.
Further, in this architecture training is done with learning rate warm-up and gradient norm clipping.

\textbf{Modified Transformer Encoder (MTE):} To overcome the implicit dependence on sequence length, reduce training specific confounding factors and to make the two sub-modules more similar to each other, we introduce the following modifications:
We remove learning rate warm-up and gradient clipping, but keep a linearly decreasing learning rate schedule, taking~\cite{budgetTraining} as reference.
Layer normalization is moved before the residual addition. Additionally, we add layer normalization on the hidden layers in the modules, i.e., before the mixing layer and before the GELU non-linearity in the feed forward network. These modifications remove the dependence on sequence length. Note that this is different from the recently studied PreNorm~\cite{TransformerRL,noTears,TransformerDifficulty} that places the normalization before the attention mechanism.
Finaly, we add an additional GELU non-linearity in the middle of the attention sub-module.
We provide an ablation of all modifications in Appendix~\ref{app:ablation}. All following architectures apply the same modifications.
The resulting \emph{MTE} architecture here still projects the attention weights to the probability simplex through the $\softmax$ in the multi-head attention. This architecture is thereby limited to convex combinations of value vectors.

\textbf{Normalized Attention Pooling (NAP):} Given the success of online normalization during training - be it through batch-~\cite{batchNorm}, layer-\cite{layerNorm} , group-~\cite{groupNorm}, instance-~\cite{instanceNorm} or weight-normalization~\cite{weightNorm} - our main proposal is to simply replace the softmax through a normalization:
\begin{equation}
\label{eq:normalize}
    \va_m^{i} = \normalize([l_m^{i,1},\dots,l_m^{i,N}])\qquad\text{with }\normalize(\vx)^j = g\cdot\frac{x^j - \mu_\vx}{\sigma_\vx} + b
\end{equation}
where $\mu_\vx=\frac{1}{N}\sum_j x^j$ and $\sigma_\vx=\frac{1}{N}\sum_j (x^j-\mu_\vx)^2$ are the mean and standard deviation of the corresponding input vector $\vx$, in our case the logit vector calculated through key-query dot products. Similar to layer normalization~\cite{layerNorm}, we introduce trainable gain and bias parameters $g$ and $b$ initialized to $1$ and $0$, respectively. However, while~\cite{layerNorm} introduce gain and bias vectors, we only introduce scalar parameters and broadcast these over the sequence/vector length, as we want the architecture to be independent of the sequence length $N$. Note that while no convex combination can represent the logical XOR, a normalized weighting can - see Appendix~\ref{app:proofs} for the corresponding proof.

\textbf{No Online Logit Normalization (NON):} To investigate whether a dynamic normalization of the attention logits is necessary, we also train a model where we use the logits $l_m^{i,j}$ directly as attention weights, i.e., $\vo_m^i=\text{GELU}(\frac{1}{\sqrt{N}}\sum_{j}l_m^{i,j}\cdot\vv_m^j)$. We also replaced the layer normalization after the attention weighting here through a simple scaling factor $\frac{1}{\sqrt{N}}$. Note that this also yields an in expectation constant contribution of context at initialization, independent of sequence length. However, the model can easily deviate from it during training.

\textbf{Simple Summation of Embeddings (sum):} From a theoretical perspective summation is sufficient for general function approximation~\cite{DeepSets,DeepSetsGNN,UniversalEquivariantSetNetworks}. Therefore, we investigate to simply replace attention through a sum-reduce-broadcast operation.

\textbf{Max Pooling over Sequence Dimension (max): } Similar to sum pooling, we can replace the attention sub-module through a simple max-reduce-broadcast operation over the sequence dimension. Note that max pooling over the sequence is a powerful operation, as 
the resulting embedding 
has a direct link to up to $d$ different tokens. 

If not varied in a corresponding experiment, we default architecture hyperparameters to $L=2$ Transformer-layers (consisting of an attention sub-module and feed forward sub-module each), $M=4$ heads to calculate the logits (if applicable), $d=128$ as model dimension and train on a total of $3200$ batches of $32$ example sequences each, using the Adam optimizer~\cite{Adam}. The hidden dimension of the feed forward sub-modules is $4\cdot d$ for the models \emph{BERT}, \emph{MTE}, \emph{NAP} and \emph{NON}. For the models \emph{sum} and \emph{max} we increase the feed forward hidden dimension to approximately match the parameter counts of the other models. 

\section{Experiments and Results}
\label{sec:experiments}
Our goal with this work is to provide an insight into the variety of performance implications that the architecture choices entail. We aim to provide these insights independent of any particular downstream application, as these architectures
can be applied to a variety of tasks -- from NLP to graph neural networks to reinforcement learning agents.
We therefore focus on carefully crafted
synthetic tasks that (1) are general enough in that we can expect the insights to generalize to a large set of downstream tasks and (2) let us modify key aspects that are hidden in real world data sets, such as a bias towards a certain sub-task.
The focus on synthetic tasks also allows us to get a better grasp on the learning dynamics -- the focus of this work -- as we can train thousands of models in diverse hyperparameter combinations.
To limit the influence of confounding variables, we generate new data points for every batch. This allows us to omit regularization. See Appendix~\ref{app:regularization} for an in depth discussion of this setup.

\subsection{Argmin-First-Argmax Case Distinction Task}
\label{sec:per_token_output}
\begin{figure}
    \centering
    \includegraphics[width=4.5cm]{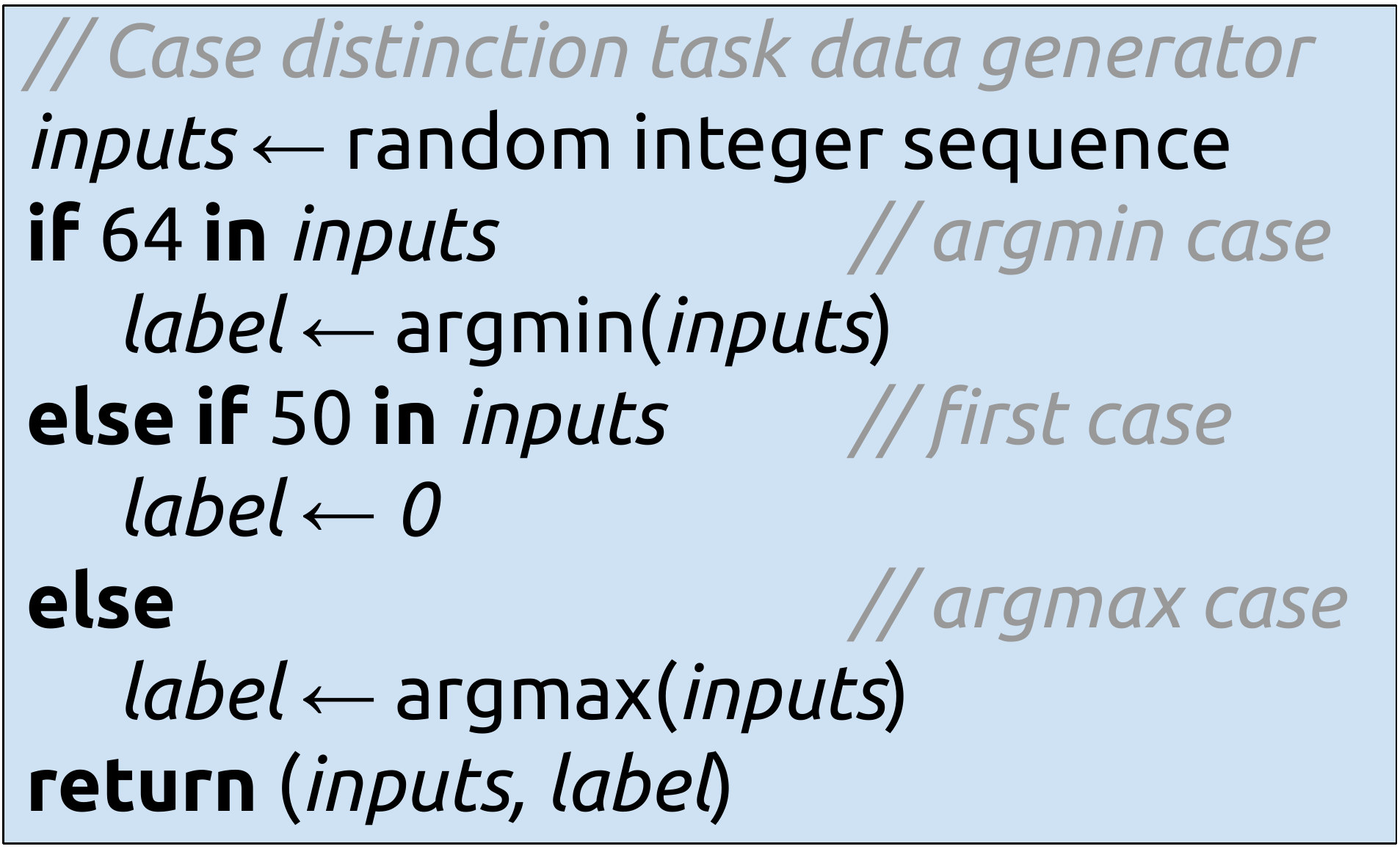}\quad
    \includegraphics[width=8.9cm]{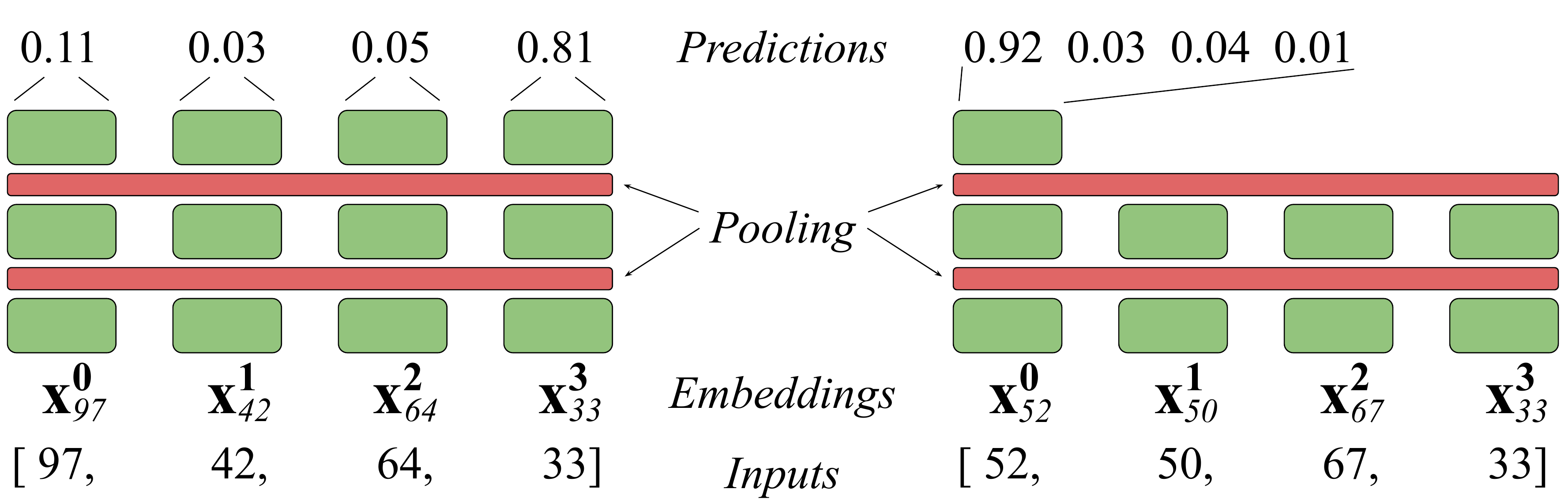}
\caption{\textbf{Left: } Pseudo code for case distinction task data. The case distinction points 64 and 50 are chosen arbitrarily. \textbf{Middle/Right: }Task setup for outputs across all tokens (middle, cf. Section~\ref{sec:per_token_output}) and outputs from the first token (right, cf. Section~\ref{sec:first_token_output}). Green boxes represent the trainable network layers (shared across tokens) while red boxes represent the pooling across tokens, the focus of this work. The targets of the displayed examples would be [0, 0, 0, 1] and [1, 0, 0, 0], respectively.}
    \label{fig:task_description}
\end{figure}
As a first task, we consider an input pipeline where tokens from a fixed integer-vocabulary are translated to a randomly initialized embedding. To the embedded tokens, a (also randomly initialized) positional embedding is added to provide position-relative information. The sequence of tokens is then processed by several architecture dependent Transformer-layers (as described in Section~\ref{sec:norm_att}). Finally, each contextualized embedding is projected to a single output. A softmax-crossentropy loss is applied over the sequence dimension to train the networks to pin-point a specific, input dependent token. See the example in the middle of Figure~\ref{fig:task_description} for a visualization.
Note that the ability to pin-point a specific token is an abstract task relevant to NLP (e.g., question answering or co-reference resolution), graph neural networks (e.g., finding the next hop in a shortest path) as well as reinforcement learning (e.g., action credit assignment).
To make the task input dependent, we generate the data as given in the pseudo code in Figure~\ref{fig:task_description}.
Note that the \texttt{argmin} and \texttt{argmax} make this task quite challenging from a learning perspective as the networks start from random embeddings which do not provide any ordering information. Which embeddings correspond to bigger integers and which to smaller integers has to be inferred during training. Further, the case distinction in this task lets us tweak the data bias towards each sub-task. Specifically, we consider a vocabulary size of $S=100$ integers (0-99) and uniformly random sampled sequences of $N=128$ tokens in length. This leads to a bias as
$p_{argmin} = 1 - (1 - \frac{1}{S})^N \approx 72.4\%$ of data points require the network to pin-point the minimum in the input sequence,
$p_{first} 
\approx 20.1\%$ 
require the network to pin-point the first token of the sequence and the remaining $p_{argmax}
\approx 7.5\%$ 
require the network to pin-point the maximum in the input.

\subsubsection{Varying Model Dimension $d$}
\label{sec:var_dim}
As a first investigation, we are interested in how varying the model dimension $d$ influences the architectures ability to learn the given task. For this, we train each of the architectures for each of the model dimensions $d\in\{8,16,32,64,128,256,512,1024\}$ using 10 different learning rates and 5 random seeds for each hyperparameter combination. 
As we want to base our insights on as many results as possible, we derive a novel, human friendly visualization of results.
Figure~\ref{fig:cases_var_dim} (top row) shows the first results as follows: The outcome of each hyperparameter combination is reported as an RGB pixel in the plot, where the R (red) value corresponds to the accuracy of the worst performing random seed, the G (green) value corresponds to the average over the random seeds and the B (blue) value corresponds to the best performing random seed. For each value (R, G and B), the max over the course of training is taken. This assignment roughly translates as follows: The brighter, the better - brighter pixels correspond to higher min-, mean- and max-accuracy. Blue/turquoise pixels highlight a large performance variation across random seeds and black/grey pixels correspond to hyperparameter combinations where none of the random seeds could solve the task.
\begin{figure}
    \centering
    \input{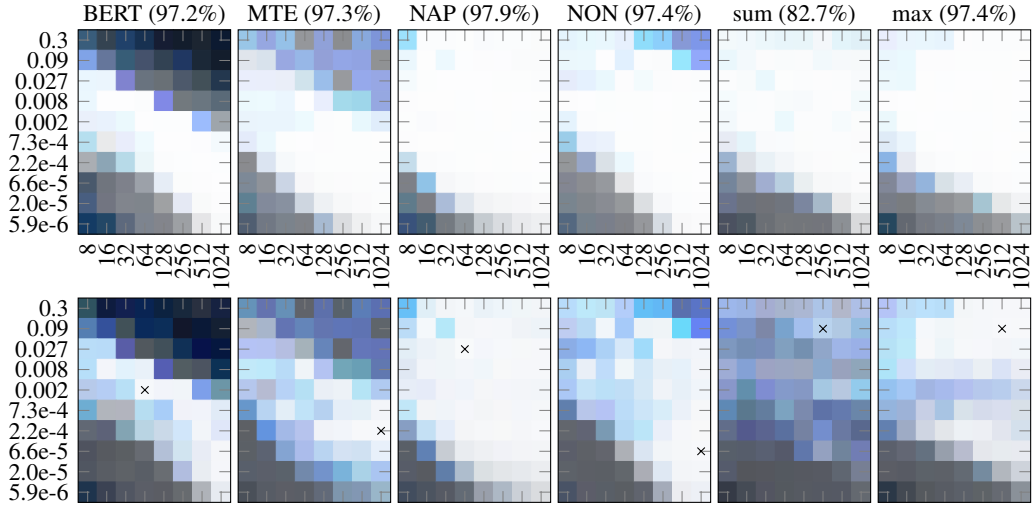}
    \caption{Learning rate (y-axis) vs. model dimension $d$ (x-axis) on the argmin-first-argmax case distinction task (with output across all tokens). The pixels' R (red), G (green) and B (blue) values correspond to min-, mean- and max-accuracy, respectively, of the corresponding hyperparameter combination
    -- see main text for details.
    \textbf{Top row: }Training accuracy (sequence length $N=128$). \textbf{Bottom row: }Validation accuracy when validating on sequences of half the length ($N=64$). Crosses indicate the combination for best mean validation accuracy, which we report behind the model name.}
    \label{fig:cases_var_dim}
\end{figure}
These condensed results directly give rise to the following observations: 
(1) All models have some hyper-parameter combinations that learn the task well (white pixels). 
(2) The optimal learning rate depends on the model size, especially in the \emph{BERT} architecture. This has profound implications for hyperparameter optimization: Tuning hyperparameters independent of each other might lead to sub-optimal results.
(3) Models with probability simplex limitations (\emph{BERT} and \emph{MTE}) work for a smaller range of hyperparameters.
We provide case learning curves and additional results in Appendix~\ref{app:full_results}.
Next, given that all architectures are applicable to sequences of any length, we investigate how the architectures generalize to sequences of different length. Specifically, we validated each of the models trained above after every 100 batches on 32 batches with sequences of half the length ($N=64$). We report the corresponding accuracies as before in Figure~\ref{fig:cases_var_dim} (bottom row). Note that as we are taking the maximum over the course of training, we report optimal early stopping results. We observe:
(1) The \emph{sum} architecture does not generalize well
in this task.
(2) Our \emph{NAP} architecture seems to be the most robust to this generalization.

\subsubsection{Case Accuracy under Varying Data Biases}
\label{sec:var_seq_length}
As a next experiment we reset the model dimension to $d=128$ and vary the sequence length $N\in\{4,8,16,32,64,128,256,512\}$. Note that this implicitly varies the biases $p_{argmin}$, $p_{first}$ and $p_{argmax}$ in the data. We report the case specific accuracies in Figure~\ref{fig:cases_var_seq} as follows: After every 100 batches, we validate the models on 1000 examples per case. Reported is the best accurracy over the course of training in form of pixel value with R (red) corresponding to the \emph{argmin}-case accuracy, G (green) corresponding to the \emph{first}-case accuracy and B (blue) corresponding to the \emph{argmax}-case accuracy. As a consequence, white pixels correspond to all cases learned and yellow pixels correspond to the \emph{argmin}- and \emph{first}-case learned.
\begin{figure}
    \centering
    \input{plots/cases/var_seq_cases}
    \caption{Learning rate (y-axis) vs. sequence length $N$ (x-axis) on the case distinction task (with output across all tokens). RGB pixel values correspond to \emph{argmin}-, \emph{first}- and \emph{argmax}-mean-case-accuracies, respectively.}
    \label{fig:cases_var_seq}
\end{figure}
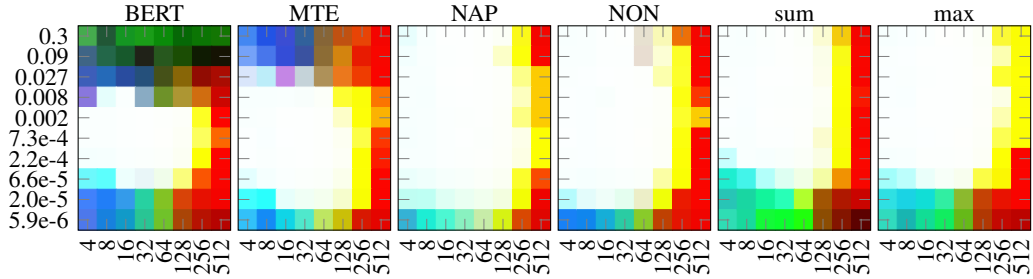
We make the following observations:
(1) If the learning rate is too low, models tend to focus on the majority case (indicated in a shift from blue to red as the bias shifts from the \emph{argmax}- to the \emph{argmin}-case with increasing sequence length $N$).
(2) If the learning rate is too high, the \emph{BERT} architecture tends to focus on the \emph{first}-case. We believe this is due to the architectural bias towards local information as discussed in Section~\ref{sec:limitations}. Note that the \emph{first}-case can be solved by relying on the local positional embedding.
(3) Only the \emph{NAP} and \emph{max} architecture manage to learn all three cases from the highly biased data when $N=256$.
In Appendix~\ref{app:var_batch} we provide a further experiment investigating different batch sizes. The results are complementary.

\subsection{First Token Output}
\label{sec:first_token_output}
The task so far requires the architectures to learn an information flow between tokens to distinguish the case and decide per token, whether it is the token that is looked for or not. Now we investigate, whether all this information can also be aggregated into a single token. We therefore modify the architecture output slightly in that we only take the contextualized embedding of the first token and project from it to a vector of size $N$ (see example on the right in Figure~\ref{fig:task_description}). Note that this task set-up is harder and can highlight bottlenecks in the information flow across tokens.

We fix the sequence lenght to $N=128$ and again vary the model dimension $d$. We report the
the case specific mean accuracies in Figure~\ref{fig:cases_var_dim_first_token_output}, min-, mean- and max-overall-accuracies are given in Appendix~\ref{app:first_token_var_dim}.
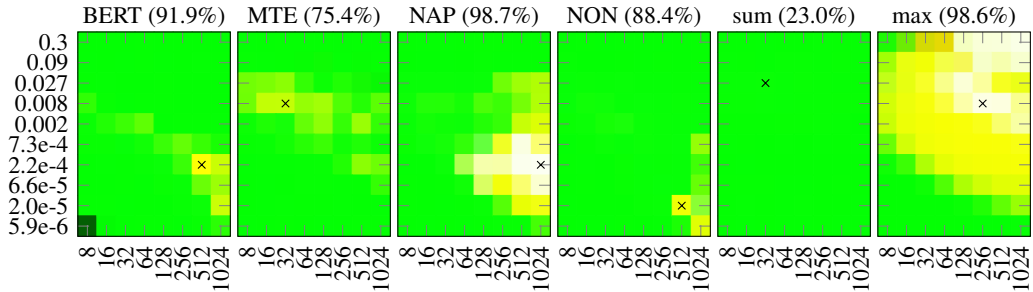
\begin{figure}
    \centering
    \input{plots/cases/first_token_output/var_dim_cases}
    \caption{Learning rate (y-axis) vs. model dimension $d$ (x-axis) on the case distinction task with output from the first token.
    RGB pixel values correspond to \emph{argmin}-, \emph{first}- and \emph{argmax}-mean-case-accuracies.
    Crosses indicate the best mean accuracy, which we report behind the model name.}
    \label{fig:cases_var_dim_first_token_output}
\end{figure}
We observe:
(1) All architectures learn for (almost) all combinations the now close to trivial \emph{first}-case.
(2) The \emph{sum pooling} architecture does not learn any of the other cases.
(3) Only \emph{NAP} and \emph{max} learn all three cases in some hyperparameter combinations.
The worse performance of \emph{NON} highlights the advantage of online normalization of the logits. While the $\softmax$ provides some form of online normalization, we hypothesize that the worse performance of \emph{MTE} and \emph{BERT} in this task stems from an information bottleneck induced by the probability simplex limitations. To test this hypothesis, we vary the number of attention heads $M$ with results in Figure~\ref{fig:cases_var_heads_first_token_output}.
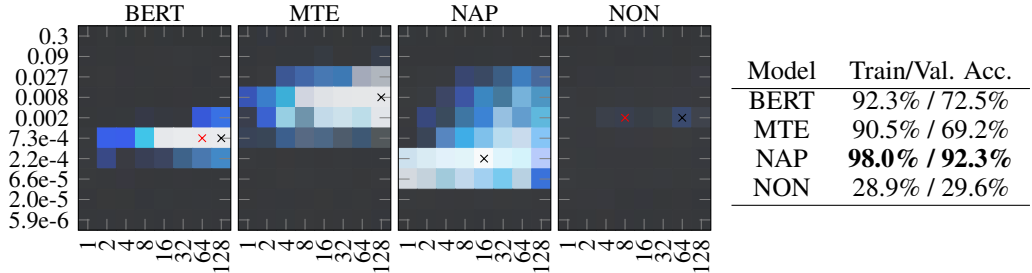
\begin{figure}
    \centering
    \input{plots/cases/first_token_output/var_heads}
    \begin{tabular}[b]{cc}
        Model & Train/Val. Acc.\\ \hline
        BERT & 92.3\% / 72.5\% \\
        MTE & 90.5\% / 69.2\% \\ 
        NAP & \textbf{98.0\% / 92.3\%} \\
        NON & 28.9\% / 29.6\% \\\hline
        \\
        \\
        \\
    \end{tabular}
\caption{Learning rate (y-axis) vs. attention heads $M$ (x-axis) on the case distinction task (output from the first token). RGB pixel values correspond to min, mean and max accuracy. Black crosses indicate the best mean accuracy, reported in the table to the right. Red crosses indicate that the best mean validation accuracy (when validating with $N=64$) was taken from a different combination.
Bold numbers indicate a min-accuracy higher than the best max-accuracy of the other models.}
\label{fig:cases_var_heads_first_token_output}
\end{figure}
We observe that increasing the number of heads helps the \emph{MTE} and \emph{BERT} architecture, supporting our hypothesis. Note however, that \emph{MTE} and \emph{BERT} are still outperformed significantly by \emph{NAP}.
In Appendix~\ref{app:first_token_var_depth} we provide a further experiment, varying the depth up to $L=64$. The results are complementary.

\subsection{Mode Finding Task}
\label{sec:mode}
Given the results so far, one could conclude that \emph{max} is the best choice due to its simplicity. Note however, that \emph{max} has an architectural prior that is in line with the underlying task of finding the maximum or minimum of the sequence. To study the effect of architectural priors, we experiment on an additional task: Finding the mode/most common integer in the input sequence. 
Also this task has ties to NLP (e.g., sentiment analysis), graph neural networks (e.g., consensus/agreement) and reinforcement learning (e.g., count based exploration).
Here we remove the positional embeddings, as this task can also be done on sets, and project from the contextualized embedding of the first token to a vector of dimension $S$ (the vocabulary size) over which we apply the softmax-cross-entropy loss. We keep $N=128$ but reduce $S$ to 10 to have meaningful modes. Ties are broken by taking the smallest integer of the ones with maximal occurrence. Results of varying the model dimension $d$ are reported in Figure~\ref{fig:mode_var_dim}.
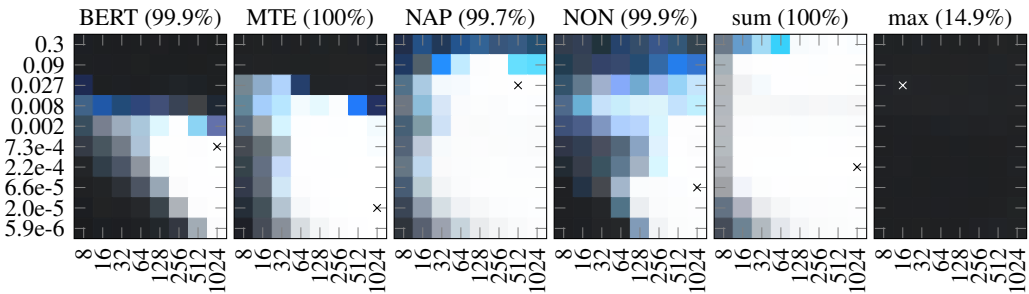
\begin{figure}
    \centering
    \input{plots/mode/var_dim}
    \caption{Learning rate (y-axis) vs. model dimension $d$ (x-axis) on the mode finding task. RGB pixel values correspond to min, mean and max accuracy. Crosses indicate the reported best mean accuracy.}
    \label{fig:mode_var_dim}
\end{figure}
We observe:
(1) \emph{sum pooling} works well on this task, as it has a suitable architecutral prior.
(2) \emph{max pooling} cannot learn the task, not even with a model dimension $d=1024=8\cdot N$.
In Appendix~\ref{app:var_vocab} we provide an additional experiment, varying the vocabulary size. The results are complementary. We refer an interested reader to~\cite{NNReasoning} for more on architecture-task alignment.

\section{Conclusion}
Taking all observations together, we come to the following conclusions: Many recent works apply some sort of neural self-attention mechanism involving a $\softmax$ that projects the attention weights to the probability simplex.
In this work we question the $\softmax$ in dot-product self-attention modules. Our theoretical investigation shows 
that $\softmax$-attention outputs are constrained to the convex hull spanned by the value vectors. In our experiments we show that this can lead to an unwanted hyperparameter sensibility. We show that simpler architectures like max- and sum-pooling perform well when their architectural prior aligns with the underlying task. These architectures however fail in cases where the architectural prior is not suitable.
As a solution, we propose to replace the $\softmax$ in attention through normalization.
Our resulting normalized attention pooling (\emph{NAP}) architecture is the only architecture of the 6 investigated that performs well in all tasks and setups, showing a broad applicability and better performance than the widely used \emph{BERT} architecture. We hope that our work provides a stepping stone to examine architectures with respect to biases in the data. Further, we see a lot of potential for future work to investigate the correlated effects of hyperparameters.

\newpage
\section*{Broader Impact}
We contrast different architectures on an abstract level in this work. 
Hence, there is no direct risk associated with system failure or an implication that would put some at a disadvantage. 
On the contrary: We see huge potential in our work to benefit (1) researchers and practitioners that do not have the computational resources to perform expensive hyperparameter optimizations and (2) minorities under-represented in data-sets, as our proposed architecture shows increased robustness to hyperparameter changes and biases in the data. 

\begin{ack}
The main author would like to thank his colleagues Damián Pascual, Béni Egressy, Lukas Faber, Gino Brunner, Zhao Meng and Johannes Ackermann for the insightful discussions and helpful feedback on preliminary versions of this work.
\end{ack}

\bibliographystyle{plain}
\bibliography{ref}

\clearpage
\appendix

\section{Lemmas and Proofs}
\label{app:proofs}
\begin{lemma}
No convex combination can represent the binary exclusive OR (XOR) function defined on binary inputs $x_1\in\{0,1\}$ and $x_2\in\{0,1\}$ 
by the indicator function as $XOR(x_1,x_2) = \1_{x_1\neq x_2}$.
\end{lemma}
\begin{proof}
Suppose there exist convex combination weights $a_1$ and $a_2$ with $a_1 + a_2 = 1$, such that $a_1\cdot x_1 + a_2\cdot x_2$ represents the XOR function. Plugging in $x_1=x_2=1$ yields $a_1\cdot x_1 + a_2\cdot x_2 = a_1 + a_2 = 1$, which gives the contradiction.
\end{proof}

\begin{lemma}
Given the two binary inputs $x_1\in\{0,1\}$ and $x_2\in\{0,1\}$, there exists an affine mapping $f: \{0,1\}^2\to\R^2$, such that 
\[\text{normalized weighting}(f,x_1,x_2) = \frac{f_1(x_1,x_2) - \mu_{f(x_1,x_2)}}{\sigma_{f(x_1,x_2)}}\cdot x_1 + \frac{f_2(x_1,x_2) - \mu_{f(x_1,x_2)}}{\sigma_{f(x_1,x_2)}}\cdot x_2\]
is equivalent to the logical exclusive OR given by the indicator function as $XOR(x_1,x_2) = \1_{x_1\neq x_2}$.
\end{lemma}
\begin{proof}
For a vector $\vl\in\R^2$, the standard deviation $\sigma_\vl$ can be simplified to
\[\sigma_\vl = \sqrt{\frac{1}{2}\sum_{i\in\{1,2\}}(l_i - \mu_\vl)^2} = \sqrt{\frac{1}{2}\left(\left(l_1 - \frac{l_1 + l_2}{2}\right)^2 + \left(l_2 - \frac{l_1 + l_2}{2}\right)^2\right)} = \frac{1}{2}|l_1 - l_2| \]
and the normalization function reduces to 
\[normalize(\vl) = \left[\frac{l_1 - \mu_\vl}{\sigma_\vl}, \frac{l_2 - \mu_\vl}{\sigma_\vl}\right]^T = \left[\frac{l_1 - l_2}{|l_1 - l_2|}, \frac{l_2 - l_1}{|l_1 - l_2|}\right]^T = \begin{cases}
[1, -1]^T & if\quad l_1 > l_2\\
[-1, 1]^T & if\quad l_1 < l_2\\
undef. & if\quad l_1 = l_2
\end{cases}\]
As an example, consider the affine mapping $f(\vx) = \vl = [3x_1 + 1, 2x_2]^T$, which for $x_1\in\{0,1\}$ and $x_2\in\{0,1\}$ results in the function
\[\textit{normalized weighting}(f,x_1,x_2) = \frac{3x_1 + 1 - 2x_2}{|3x_1 + 1 - 2x_2|}\cdot x_1 + \frac{-3x_1 - 1 + 2x_2}{|3x_1 + 1 - 2x_2|}\cdot x_2= 
\begin{cases}
1 & if\text{ }x_1\neq x_2\\
0 & otherwise
\end{cases}\]
\end{proof}

We note that for a realization of such an affine mapping across tokens given the weight sharing constraints of the discussed architectures we would need $x_1$ and $x_2$ to be distinguishable for the mapping to keys and queries, e.g., through positional embeddings. This however does not invalidate our conclusion that normalized weighting is more expressive than softmax weighting, as we do not require the inputs that are weighted to be distinguishable.

\section{Sequence Length Dependent Local/Context-Focus}
\label{app:seq_dependent_aggregation}
For the middle and right plot in Figure~\ref{fig:convex_hull} we sample 16'384 value, key and query vectors of dimension $d_h=128$ per sequence length $N\in\{1,2,4,8,16,32,64,128,512,1024,2048\}$ from a normal Gaussian $\mathcal{N}(\mathbf{0},\mI_{d_h})$ - $\mI_{d_h}$ being the $d_h$-dimensional identity matrix. We split the samples to form the sequences and calculate the corresponding output vectors $\vo^i$ for $i\in\{1,\dots,N\}$. Here, the softmax attention outputs are calculated as described in Section~\ref{sec:background}, while the mean-, sum- and max-outputs are calculated as mean-, sum- and max-reduce of the value vectors over the sequence dimension. For the normalized results we take the sum-output vectors and normalize them (over the $d_h$-dimensional vector dimension). Note that such a normalization can be applied to any of the aggregation methods to get qualitatively similar results. The plots in Figure~\ref{fig:convex_hull} are generated by reporting the standard deviation over all output values and the mean norm of the output values, respectively.

Given the numerous successes of Transformers in natural language processing, we conjecture that a bias towards local information might be beneficial in language modeling. However, the implicit dependence on sequence length in a model that should be oblivious to different input sequence lengths is questionable. We leave an in depth investigation to future work.

\newpage
\section{Architectures}
\label{app:arch}
We provide a schematic of 1 Transformer-layer of each architecture investigated in Figure~\ref{fig:architectures}. Our base architectures consist of 2 such layers followed by a projection to the output dependent on the task as described in the corresponding sections (cf. Section~\ref{sec:per_token_output}, \ref{sec:first_token_output} and \ref{sec:mode}).
\begin{figure}[h!]
    \centering
    \begin{subfigure}[b]{4.5cm}
    \includegraphics[width=4.5cm]{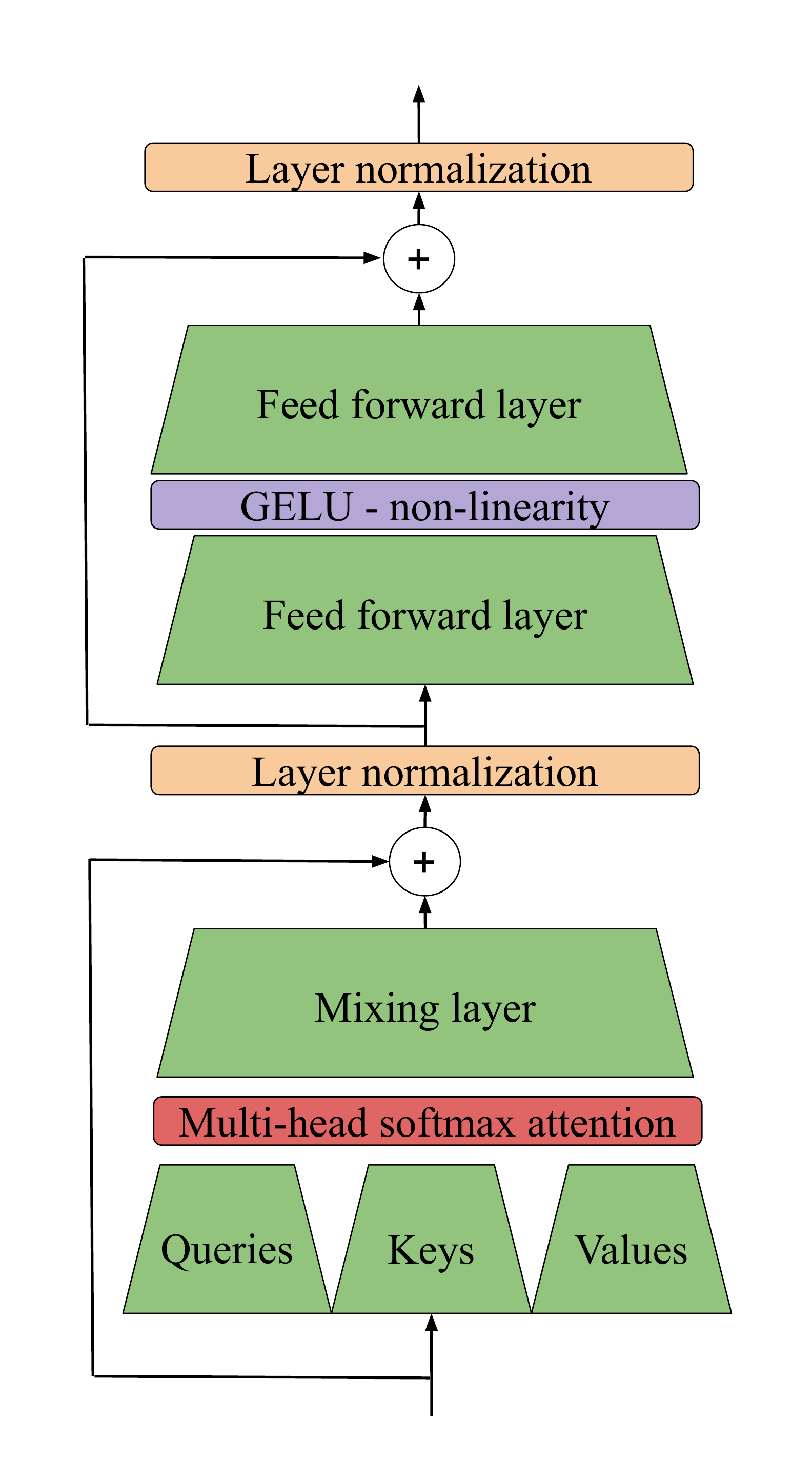}
    \subcaption{BERT}
    \end{subfigure}
    \begin{subfigure}[b]{4.5cm}
    \includegraphics[width=4.5cm]{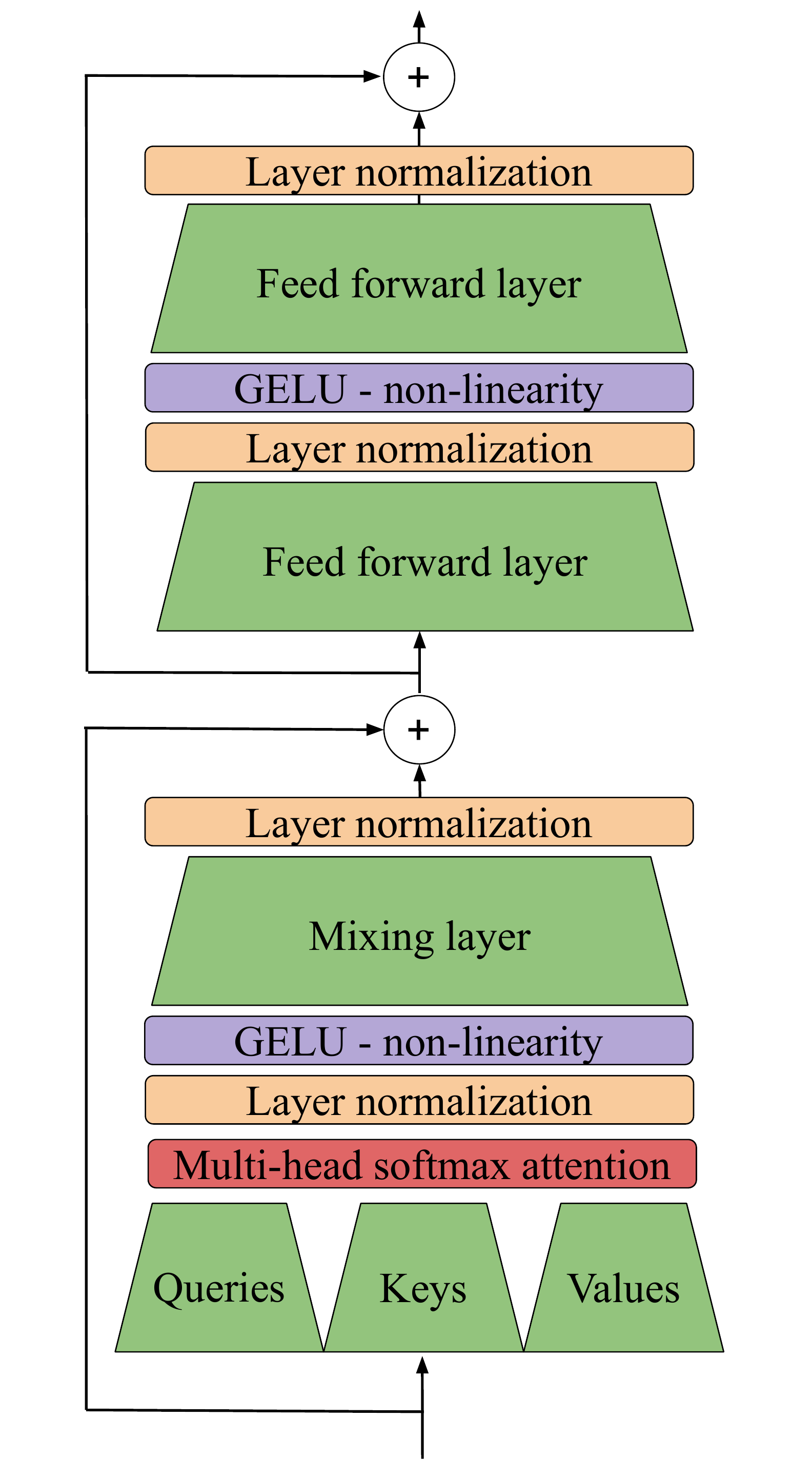}
    \subcaption{MTE}
    \end{subfigure}
    \begin{subfigure}[b]{4.5cm}
    \includegraphics[width=4.5cm]{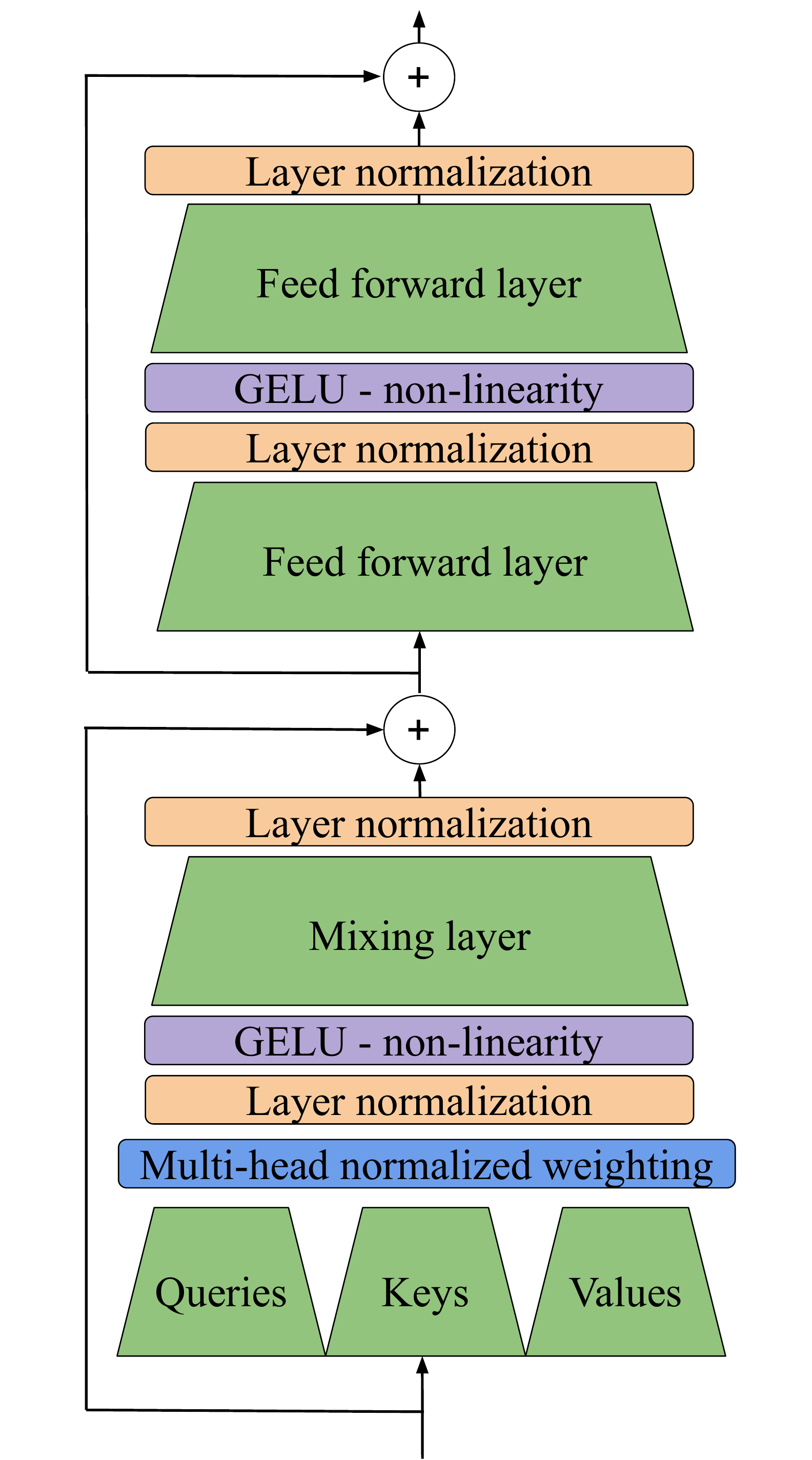}
    \subcaption{NAP}
    \end{subfigure}
    \begin{subfigure}[b]{4.5cm}
    \includegraphics[width=4.5cm]{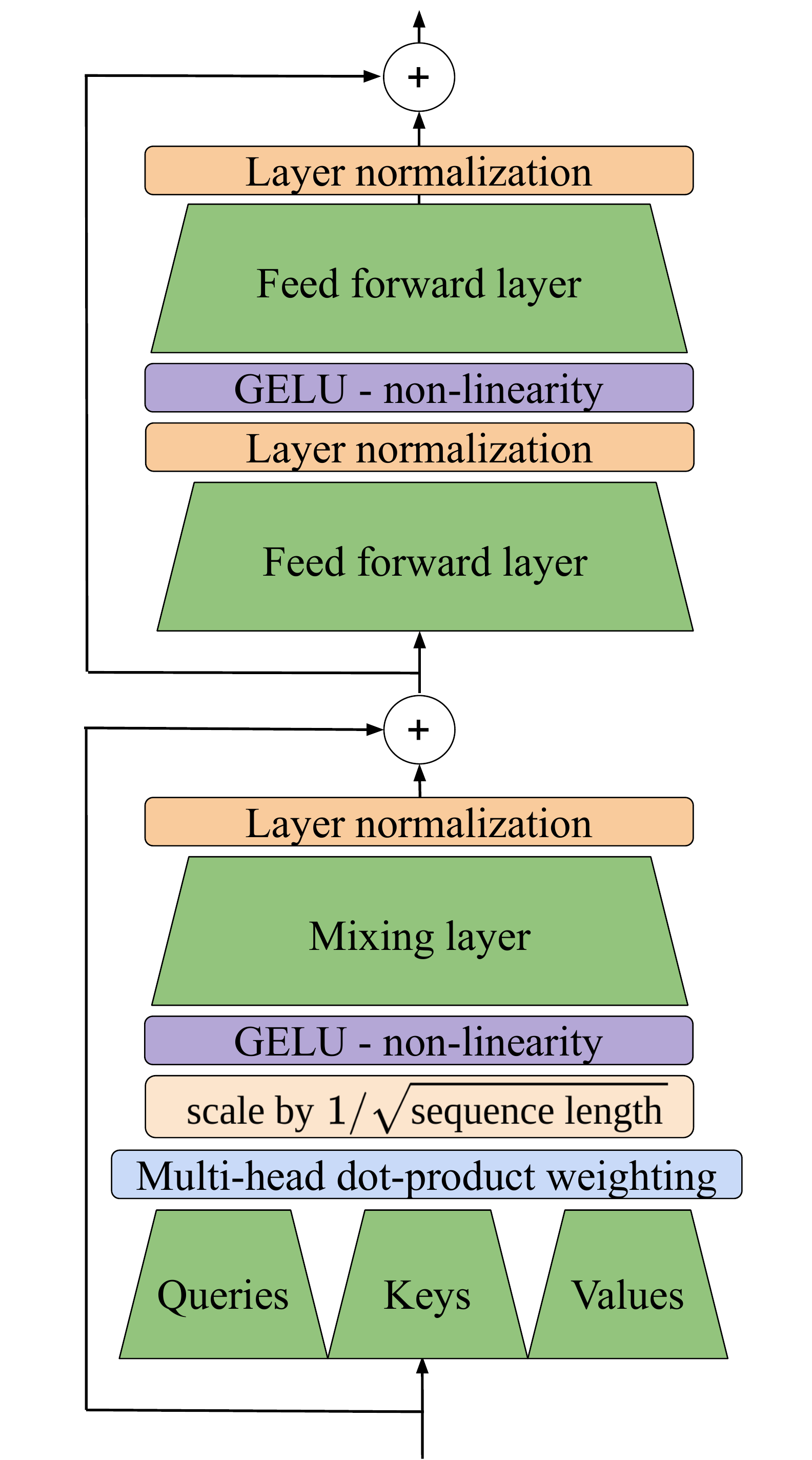}
    \subcaption{NON}
    \end{subfigure}
    \begin{subfigure}[b]{4.5cm}
    \includegraphics[width=4.5cm]{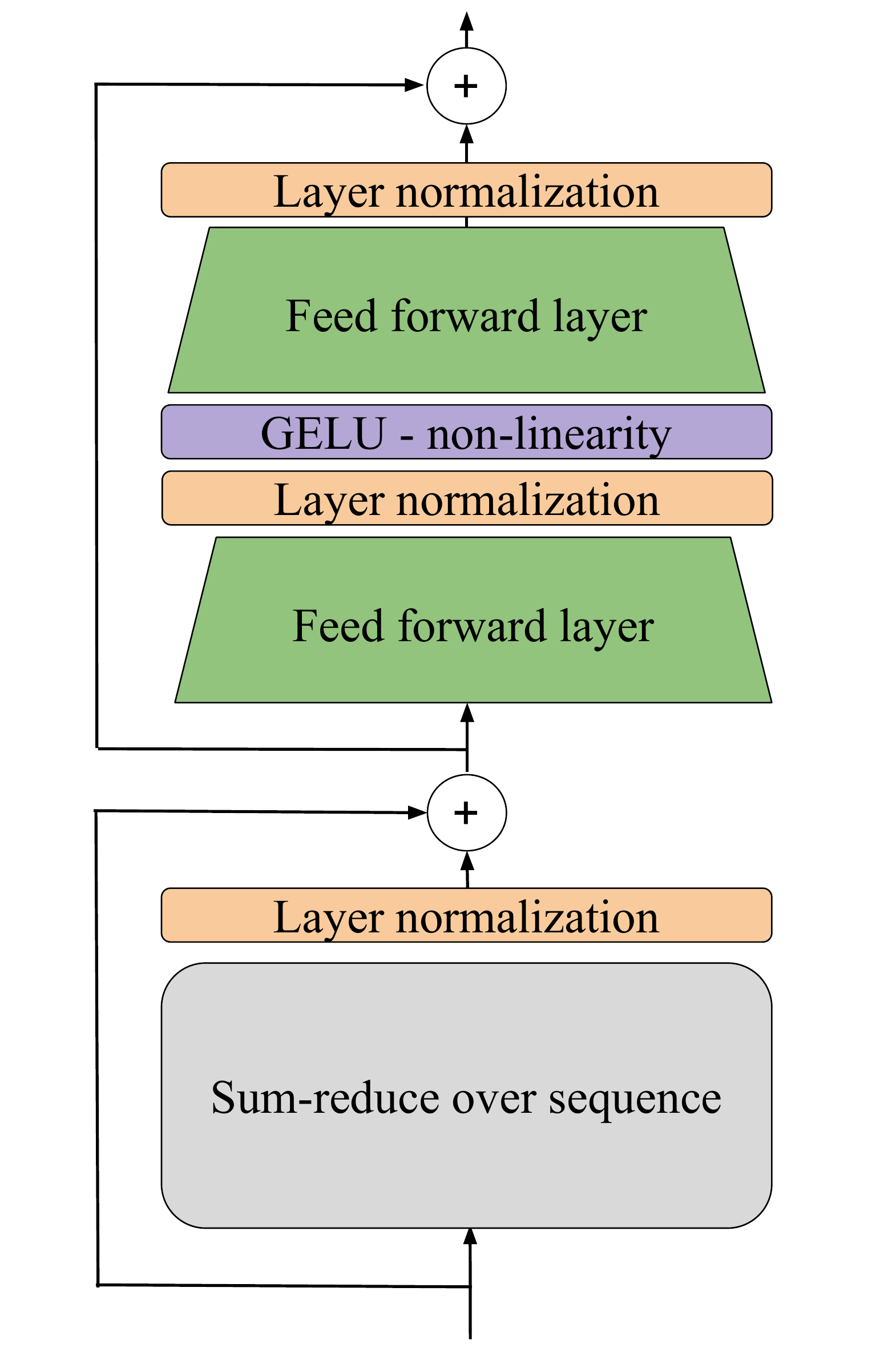}
    \subcaption{sum}
    \end{subfigure}
    \begin{subfigure}[b]{4.5cm}
    \includegraphics[width=4.5cm]{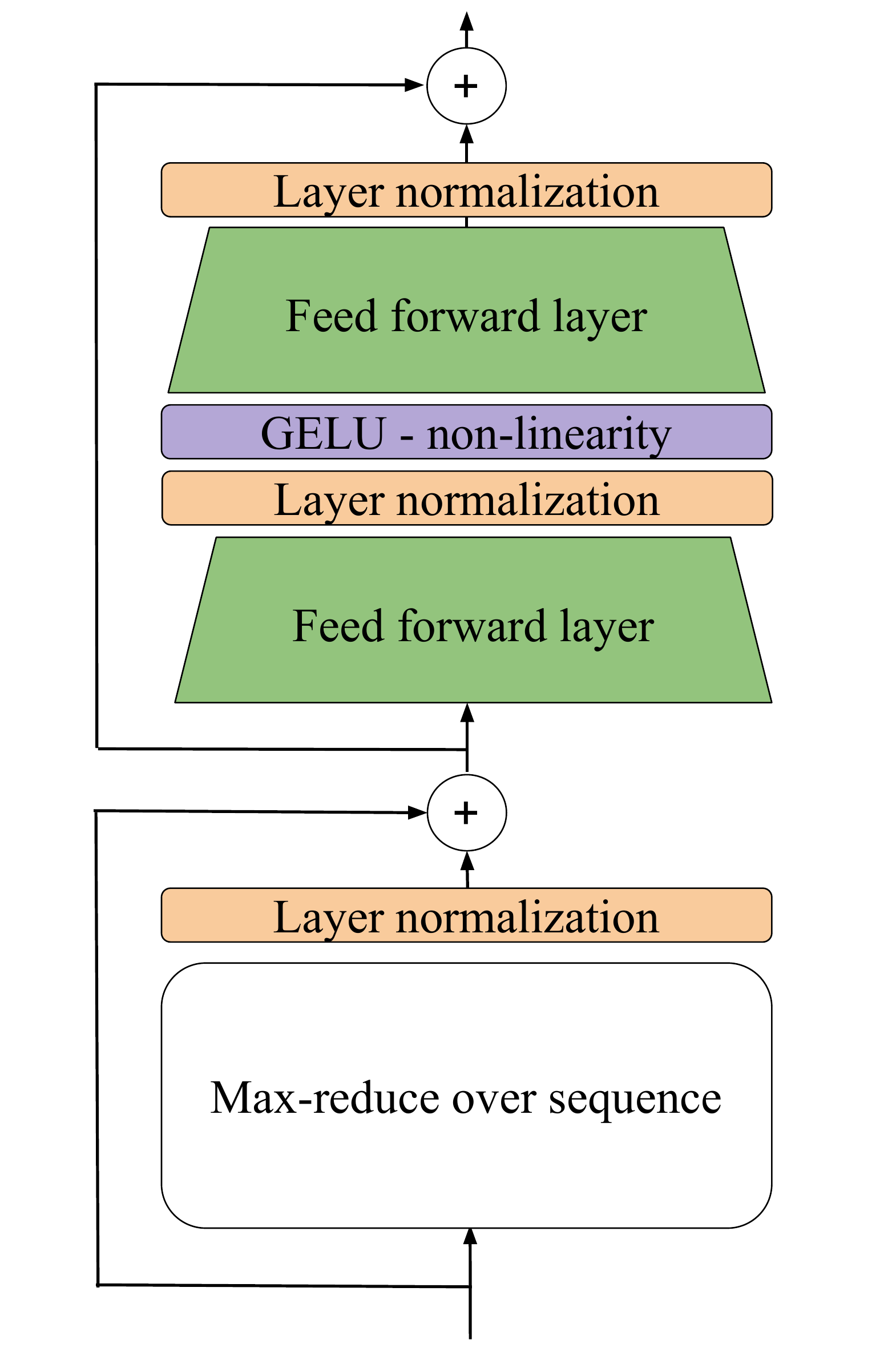}
    \subcaption{max}
    \end{subfigure}
    \caption{Schematics of 1 Transformer-layer block of the different architectures investigated. Green layers correspond to the main weight matrices that are trained. Note that displayed dimensions are not to scale - the hidden dimension of the feed forward layer is larger than the model dimension and the hidden layer size in the feed forward network of \quotes{max} and \quotes{sum} are adjusted to approximately match the parameter count of the other architectures.}
    \label{fig:architectures}
\end{figure}

\newpage
\section{Architecture Modification Ablations}
\label{app:ablation}
An empirical ablation of the modifications that lead from the \emph{BERT} architecture to the \emph{MTE} architecture is given in Figure~\ref{fig:ablation}. The plots are generated as described in Sections~\ref{sec:var_dim} and \ref{sec:var_seq_length}. 

\begin{figure}[bh]
    \centering
    \input{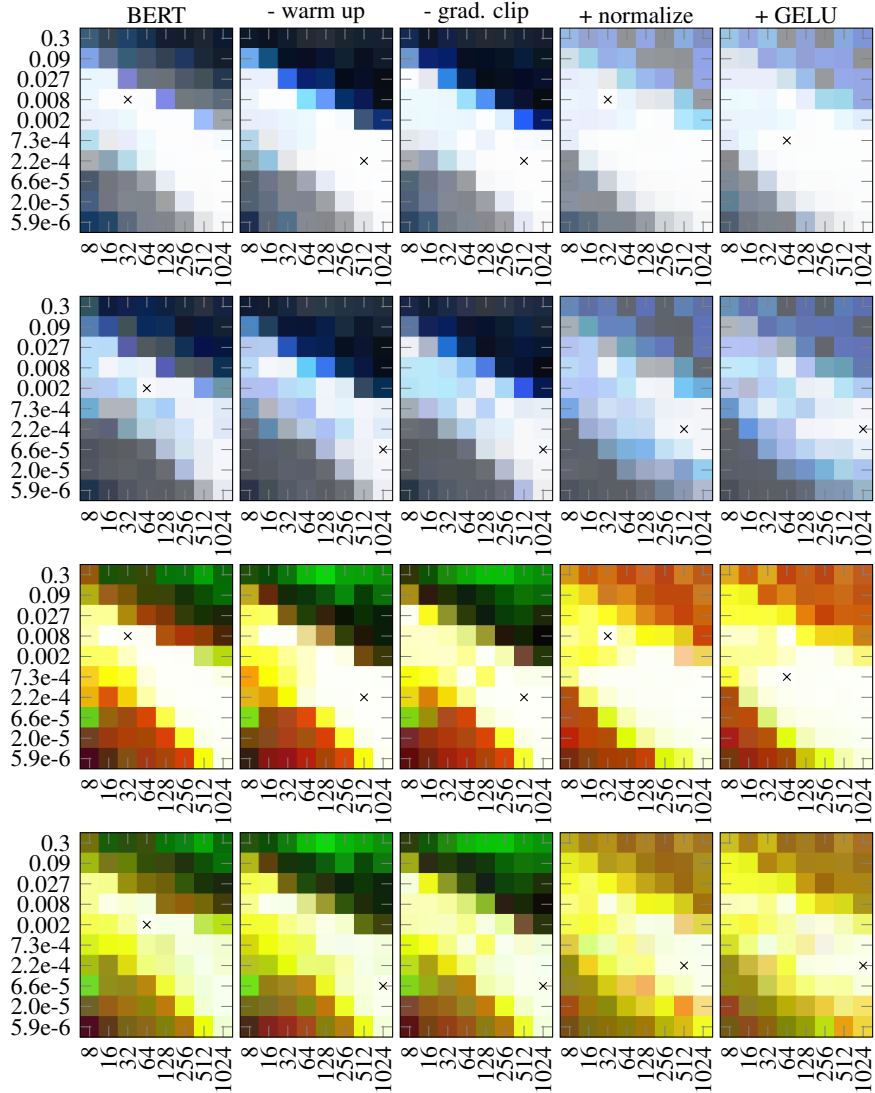}
    \caption{Learning rate (y-axis) vs. model dimension $d$ (x-axis) on the argmin-first-argmax case distinction task (with output across all tokens) - architecture modification ablation study. In the first two rows, RGB pixel values correspond to min-, mean- and max-accuracy. In the last two rows, RGB pixel values correspondto \emph{argmin}-, \emph{first}- and \emph{argmax}-mean-case-accuracies. \textbf{1. row: }Training accuracy (sequence length $N=128$). \textbf{2. row: }Validation accuracy when validating on sequences of half the length ($N=64$). \textbf{3. row: } Training case accuracy (sequence length $N=128$). \textbf{4. row: } Validation case accuracy when validating on sequences of half the length ($N=64$). Crosses indicate the combination for best mean accuracy, the accuracies at these locations are reported in Table~\ref{tab:ablation}.}
    \label{fig:ablation}
\end{figure}

The first column in Figure~\ref{fig:ablation} corresponds to the original \emph{BERT} architecture, trained with gradient norm clipping and learning rate warm up. 

The second column (- warm up) corresponds to the same architecture, but trained without learning rate warm up. Here we see that too high learning rates learn even less without learning rate warm up in the \emph{BERT} architecture, hinting at a necessity for learning rate warm up for the original architecture.

The third column (- grad. clip) reports the results if we further remove gradient clipping from the training schedule. This does not seem to have a big impact in our setup.

Next, we report in the forth column (+ normalize) the results of moving the layer normalization before the residual addition and introducing an additional layer normalization right after the attention mechanism as well as on the hidden layer of the feed forward network. Note that this change removes the bias towards local information discussed in the end of Section~\ref{sec:limitations}. We see that this change leads to a profound shift in focus in regions where the learning rate is high: models with the original normalization focus the (local) \emph{first}-case, while models with our normalization focus on the (majority) \emph{argmin}-case. This is in line with the insights stated in Section~\ref{sec:var_seq_length}.

Finally, we report in the fifth column (+ GELU) the results of adding an additional GELU layer after the attention mechanism. These results correspond to the \emph{MTE} architecture used throughout the paper.

Apart from the performance landscape changes just mentioned, the best hyper-parameter accuracies remain similar throughout all modifications, cf. Table~\ref{tab:ablation}.

\begin{table}
    \caption{Ablation study accuracy values taken from the hyper-parameter combination that led to the best mean overall accuracy, indicated by a cross in Figure~\ref{fig:ablation}.}
    \label{tab:ablation}
    \centering
    \begin{tabular}{ccccccc}
        & & BERT & - warm up & - grad. clip & + normalize & + GELU\\ \hline
        Overall & min & 99.3\% & 99.4\% & 99.3\% & 99.2\% & 99.3\%\\
        Training & mean & 99.4\% & 99.5\% & 99.4\% & 99.3\% & 99.3\%\\
        Accuracy & max & 99.5\% & 99.6\% & 99.6\% & 99.5\% & 99.5\% \\\hline
        Overall & min & 96.5\% & 96.9\% & 96.9\% & 96.3\% & 96.7\% \\
        Validation & mean & 97.2\% & 97.6\% & 97.2\% & 96.8\% & 97.3\%\\
        Accuracy & max & 98.2\% & 98.2\% & 98.2\% & 98.4\% & 98.2\% \\\hline
        Mean Case & argmin & 99.3\% & 99.6\% & 99.5\% & 99.6\% & 99.5\%\\
        Accuracy & first & 100\% & 100\% & 100\% & 100\% & 100\% \\
        Training & argmax & 96.9\% & 98.1\% & 97.5\% & 96.5\% & 98.0\% \\\hline
        Mean Case & argmin & 98.0\% & 98.0\% & 98.0\% & 98.1\% & 97.9\%\\
        Accuracy & first & 100\% & 100\% & 100\% & 100\% & 100\%\\
        Validation & argmax & 93.9\% & 93.8\% & 93.1\% & 91.2\% & 93.8\%\\\hline
    \end{tabular}
\end{table}

\section{Regularization Experiments}
\label{app:regularization}
To limit the number of variables which are not accounted for in the experiments, we focus on the \emph{infinite data but limited training time} regime. In this regime, every batch consists of new data points. We believe that this regime is of paramount interest in future research, as more devices create a constant stream of data and training is more limited by the available training time than the available data. This regime allows us to omit regularization in all architectures as over-fitting is not an issue. In fact, our supplementary experiments below as well as related work~\cite{ALBERT} show that regularization does not help in this regime. 
We leave a comparison of the architectures in the limited data regime to future work.

Here, we show empirical results supporting the intuition that $L2$ as well as \emph{dropout} regularization does not help in our setup. For each of our tasks, we take our default hyper-parameters ($d=128$, $L=2$, $M=4$, $N=128$) and train 5 random seeds per learning rate for models with regularization, varying the dropout rate in $\{0.0625, 0.125, 0.25, 0.5\}$ and the $L2$ regularization weighting in $\{0.0001,0.001,0.01,0.1\}$. Tables~\ref{tab:regularization_1}, \ref{tab:regularization_2} and \ref{tab:regularization_3} report the best mean accuracy achieved with the
small number behind the accuracies indicating the regularization used, 1 referring to the smallest, 4 to the largest. We underline the results where regularization did lead to an improvement in mean accuracy. Note however that these improvements should be taken with a grain of salt, as (1) none of these improvements is significant considering the performance variation across random seeds and (2) the regularized values are likely to be overestimated, as the max is taken over 40 averages (4 regularization values times 10 learning rates) as compared to 10 averages (10 learning rates) in the unregulated case.
\begin{table}[h]
    \caption{Regularization results in the case distinction task with output taken across all tokens. The top three rows correspond to the best mean training accuracy, while the bottom three rows correspond to the best mean validation accuracy when validating on sequences of half the length.}
    \label{tab:regularization_1}
    \centering
    \begin{tabular}{ccccccc}
        & BERT & MTE & NAP & NON & sum & max\\ \hline
        unregularized & 99.3\% & 99.1\% & 99.3\% & 99.1\% & 98.9\% & 99.2\%\\
        with dropout & $98.1\%^1$ & $97.3\%^1$ & $97.8\%^1$ & $97.5\%^1$ & $97.3\%^1$ & $98.2\%^1$ \\
        with $L2$-regularization & $99.3\%^2$ & \underline{$99.2\%^1$} & $99.2\%^2$ & \underline{$99.2\%^1$} & $98.9\%^1$ & \underline{$99.4\%^2$}\\ \hline
        unregularized & 95.5\% & 95.5\% & 97.0\% & 95.3\% & 75.0\% & 97.1\% \\
        with dropout & $94.4\%^1$ & $94.6\%^1$ & $96.8\%^2$ & \underline{$96.0\%^1$} & \underline{$83.1\%^1$} & $96.3\%^1$ \\
        with $L2$-regularization & \underline{$97.2\%^2$} & $93.6\%^2$ & \underline{$97.1\%^1$} & \underline{$96.1\%^2$} & $67.7\%^2$ & \underline{$97.2\%^2$} \\
    \end{tabular}
\end{table}
\begin{table}[h]
    \caption{Regularization results in the case distinction task with output from the first token. The top three rows correspond to the best mean training accuracy, while the bottom three rows correspond to the best mean validation accuracy when validating on sequences of half the length.}
    \label{tab:regularization_2}
    \centering
    \begin{tabular}{ccccccc}
        & BERT & MTE & NAP & NON & sum & max\\ \hline
        unregularized & 36.6\% & 66.5\% & 94.5\% & 23.2\% & 22.8\% & 97.8\%\\
        with dropout & \underline{$44.9\%^1$} & $44.3\%^1$ & $85.0\%^1$ & $23.2\%^1$ & $22.6\%^1$ & $92.6\%^1$ \\
        with $L2$-regularization & $36.0\%^2$ & $55.3\%^1$ & $93.8\%^2$ & $22.8\%^1$ & $22.8\%^1$ & $95.4\%^1$\\ \hline
        unregularized & 36.7\% & 50.6\% & 83.9\% & 29.6\% & 28.5\% & 88.5\% \\
        with dropout & \underline{$41.4\%^2$} & $40.7\%^1$ & $74.6\%^1$ & $29.6\%^3$ & \underline{$28.9\%^4$} & $87.8\%^1$ \\
        with $L2$-regularization & \underline{$37.2\%^2$} & $45.7\%^1$ & $82.5\%^2$ & $28.9\%^1$ & \underline{$29.0\%^1$} & $81.0\%^1$ \\
    \end{tabular}
\end{table}
\begin{table}[h!]
    \caption{Regularization results in the mode finding task. The top three rows correspond to the best mean training accuracy, while the bottom three rows correspond to the best mean validation accuracy when validating on sequences of twice the length.}
    \label{tab:regularization_3}
    \centering
    \begin{tabular}{ccccccc}
        & BERT & MTE & NAP & NON & sum & max\\ \hline
        unregularized & 99.6\% & 99.8\% & 99.6\% & 98.7\% & 99.8\% & 14.4\%\\
        with dropout & $93.9\%^1$ & $93.3\%^1$ & $94.3\%^1$ & $91.8\%^1$ & $93.3\%^1$ & $24.5\%^1$ \\
        with $L2$-regularization & $99.5\%^1$ & \underline{$99.9\%^2$} & \underline{$99.7\%^1$} & \underline{$98.8\%^1$} & \underline{$99.9\%^4$} & $14.4\%^2$\\ \hline
        unregularized & 95.3\% & 95.4\% & 94.9\% & 91.3\% & 95.8\% & 13.5\% \\
        with dropout & $94.8\%^2$ & $95.4\%^2$ & $93.8\%^1$ & \underline{$92.6\%^1$} & $95.7\%^2$ & $13.4\%^4$ \\
        with $L2$-regularization & $94.7\%^1$ & \underline{$96.0\%^1$} & $94.9\%^1$ & \underline{$94.7\%^2$} & $95.8\%^1$ & \underline{$13.7\%^1$} \\
    \end{tabular}
\end{table}

Overall we note that none of the architectures consistently benefits from regularization in our setup and regularization often decreases mean performance. Further, we point out that the best performance with regularization is most of the times achieved with the smallest regularization.

\section{Case Learning Curves}
\label{app:full_results}
Figures~\ref{fig:case_learning_curves_1}, \ref{fig:case_learning_curves_2} and \ref{fig:case_learning_curves_3} show the case accuracies over the course of training. The corresponding results in the main text are given in Figure~\ref{fig:cases_var_dim} (top row). Besides the observations made in the main text, a few additional insights can be noted: (1) Cases are mostly learned in the order of their occurrences (recall that $~72.37\%$ of the examples are from the \emph{argmin} case, $~20.09\%$ are from the \emph{first} case and $~7.53\%$ are from the \emph{argmax} case). This is to be expected when training with gradient descent, cf. \cite{coherentGradients}. (2) This order is not always given in the \emph{BERT} architecture. Besides the focus on the \emph{first} case if the learning rate is too high - discussed in the main text - we also highlight a curiosity that occurs when the model dimension is too small (see plot highlighted in with red in Figure~\ref{fig:case_learning_curves_1}): The \emph{first} case is learned and then unlearned in favor of the \emph{argmin} case. Note that all 5 random seeds follow this pattern. Note also that for a different learning rate, the opposite holds as seen in the plot just below the highlighted plot.

We highly encourage an interested reader to check out our code release\footnote{\url{https://github.com/OliverRichter/normalized-attention}}, which includes all results as well as visualization scripts to inspect them further.
\begin{figure}
    \centering
    \includegraphics[width=0.94\columnwidth]{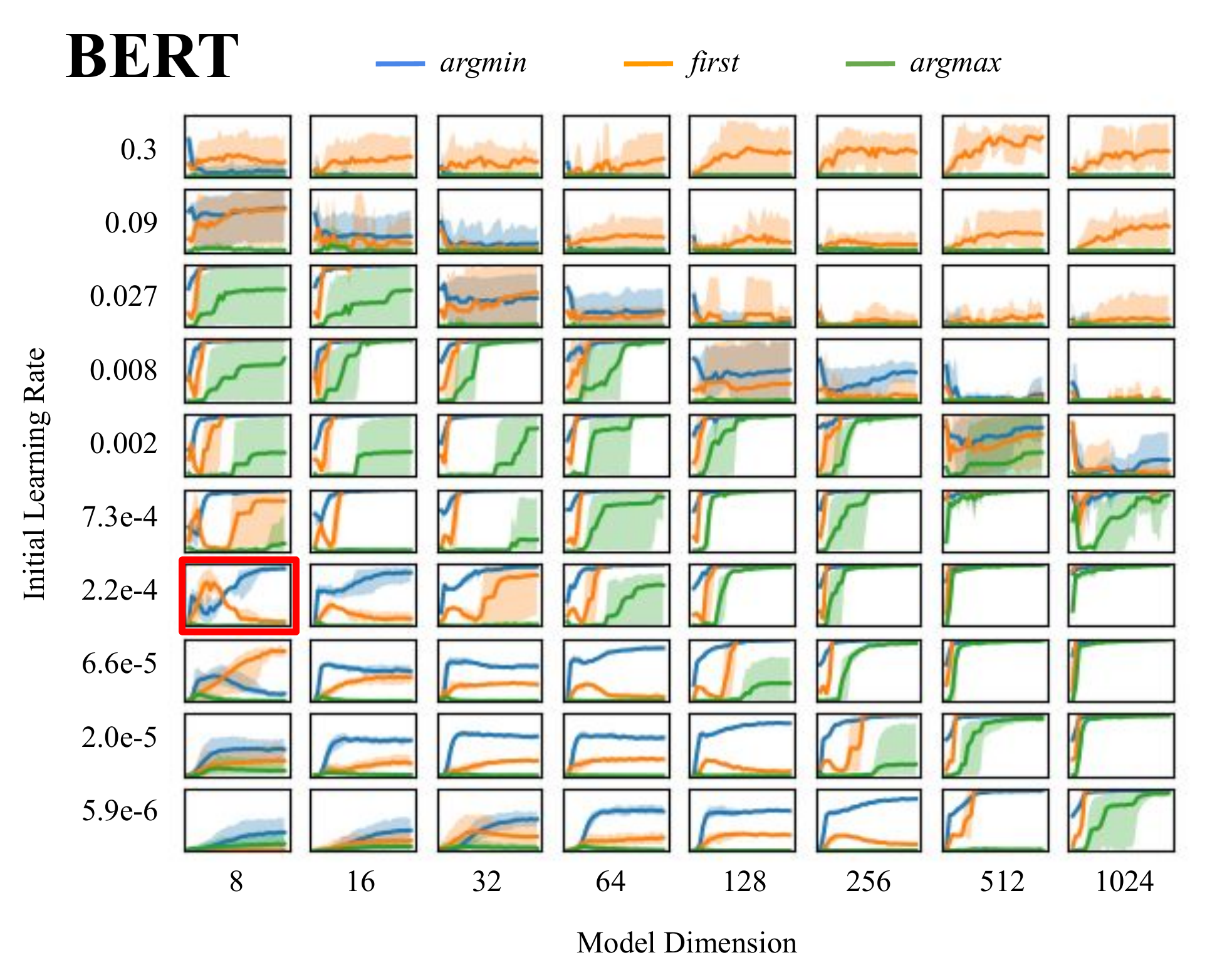}
    \includegraphics[width=0.94\columnwidth]{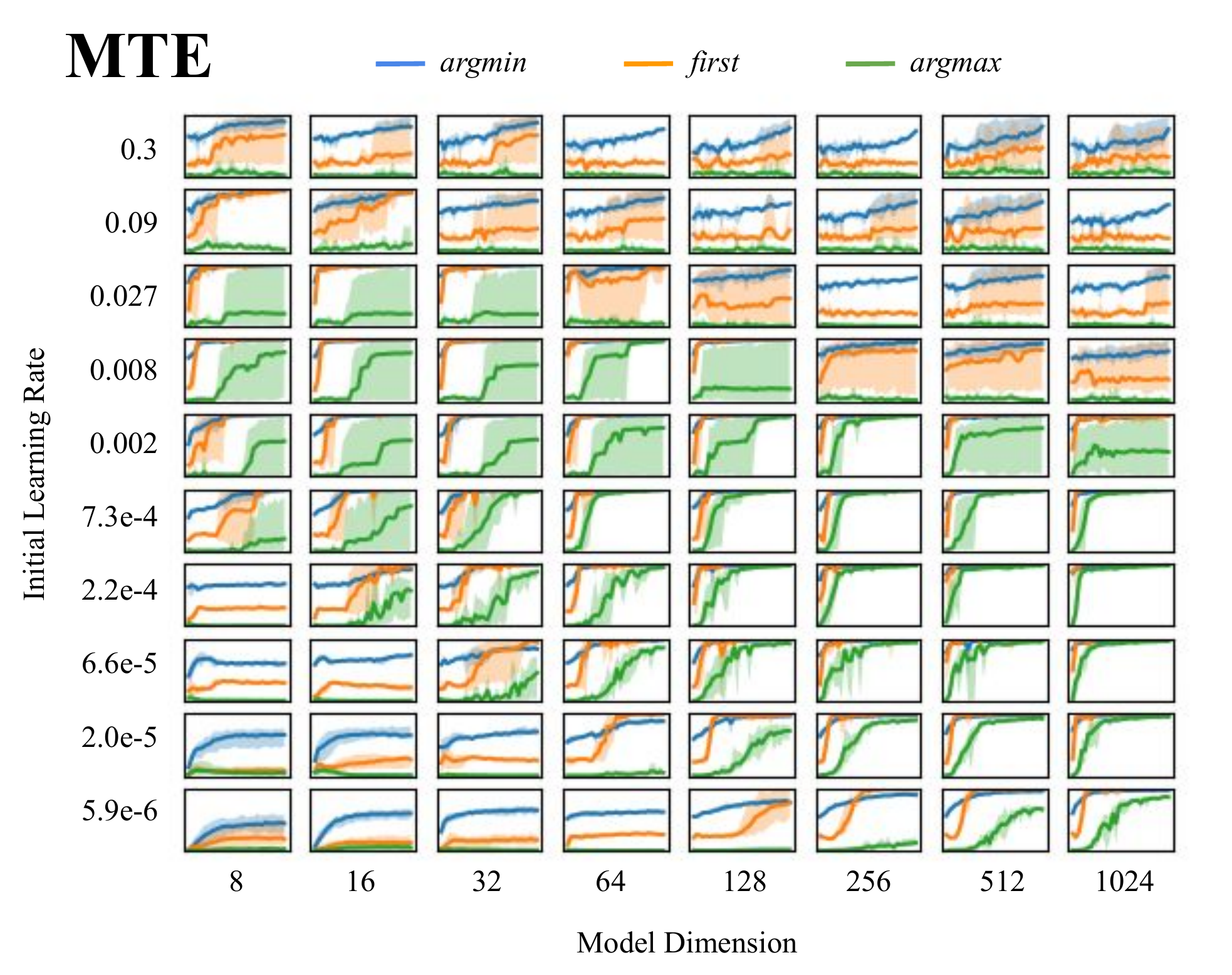}
    \caption{Case accuracies over the course of training on the \emph{argmin}-\emph{first}-\emph{argmax} case distinction task with output across all tokens, cf. Section~\ref{sec:per_token_output}. Each small sub-plot shows the case accuracies (y-axis, bottom is set to 0\%, top to 100\%) over the course of training (x-axis). Solid lines represent the mean accuracy over the 5 random seeds while shaded areas fill the spread between min- and max-accuracy achieved. Models \emph{BERT} and \emph{MTE} are shown here, cf. Figures~\ref{fig:case_learning_curves_2}~and~\ref{fig:case_learning_curves_3}.}
    \label{fig:case_learning_curves_1}
\end{figure}
\begin{figure}
    \centering
    \includegraphics[width=0.94\columnwidth]{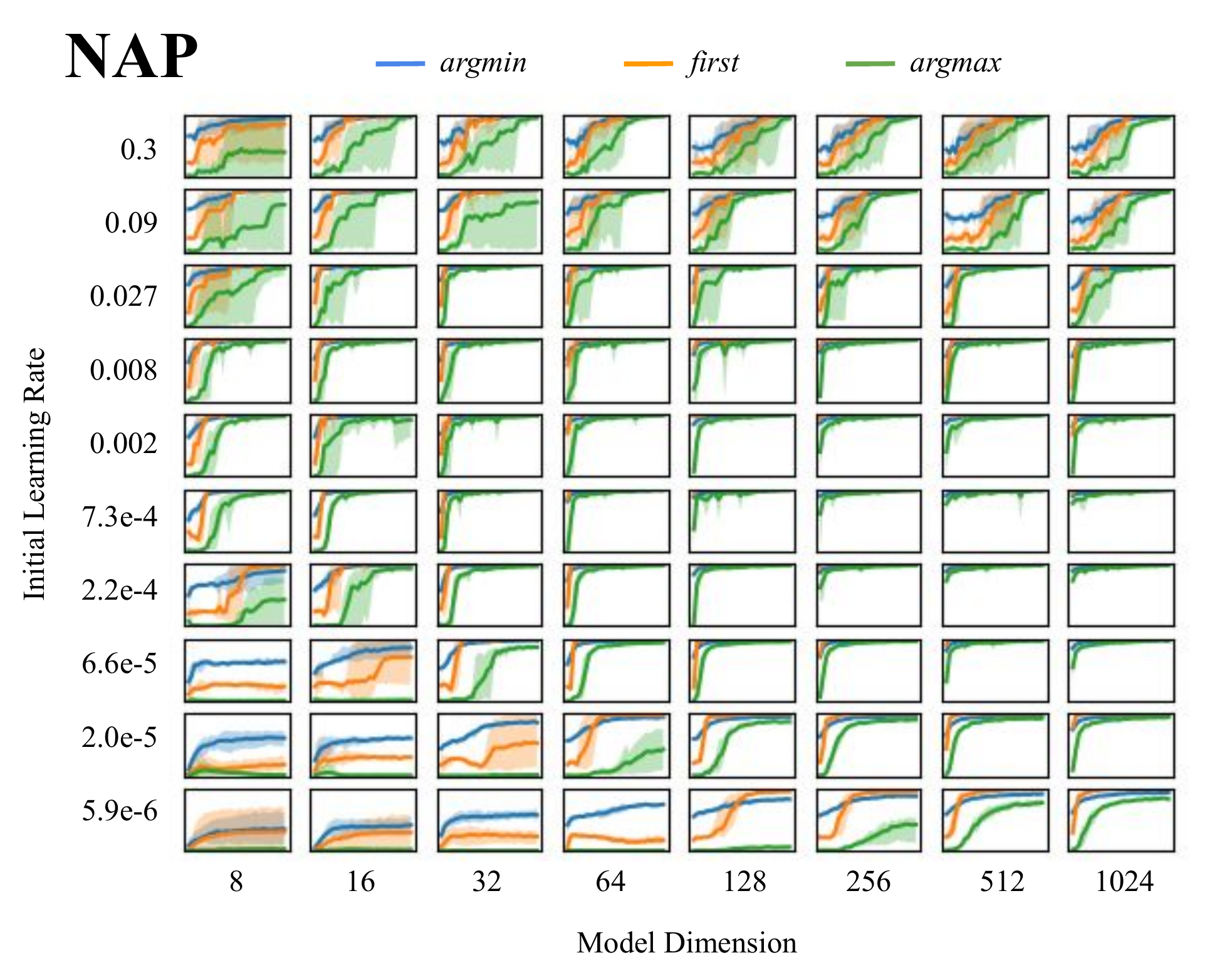}
    \includegraphics[width=0.94\columnwidth]{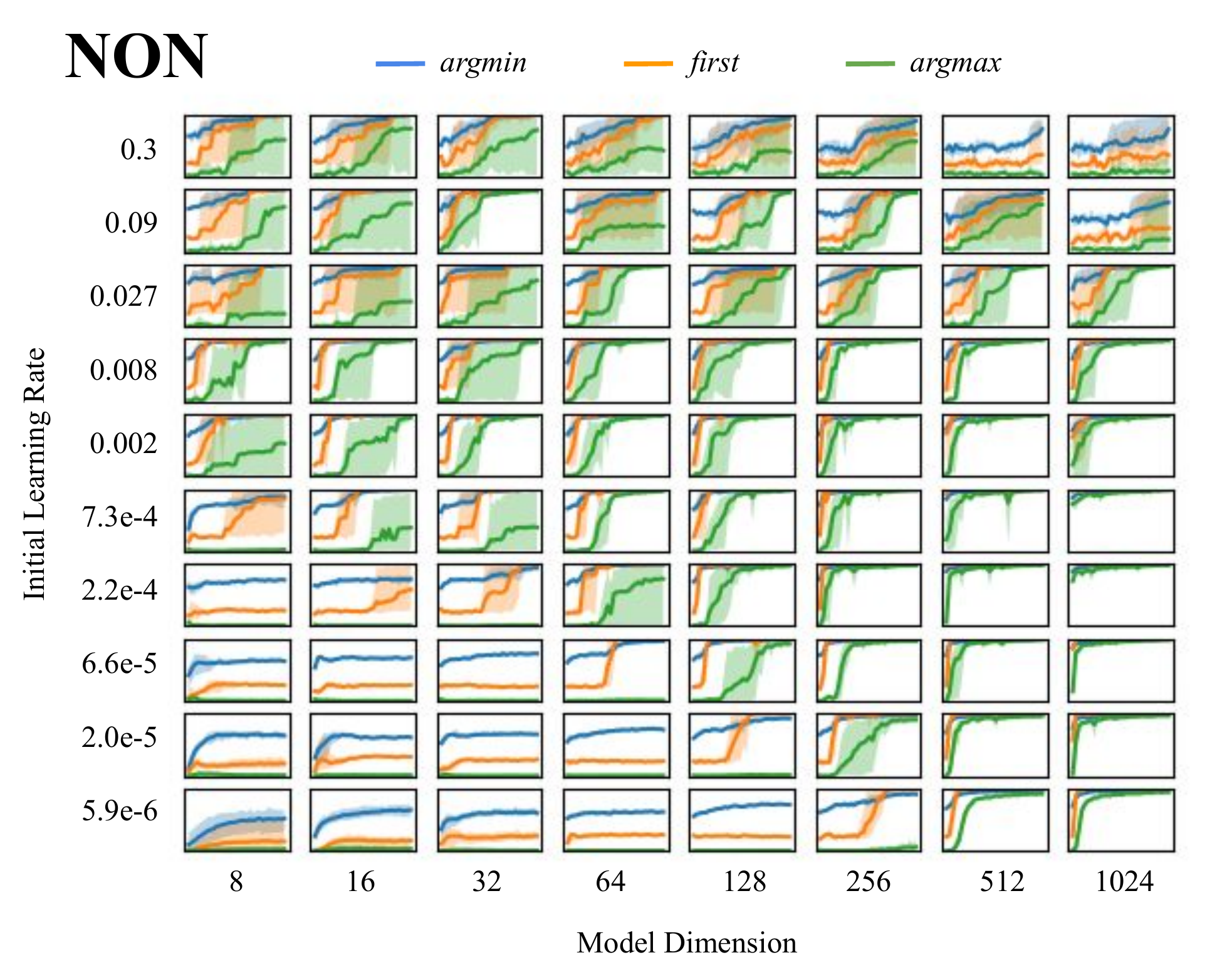}
    \caption{Case accuracies over the course of training on the \emph{argmin}-\emph{first}-\emph{argmax} case distinction task with output across all tokens, cf. Section~\ref{sec:per_token_output}. Each small sub-plot shows the case accuracies (y-axis, bottom is set to 0\%, top to 100\%) over the course of training (x-axis). Solid lines represent the mean accuracy over the 5 random seeds while shaded areas fill the spread between min- and max-accuracy achieved. Models \emph{NAP} and \emph{NON} are shown here, cf. Figures~\ref{fig:case_learning_curves_1}~and~\ref{fig:case_learning_curves_3}.}
    \label{fig:case_learning_curves_2}
\end{figure}
\begin{figure}
    \centering
    \includegraphics[width=0.94\columnwidth]{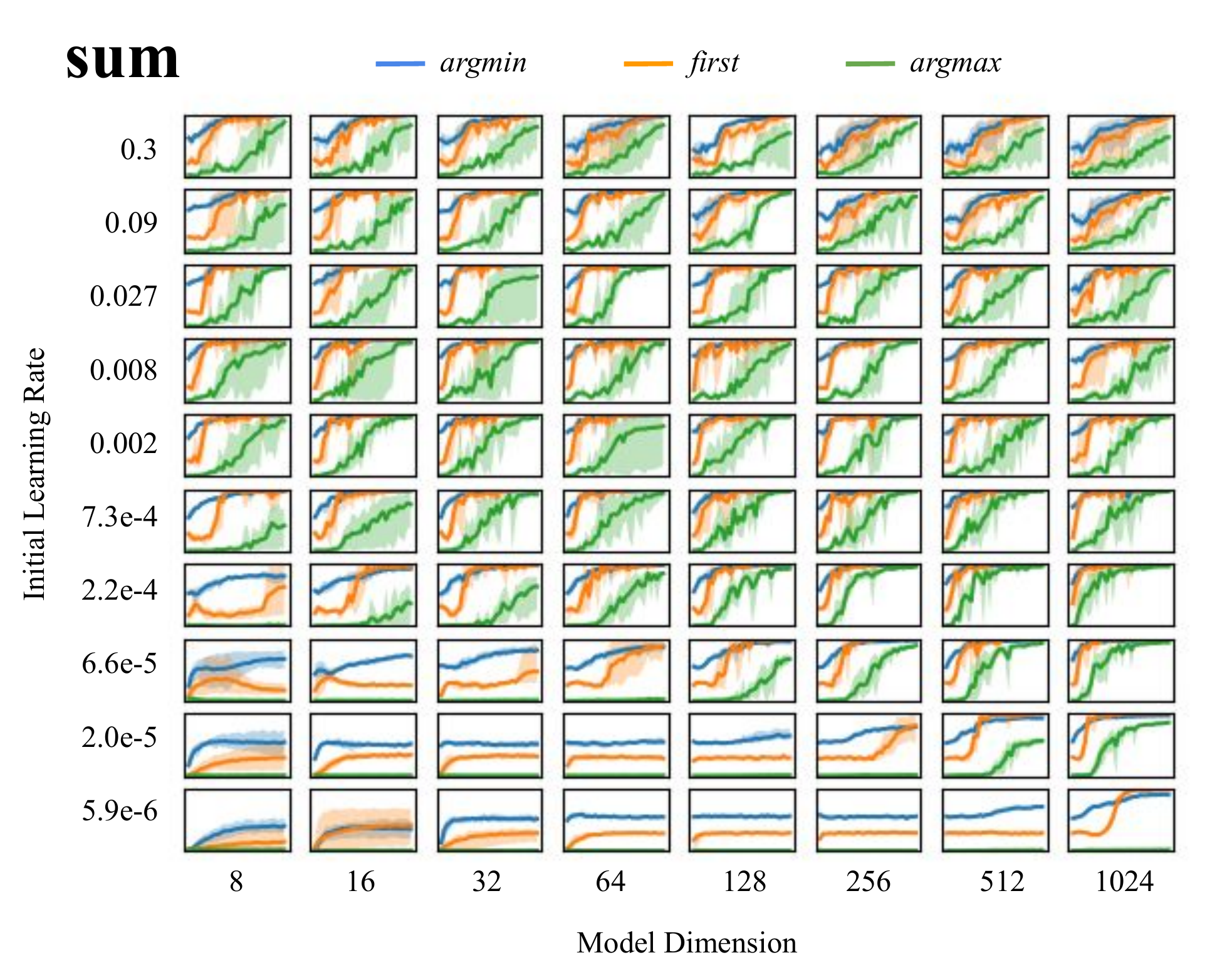}
    \includegraphics[width=0.94\columnwidth]{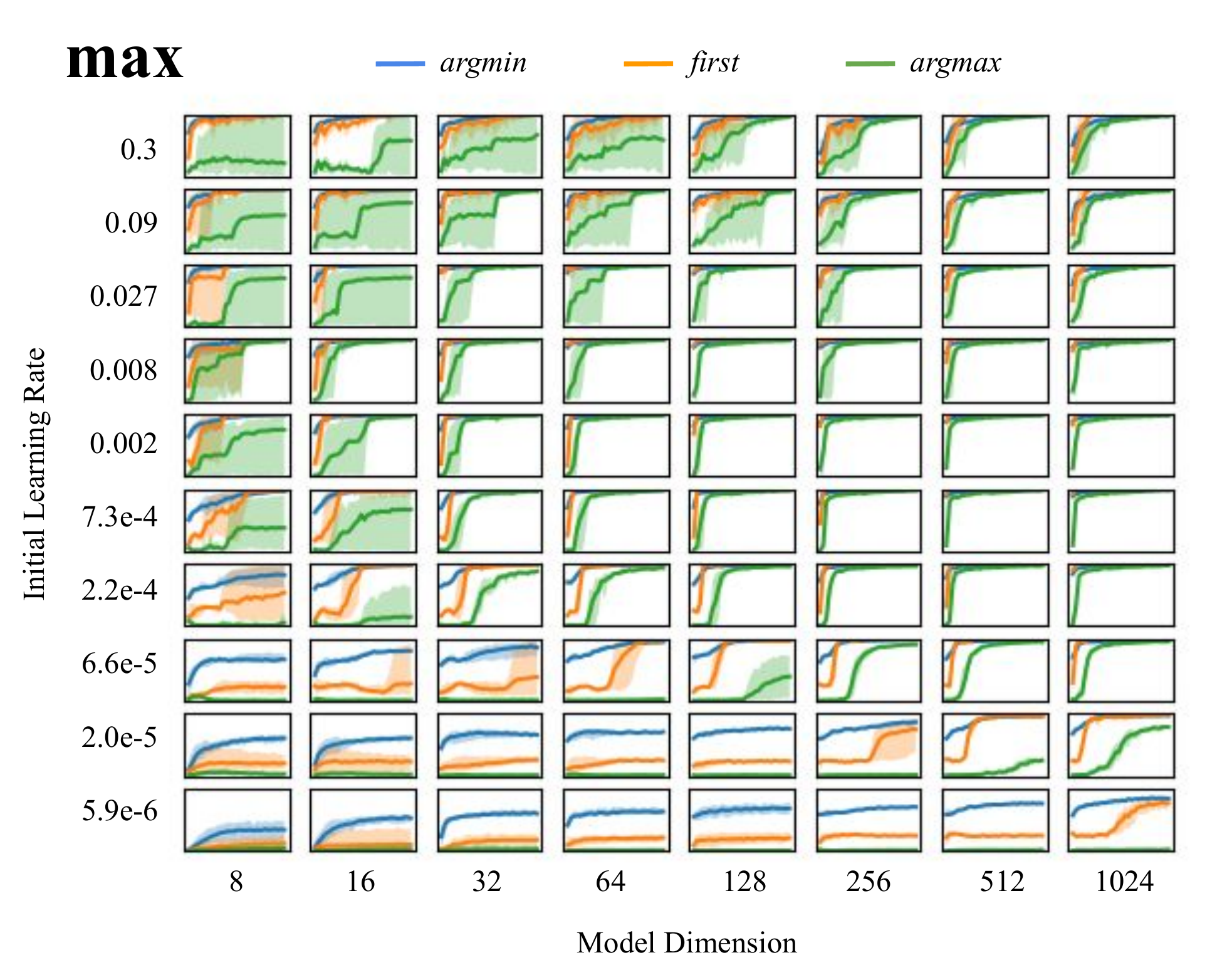}
    \caption{Case accuracies over the course of training on the \emph{argmin}-\emph{first}-\emph{argmax} case distinction task with output across all tokens, cf. Section~\ref{sec:per_token_output}. Each small sub-plot shows the case accuracies (y-axis, bottom is set to 0\%, top to 100\%) over the course of training (x-axis). Solid lines represent the mean accuracy over the 5 random seeds while shaded areas fill the spread between min- and max-accuracy achieved. Models \emph{sum} and \emph{max} are shown here, cf. Figures~\ref{fig:case_learning_curves_1}~and~\ref{fig:case_learning_curves_2}.}
    \label{fig:case_learning_curves_3}
\end{figure}

\newpage
\section{Argmin-First-Argmax Case Distinction Task - Additional Results}

\subsection{Varying Batch Size}
\label{app:var_batch}
In Figure~\ref{fig:cases_var_batch} we provide the case accuracy results of an additional experiment, varying the batch size. In this experiment we train the models using different batch sizes, adjusting the number of training steps accordingly to keep the total number of training points seen constant. With this experiment we aim to show the training behaviour of the different architectures if we go from single example batches (many, potentially noisier updates) to batches of size 128 - a batch size in which each batch contains in expectation several examples per case, but fewer updates are made to the network parameters. Besides replicating several insights made in the main text, this experiment additionally shows: (1) smaller batches require a smaller learning rate, supporting our argument that hyper-parameters should not be optimized independent of each other. (2) The focus of \emph{BERT} on the \emph{first}-case when the learning rate is too high is amplified in smaller batches.
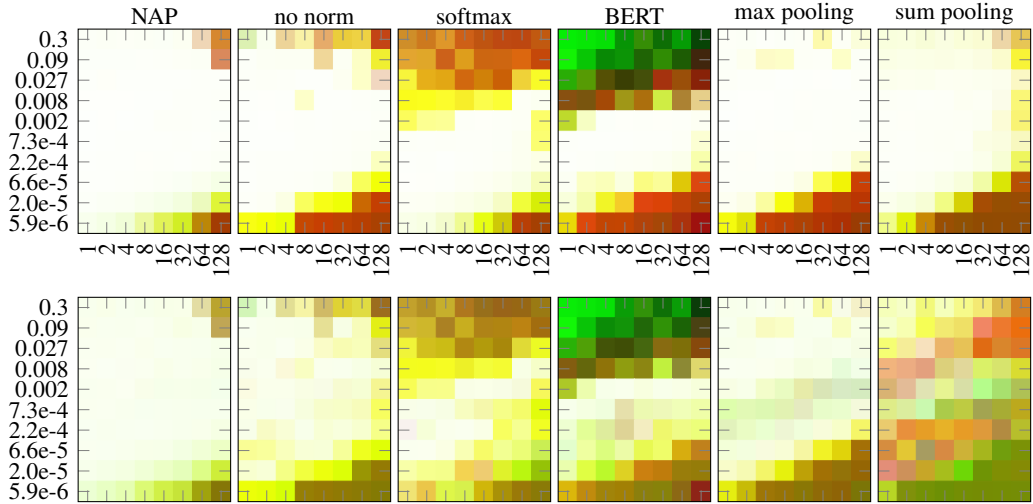
\begin{figure}[h]
    \centering
    \input{plots/cases/var_batch_cases+cases_val}
    \caption{Learning rate (y-axis) vs. batch size (x-axis) on the argmin-first-argmax case distinction task (with output across all tokens). RGB pixel values correspond to \emph{argmin}-, \emph{first}- and \emph{argmax}-case-accuracies, respectively.}
    \label{fig:cases_var_batch}
\end{figure}

\subsection{First Token Output - Varying Model Dimension}
\label{app:first_token_var_dim}
Section~\ref{sec:first_token_output} discusses the case accuracies when training on the case distinction task with outputs taken from the first token. In Figure~\ref{fig:var_dim_first_token_output} we addtionally give best the min-, mean- and max-accuracies over the course of training. The top row corresponds to in-distribution/training accuracy ($N=128$) while the bottom row corresponds to out-of-distribution generalization accrucay when validating on sequences of half the length ($N=64$). Again we note a correlation between optimal learning rate and model dimension, especially in the \emph{BERT} and \emph{MTE} architecture. We also note that these probability simplex constrained architectures have a large performance variation across random seeds in this setup.
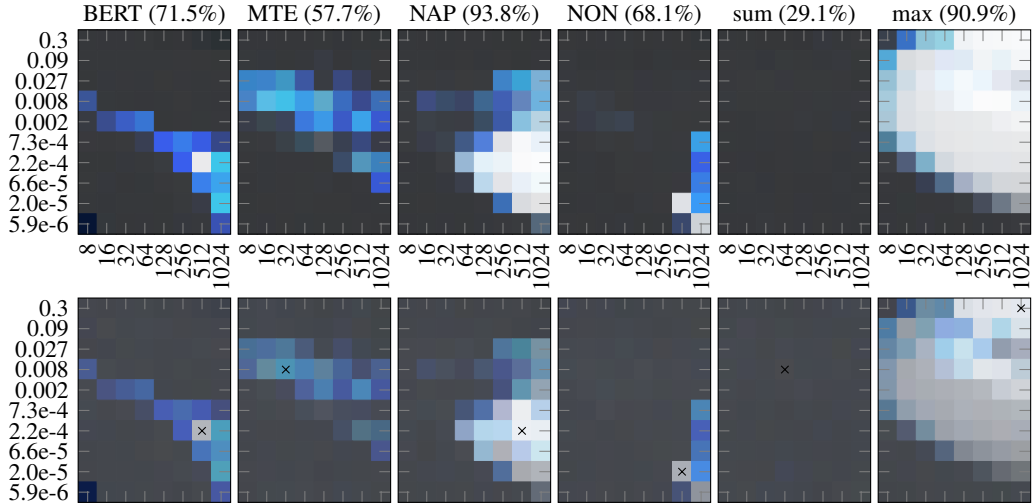
\begin{figure}[h]
    \centering
    \input{plots/cases/first_token_output/var_dim}
    \caption{Learning rate (y-axis) vs. model dimension $d$ (x-axis) on the case distinction task with output from the first token. RGB pixel values correspond to min, mean and max accuracy.
    \textbf{Top row: }Training accuracy (sequence length $N=128$). \textbf{Bottom row: }Validation accuracy when validating on sequences of half the length ($N=64$). Crosses indicate the combination for best mean validation accuracy, which we report behind the model name.}
    \label{fig:var_dim_first_token_output}
\end{figure}

\subsection{First Token output - Varying Depth}
\label{app:first_token_var_depth}
In this section we investigate whether our results are tied to the shallow architecture of $L=2$ Transformer layers. We therefore vary the number of Tranfromer layers $L$ and report the results on the case distinction task with outputs taken from the first token in Figure~\ref{fig:cases_var_layers_first_token_output}. The results lead us to the following observations: (1) The \emph{BERT} architecture does seem to perform better when the number of Transformer layers is increased to $L=4$. However, the performance degrades if we further increase the depth. (2) The \emph{NAP} architecture achieves a higher best mean accuracy and performs well on a wide range of depths. (3) The \emph{max} architecture performs well on the biggest range of hyperparameters. This is due to the beneficial architectural prior as discussed in the main text.
\begin{figure}
    \centering
    \input{plots/cases/first_token_output/var_layers}
    \caption{Learning rate (y-axis) vs. Transformer-layers $L$ (x-axis) on the case distinction task (output from the first token). RGB pixel values correspond to min, mean and max accuracy. \textbf{Top row: }Training accuracy (sequence length $N=128$). \textbf{Bottom row: }Validation accuracy when validating on sequences of half the length ($N=64$). Crosses indicate the combination for best mean validation accuracy, which we report behind the model name.}
    \label{fig:cases_var_layers_first_token_output}
\end{figure}
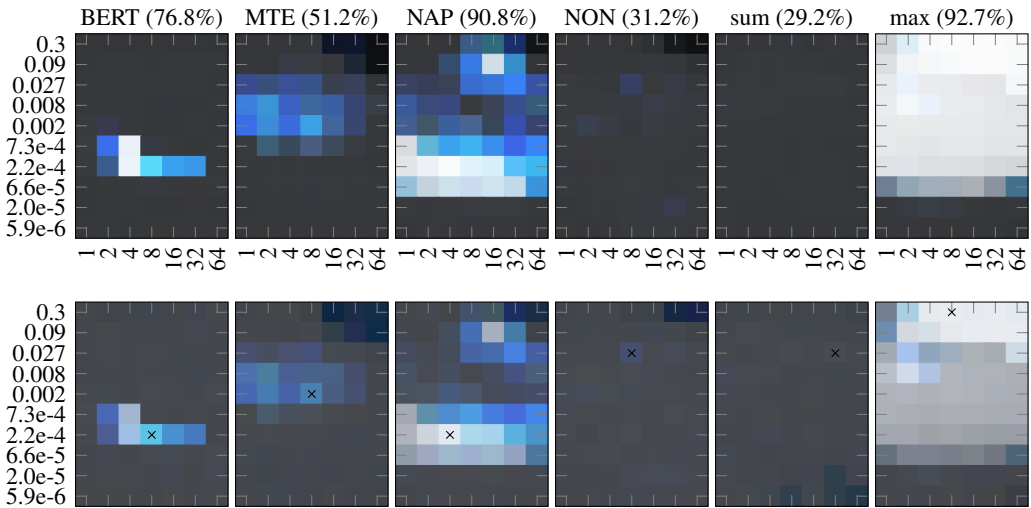
\section{Mode Finding Task - Varying Vocabulary Size}
\label{app:var_vocab}
Figure~\ref{fig:mode_var_vocab} shows the results of an additional experiment, varying the vocabulary size $S$ while keeping the sequence length $N=128$ constant during training. For this experiment, we also vary the total number of training steps and set it to $400\cdot S$, to keep the number of examples seen per vocabulary token approximately constant. We also include zero-shot generalization results when testing on sequences of twice the length ($N=256$). Compared to the case distinction task we can do such a generalization evaluation here as we do not learn any positional embeddings in this setup. We make the following observations:
(1) \emph{max} completely fails to learn in any of the vocabulary sizes. Note that the shading to the left merely corresponds to the majority class base rate.
(2) \emph{NAP} struggles when the vocabulary consists of only 2 tokens. This is expected, as the mean subtraction in the normalization effectively removes the task relevant information (the mode) in this case. Note however, that for a high enough learning rate, the model learns to use the bias parameter $b$ introduced in Equation~\ref{eq:normalize} - effectively reverting to sum pooling.
(3) While all models learn the task well on small vocabularies, \emph{NAP} outperforms all other approaches significantly when $S$ gets larger then the training sequence length, cf. Table~\ref{tab:mode_var_vocab}.
\begin{figure}[h]
    \centering
    \input{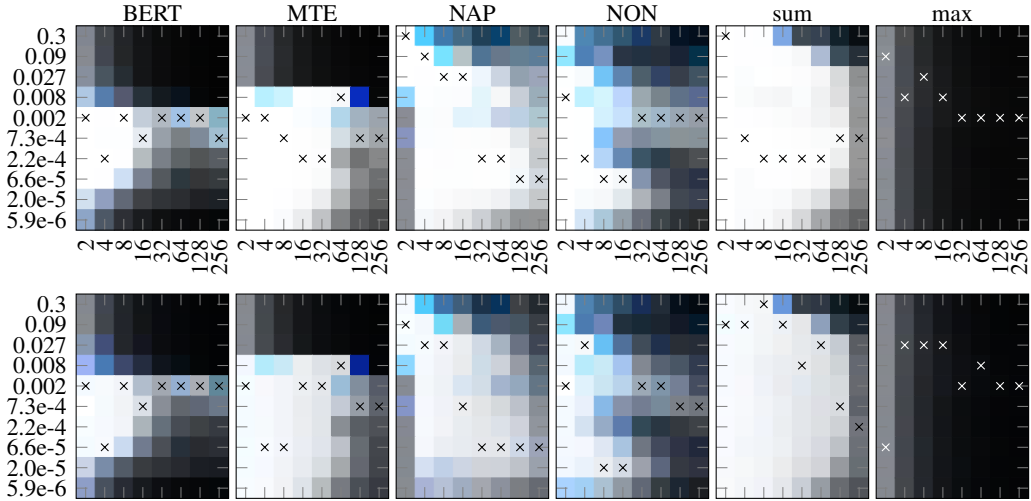}
    \caption{Learning rate (y-axis) vs. vocabulary size $S$ (x-axis) on the mode finding task. RGB pixel values correspond to min, mean and max accuracy. \textbf{Top row: }Training accuracy (sequence length $N= 128$). \textbf{Bottom row: }Validation accuracy when validating on sequences of twice the length ($N= 256$). Crosses indicate the learning rate for best mean accuracy, which we report in Table~\ref{tab:mode_var_vocab}.}
    \label{fig:mode_var_vocab}
\end{figure}
\begin{table}[h]
    \caption{Best mean accuracy per vocabulary size, taken from the combinations indicated in Figure~\ref{fig:mode_var_vocab}. First six rows correspond to training accuracies, bottom six rows correspond to validation accuracies. Bold numbers indicate a min-accuracy higher than the best max accuracy of all other models.}
    \label{tab:mode_var_vocab}
    \centering
    \begin{tabular}{ccccccccc}
        & $S=2$ & $S=4$ & $S=8$ & $S=16$ & $S=32$ & $S=64$ & $S=128$ & $S=256$\\ \hline
        BERT & 100\% & 99.9\% & 99.9 \% & 92.1\% & 72.5\% & 76.2\% & 77.4\% & 74.4\% \\
        MTE & 100\% & 100\% & 99.9\% & 99.8\% & 99.3\% & 97.3\% & 73.3\% & 64.9\% \\
        NAP & 100\% & 99.9\% & 99.8\% & 99.6\% & 99.7\% & \textbf{99.6\%} & 97.4\% & \textbf{84.6\%} \\
        NON & 100\% & 99.9\% & 99.2\% & 97.3\% & 74.7\% & 71.5\% & 65.2\% & 61.5\% \\
        sum & 100\% & 100\% & 99.9\% & 99.8\% & 99.7\% & 99.2\% & 97.5\% & 60.6\% \\
        max & 55.7\% & 30.1\% & 17.3\% & 10.4\% & 6.6\% & 4.6\% & 3.6\% & 3.1\% \\ \hline
        BERT & 100\% & 98.2\% & 95.8 \% & 88.0\% & 65.7\% & 68.6\% & 68.0\% & 53.0\% \\
        MTE & 99.2\% & 98.4\% & 96.1\% & 93.6\% & 90.5\% & 85.4\% & 61.6\% & 38.9\% \\
        NAP & 99.6\% & 98.4\% & 95.8\% & 93.1\% & 90.6\% & 90.4\% & 84.3\% & \textbf{64.3\%} \\
        NON & 100\% & 97.7\% & 93.1\% & 85.7\% & 66.4\% & 58.3\% & 50.2\% & 46.4\% \\
        sum & 99.0\% & 97.9\% & 96.7\% & 94.4\% & 91.6\% & 89.1\% & 85.9\% & 45.6\% \\
        max & 53.8\% & 29.3\% & 16.1\% & 9.5\% & 6.0\% & 4.2\% & 3.0\% & 2.1\% \\
    \end{tabular}
\end{table}

\end{document}

%% file: plots/std_scaling.tex
\begin{tikzpicture}

\definecolor{color0}{rgb}{0.12156862745098,0.466666666666667,0.705882352941177}
\definecolor{color1}{rgb}{1,0.498039215686275,0.0549019607843137}
\definecolor{color2}{rgb}{0.172549019607843,0.627450980392157,0.172549019607843}
\definecolor{color3}{rgb}{0.83921568627451,0.152941176470588,0.156862745098039}
\definecolor{color4}{rgb}{0.580392156862745,0.403921568627451,0.741176470588235}

\begin{axis}[
width=4.3cm,
log basis x={10},
log basis y={10},
tick align=outside,
tick pos=left,
x grid style={white!69.0196078431373!black},
title={$\sigma$},
title style={at={(0.1,0.7)}},
xmin=0.683020128377198, xmax=2998.44750529664,
xmode=log,
xtick style={color=black},
y grid style={white!69.0196078431373!black},
ymin=0.0149407217911633, ymax=65.5895311438091,
ymode=log,
ytick style={color=black}
]
\addplot [semithick, color0]
table {%
1 1.00111916468965
2 0.79826963383494
4 0.625655236187075
8 0.48141040503904
16 0.363615598626218
32 0.269826391091718
64 0.197757402833146
128 0.142857763918753
256 0.102255732762415
512 0.0727782728991785
1024 0.0519473149176169
2048 0.0363392737206586
};
\addplot [semithick, color1]
table {%
1 1.00111916468965
2 0.706795661066525
4 0.500020934640853
8 0.353411001220428
16 0.249904639230572
32 0.176608357888241
64 0.125838629860204
128 0.0888295752733013
256 0.0630432674286595
512 0.0447665596306311
1024 0.0319845752037214
2048 0.0218744970615453
};
\addplot [semithick, color2]
table {%
1 1.00111916468965
2 1.41359132213305
4 2.00008373856341
8 2.82728800976342
16 3.99847422768915
32 5.65146745242371
64 8.05367231105307
128 11.3701856349826
256 16.1390764617368
512 22.9204785308831
1024 32.7522050086107
2048 44.7989699820447
};
\addplot [semithick, color3]
table {%
1 1.00111916468965
2 0.825578143860569
4 0.701934705627453
8 0.609717731297338
16 0.543863553185331
32 0.491538587640635
64 0.45084253204467
128 0.417992513605165
256 0.391039817422694
512 0.366946225487878
1024 0.358025170423726
2048 0.327306233769
};
\addplot [semithick, color4]
table {%
1 0.99899222244216
2 0.9992859257024
4 0.999495361003494
8 0.999642957088773
16 0.999747523031338
32 0.999821272312607
64 0.999874480545435
128 0.999911183083988
256 0.999937535056977
512 0.999956043508122
1024 0.999969113425019
2048 0.999977414608136
};
\end{axis}

\end{tikzpicture}

%% file: plots/norm_scaling.tex
\begin{tikzpicture}

\definecolor{color0}{rgb}{0.12156862745098,0.466666666666667,0.705882352941177}
\definecolor{color1}{rgb}{1,0.498039215686275,0.0549019607843137}
\definecolor{color2}{rgb}{0.172549019607843,0.627450980392157,0.172549019607843}
\definecolor{color3}{rgb}{0.83921568627451,0.152941176470588,0.156862745098039}
\definecolor{color4}{rgb}{0.580392156862745,0.403921568627451,0.741176470588235}

\begin{axis}[
width=4.3cm,
legend cell align={left},
legend style={fill opacity=0.8, draw opacity=1, text opacity=1, at={(1.03,0.0)}, anchor=south west, draw=white!80!black},
log basis x={10},
log basis y={10},
tick align=outside,
tick pos=left,
x grid style={white!69.0196078431373!black},
title={Norm},
title style={at={(0.2,0.7)}},
xmin=0.683020128377198, xmax=2998.44750529664,
xmode=log,
xtick style={color=black},
y grid style={white!69.0196078431373!black},
ymin=0.168616679645055, ymax=740.22454306633,
ymode=log,
ytick style={color=black}
]
\addplot [semithick, color2]
table {%
1 11.3042319645183
2 15.9604909246871
4 22.5847568240704
8 31.92612417481
16 45.1541577741014
32 63.8163664374908
64 90.9033527066598
128 128.39811702508
256 182.320812645018
512 258.985053840978
1024 369.990434850161
2048 505.588262433117
};
\addlegendentry{sum}
\addplot [semithick, color3]
table {%
1 11.3042319645183
2 11.2860779409096
4 14.0731478879684
8 17.5059985448934
16 20.8932384506552
32 24.0277779150012
64 26.9863794911722
128 29.7480744025026
256 32.3505930950084
512 34.6220513862113
1024 37.0757659833119
2048 39.2465622771691
};
\addlegendentry{max}
\addplot [semithick, color4]
table {%
1 11.3023067743871
2 11.3056296584797
4 11.3079991547862
8 11.3096690167363
16 11.3108520468045
32 11.3116864253151
64 11.312288408027
128 11.3127036501146
256 11.3130017887637
512 11.3132111880094
1024 11.3133590572538
2048 11.3134529744294
};
\addlegendentry{normalized}
\addplot [semithick, color0]
table {%
1 11.3042319645183
2 8.97038297738671
4 7.00189372339178
8 5.36807629473375
16 4.04844523424554
32 3.00354560980339
64 2.20367056758581
128 1.595464020624
256 1.14492896698552
512 0.817258202546947
1024 0.583993256439795
2048 0.408760237890229
};
\addlegendentry{attention}
\addplot [semithick, color1]
table {%
1 11.3042319645183
2 7.98024546234356
4 5.6461892060176
8 3.99076552185124
16 2.82213486088134
32 1.99426145117159
64 1.42036488604156
128 1.00311028925844
256 0.712190674394601
512 0.505830183283161
1024 0.36131878403336
2048 0.24686926876617
};
\addlegendentry{mean}
\end{axis}

\end{tikzpicture}

%% file: plots/cases/var_seq_cases.tex

\begin{tikzpicture}

\begin{groupplot}[
group style={
    group size=6 by 1,
    horizontal sep=0.1cm,
    vertical sep=0.85cm,
    xlabels at=edge bottom,
    ylabels at=edge left,
},
height=4.3cm, 
width=1.55*\columnwidth/6,
enlargelimits=false,
xtick=data, ytick=data, xticklabels={4, 8,16,32,64,128,256,512}, xticklabel style={font=\small, rotate=90}, yticklabel style={font=\small}, title style={at={(0.5,0.9)}, font=\small},]

\nextgroupplot[title=BERT, yticklabels={0.3,0.09,0.027,0.008,0.002,7.3e-4,2.2e-4,6.6e-5,2.0e-5,5.9e-6}]
\addplot[
matrix plot,
mesh/cols=8,
mesh/color input=explicit,
]
table[meta=rgb] {
x y rgb
0 0 0.27,0.6718,0.3048
1 0 0.1296,0.3396,0.21319999999999997
2 0 0.0766,0.5894,0.134
3 0 0.1004,0.6140000000000001,0.09260000000000002
4 0 0.040400000000000005,0.5452,0.034999999999999996
5 0 0.0224,0.4074,0.0216
6 0 0.0444,0.45940000000000003,0.014600000000000002
7 0 0.0756,0.4638,0.0126
0 1 0.24579999999999996,0.3962,0.5237999999999999
1 1 0.138,0.4048,0.28559999999999997
2 1 0.0902,0.348,0.1738
3 1 0.12620000000000003,0.1332,0.09440000000000001
4 1 0.152,0.3442,0.11760000000000001
5 1 0.1056,0.30279999999999996,0.0162
6 1 0.05840000000000001,0.1068,0.0096
7 1 0.0812,0.0746,0.012199999999999999
0 2 0.2386,0.3686,0.7236
1 2 0.144,0.4138,0.7198
2 2 0.1606,0.29679999999999995,0.6224000000000001
3 2 0.14479999999999998,0.313,0.4830000000000001
4 2 0.2696,0.3182,0.0946
5 2 0.5326000000000001,0.39880000000000004,0.028000000000000004
6 2 0.5668,0.0458,0.0064
7 2 0.6666,0.048,0.008
0 3 0.5357999999999999,0.46559999999999996,0.8657999999999999
1 3 0.8974,0.9948,0.9890000000000001
2 3 0.9818,1.0,0.9958
3 3 0.5386,0.6714,0.7646
4 3 0.4314,0.5732000000000002,0.16260000000000002
5 3 0.6768,0.5871999999999999,0.1992
6 3 0.8154,0.4282,0.063
7 3 0.7988000000000001,0.022800000000000004,0.0052
0 4 0.9805999999999999,1.0,0.9978
1 4 0.9810000000000001,1.0,0.998
2 4 0.99,1.0,0.998
3 4 0.993,1.0,0.9965999999999999
4 4 0.9934,1.0,0.99
5 4 0.9924000000000002,1.0,0.9751999999999998
6 4 0.9918000000000001,1.0,0.14100000000000001
7 4 0.9904,0.0154,0.0004
0 5 0.9792,1.0,0.998
1 5 0.9802,1.0,0.9978
2 5 0.9893999999999998,1.0,0.9974000000000001
3 5 0.994,1.0,0.9966000000000002
4 5 0.9942,1.0,0.99
5 5 0.9918000000000001,1.0,0.9642
6 5 0.9898,0.9997999999999999,0.3652
7 5 0.9898,0.40219999999999995,0.005599999999999999
0 6 0.9606,1.0,0.9968
1 6 0.9654,1.0,0.9975999999999999
2 6 0.9788,1.0,0.9962
3 6 0.9842000000000001,1.0,0.9928000000000001
4 6 0.9890000000000001,1.0,0.9856
5 6 0.9882,0.9997999999999999,0.951
6 6 0.9896,0.9987999999999999,0.016999999999999998
7 6 0.9872,0.0156,0.0002
0 7 0.42160000000000003,0.9762000000000001,0.9889999999999999
1 7 0.36200000000000004,0.9998000000000001,0.9924
2 7 0.9304,0.9984,0.9917999999999999
3 7 0.913,0.9630000000000001,0.9743999999999999
4 7 0.9425999999999999,0.9916,0.9187999999999998
5 7 0.9789999999999999,0.9934000000000001,0.48460000000000003
6 7 0.9815999999999999,0.3338,0.006
7 7 0.9814,0.0164,0.0012000000000000001
0 8 0.2554,0.5045999999999999,0.984
1 8 0.1312,0.495,0.9673999999999999
2 8 0.2036,0.6090000000000001,0.916
3 8 0.1524,0.8206000000000001,0.6304000000000001
4 8 0.5124,0.7608,0.14120000000000002
5 8 0.8482,0.26139999999999997,0.0322
6 8 0.907,0.0644,0.0036000000000000003
7 8 0.9551999999999999,0.012,0.0004
0 9 0.30100000000000005,0.46959999999999996,0.923
1 9 0.09140000000000001,0.5048,0.8702
2 9 0.0668,0.6043999999999999,0.7860000000000001
3 9 0.1614,0.7744000000000001,0.5791999999999999
4 9 0.4406,0.7605999999999999,0.16460000000000002
5 9 0.666,0.27080000000000004,0.029000000000000005
6 9 0.7173999999999998,0.0716,0.0048000000000000004
7 9 0.7140000000000001,0.0108,0.002
};

\nextgroupplot[title=MTE, yticklabels={}]
\addplot[
matrix plot,
mesh/cols=8,
mesh/color input=explicit,
]
table[meta=rgb] {
x y rgb
0 0 0.2632,0.5304,0.9491999999999999
1 0 0.13340000000000002,0.3246,0.8992000000000001
2 0 0.23200000000000004,0.30279999999999996,0.8274000000000001
3 0 0.19340000000000002,0.3802,0.6073999999999999
4 0 0.5256000000000001,0.3952,0.215
5 0 0.7544000000000001,0.3166,0.0648
6 0 0.9108,0.267,0.022399999999999996
7 0 0.9586,0.0196,0.0019999999999999996
0 1 0.45780000000000004,0.6336,0.9705999999999999
1 1 0.28040000000000004,0.462,0.9231999999999999
2 1 0.2022,0.3028,0.843
3 1 0.2266,0.32059999999999994,0.6274
4 1 0.55,0.42440000000000005,0.20780000000000004
5 1 0.8286000000000001,0.5176,0.058600000000000006
6 1 0.9102,0.10399999999999998,0.005
7 1 0.9715999999999999,0.0252,0.0004
0 2 0.8518000000000001,0.9023999999999999,0.9856
1 2 0.7392000000000001,0.9308,0.9761999999999998
2 2 0.7682,0.526,0.8986000000000001
3 2 0.7722,0.8266,0.8503999999999999
4 2 0.6804,0.57,0.3484
5 2 0.9231999999999999,0.47119999999999995,0.11380000000000001
6 2 0.9356,0.23559999999999998,0.0048000000000000004
7 2 0.9789999999999999,0.017599999999999998,0.0002
0 3 0.9705999999999999,0.9982,0.9970000000000001
1 3 0.9712,0.9992000000000001,0.9960000000000001
2 3 0.9783999999999999,0.9998000000000001,0.9954000000000001
3 3 0.9734,0.9964000000000001,0.9865999999999999
4 3 0.9880000000000001,0.9996,0.9783999999999999
5 3 0.9917999999999999,0.9978,0.18719999999999998
6 3 0.9852000000000001,0.9978,0.029600000000000005
7 3 0.9823999999999999,0.45779999999999993,0.005
0 4 0.9757999999999999,0.9998000000000001,0.9974000000000001
1 4 0.9816,1.0,0.9974000000000001
2 4 0.9902,1.0,0.9970000000000001
3 4 0.991,1.0,0.9954000000000001
4 4 0.9926,0.9997999999999999,0.9898
5 4 0.9917999999999999,0.9972,0.783
6 4 0.9906,0.999,0.028000000000000004
7 4 0.985,0.6986,0.0086
0 5 0.9870000000000001,1.0,0.999
1 5 0.9747999999999999,1.0,0.999
2 5 0.9905999999999999,1.0,0.9982000000000001
3 5 0.9906,1.0,0.9960000000000001
4 5 0.9906,1.0,0.9888
5 5 0.9916,0.9996,0.9494
6 5 0.9890000000000001,0.9948,0.27880000000000005
7 5 0.9874,0.21099999999999994,0.003
0 6 0.9888,1.0,0.9984
1 6 0.9822000000000001,1.0,0.9970000000000001
2 6 0.9865999999999999,1.0,0.9960000000000001
3 6 0.9905999999999999,1.0,0.9952000000000002
4 6 0.9904,1.0,0.99
5 6 0.992,1.0,0.9649999999999999
6 6 0.9896,1.0,0.0134
7 6 0.9865999999999999,0.11080000000000001,0.0016
0 7 0.8934,1.0,0.9987999999999999
1 7 0.9625999999999999,1.0,0.9974000000000001
2 7 0.9798,1.0,0.9944
3 7 0.9846,1.0,0.9890000000000001
4 7 0.986,1.0,0.9798
5 7 0.9858,1.0,0.9254
6 7 0.9890000000000001,1.0,0.0142
7 7 0.9884000000000001,0.019799999999999998,0.0002
0 8 0.24939999999999998,0.9992000000000001,0.9948
1 8 0.16440000000000002,1.0,0.9906
2 8 0.9056,1.0,0.9858
3 8 0.9425999999999999,1.0,0.9601999999999998
4 8 0.9558,1.0,0.9326000000000001
5 8 0.9663999999999999,0.9928000000000001,0.7992
6 8 0.9751999999999998,0.9936,0.013000000000000001
7 8 0.9827999999999999,0.0158,0.0002
0 9 0.19419999999999998,0.757,0.977
1 9 0.08020000000000001,0.546,0.9612
2 9 0.0634,0.8343999999999999,0.9067999999999999
3 9 0.30839999999999995,0.9082000000000001,0.7994
4 9 0.746,0.9008,0.3842
5 9 0.7904,0.7476,0.030199999999999998
6 9 0.8462,0.09519999999999999,0.003999999999999999
7 9 0.8966000000000001,0.0156,0.0002
};
\nextgroupplot[title=NAP, yticklabels={}]
\addplot[
matrix plot,
mesh/cols=8,
mesh/color input=explicit,
]
table[meta=rgb] {
x y rgb
0 0 0.908,0.9964000000000001,0.9934
1 0 0.9705999999999999,0.9998000000000001,0.9948
2 0 0.9906,1.0,0.9945999999999999
3 0 0.9924,1.0,0.9942
4 0 0.9916,1.0,0.9879999999999999
5 0 0.9878,1.0,0.9650000000000001
6 0 0.9843999999999999,0.9771999999999998,0.0994
7 0 0.9583999999999999,0.022,0.0004
0 1 0.9531999999999998,0.9984,0.9955999999999999
1 1 0.9763999999999999,1.0,0.9965999999999999
2 1 0.9869999999999999,1.0,0.9972
3 1 0.9924,1.0,0.9944
4 1 0.9926,1.0,0.9888
5 1 0.9896,0.999,0.783
6 1 0.9894000000000001,0.9986,0.041800000000000004
7 1 0.9752000000000001,0.0222,0.0004
0 2 0.974,0.9996,0.9972000000000001
1 2 0.9730000000000001,1.0,0.9966000000000002
2 2 0.9924,1.0,0.9975999999999999
3 2 0.9947999999999999,1.0,0.9944000000000001
4 2 0.9937999999999999,1.0,0.991
5 2 0.9924,1.0,0.9816
6 2 0.9882,0.9998000000000001,0.5388
7 2 0.9865999999999999,0.795,0.0098
0 3 0.9725999999999999,0.999,0.9959999999999999
1 3 0.9692000000000001,1.0,0.9969999999999999
2 3 0.9894000000000001,1.0,0.9954000000000001
3 3 0.992,1.0,0.9960000000000001
4 3 0.9931999999999999,1.0,0.9907999999999999
5 3 0.993,1.0,0.982
6 3 0.9924,1.0,0.7478
7 3 0.9890000000000001,0.7971999999999999,0.0102
0 4 0.9667999999999999,0.9984,0.9958
1 4 0.9738,0.9997999999999999,0.9975999999999999
2 4 0.9858,1.0,0.9966000000000002
3 4 0.9865999999999999,1.0,0.9916
4 4 0.991,1.0,0.9904
5 4 0.9904,1.0,0.975
6 4 0.992,1.0,0.5092000000000001
7 4 0.9863999999999999,0.7948,0.009
0 5 0.9664000000000001,0.9994,0.9978
1 5 0.9742000000000001,0.9998000000000001,0.998
2 5 0.9827999999999999,1.0,0.9968
3 5 0.9914000000000002,1.0,0.9944
4 5 0.9886000000000001,1.0,0.9885999999999999
5 5 0.9932000000000001,1.0,0.9748000000000001
6 5 0.9902000000000001,1.0,0.9128000000000001
7 5 0.9873999999999998,0.9972,0.0126
0 6 0.9732000000000001,1.0,0.9960000000000001
1 6 0.9712,1.0,0.9956000000000002
2 6 0.9831999999999999,1.0,0.9959999999999999
3 6 0.9884000000000001,1.0,0.9926
4 6 0.9912000000000001,1.0,0.9879999999999999
5 6 0.9934,1.0,0.9743999999999999
6 6 0.9914,0.9998000000000001,0.8952
7 6 0.9895999999999999,0.9974000000000001,0.0126
0 7 0.9199999999999999,1.0,0.9876000000000001
1 7 0.9404,0.9998000000000001,0.9864
2 7 0.9662,1.0,0.9876000000000001
3 7 0.9736,1.0,0.9848000000000001
4 7 0.9823999999999999,1.0,0.9753999999999999
5 7 0.9837999999999999,1.0,0.9513999999999999
6 7 0.9832000000000001,0.9978,0.1668
7 7 0.9751999999999998,0.2958,0.003
0 8 0.7502,0.9918000000000001,0.9491999999999999
1 8 0.7466000000000002,0.9986,0.9536
2 8 0.8704000000000001,0.9992000000000001,0.9564
3 8 0.9192,0.9994,0.9474
4 8 0.9463999999999999,0.9998000000000001,0.9332
5 8 0.9536,0.9978,0.8746
6 8 0.9541999999999999,0.9934000000000001,0.0188
7 8 0.9391999999999999,0.0158,0.0002
0 9 0.2056,0.6963999999999999,0.8363999999999999
1 9 0.1156,0.9276,0.8301999999999999
2 9 0.2986,0.9678000000000001,0.8228
3 9 0.537,0.9470000000000001,0.752
4 9 0.7792,0.9254,0.6602
5 9 0.8382,0.9524000000000001,0.07699999999999999
6 9 0.8512000000000001,0.28300000000000003,0.006
7 9 0.8800000000000001,0.0188,0.0002
};
\nextgroupplot[title=NON, yticklabels={}]
\addplot[
matrix plot,
mesh/cols=8,
mesh/color input=explicit,
]
table[meta=rgb] {
x y rgb
0 0 0.9743999999999999,0.9998000000000001,0.9975999999999999
1 0 0.9778,1.0,0.9979999999999999
2 0 0.9894000000000001,1.0,0.9972000000000001
3 0 0.9912000000000001,1.0,0.9904
4 0 0.9042,0.8815999999999999,0.8144
5 0 0.9773999999999999,0.9766,0.7140000000000001
6 0 0.9289999999999999,0.4398000000000001,0.0316
7 0 0.962,0.0236,0.0006000000000000001
0 1 0.9745999999999999,0.9994,0.9957999999999998
1 1 0.9772000000000001,1.0,0.9972
2 1 0.9884000000000001,1.0,0.9962
3 1 0.9880000000000001,1.0,0.9953999999999998
4 1 0.9094000000000001,0.8686,0.8206
5 1 0.9898,0.9964000000000001,0.8304
6 1 0.9872,0.998,0.1476
7 1 0.975,0.0184,0.0006000000000000001
0 2 0.9812,1.0,0.9969999999999999
1 2 0.9827999999999999,1.0,0.9982
2 2 0.9904,1.0,0.9982
3 2 0.9914,1.0,0.9958
4 2 0.992,1.0,0.9882
5 2 0.9879999999999999,0.9986,0.6778
6 2 0.9890000000000001,0.9996,0.1678
7 2 0.9798,0.0184,0.0002
0 3 0.9836,0.9998000000000001,0.9982
1 3 0.9814,1.0,0.9982
2 3 0.9762000000000001,0.999,0.9928000000000001
3 3 0.9924,1.0,0.9955999999999999
4 3 0.994,1.0,0.9911999999999999
5 3 0.9907999999999999,1.0,0.9762000000000001
6 3 0.9889999999999999,0.9992000000000001,0.0714
7 3 0.9827999999999999,0.21339999999999998,0.0028
0 4 0.9836,1.0,0.9984
1 4 0.9812000000000001,1.0,0.9988000000000001
2 4 0.9894000000000001,1.0,0.9969999999999999
3 4 0.992,1.0,0.9958
4 4 0.9911999999999999,1.0,0.9892
5 4 0.9918000000000001,1.0,0.9746
6 4 0.9898,0.9994,0.16940000000000002
7 4 0.9856,0.7995999999999999,0.0094
0 5 0.9894000000000001,0.9996,0.999
1 5 0.9802,1.0,0.9986
2 5 0.9885999999999999,1.0,0.9972
3 5 0.9936,1.0,0.9958
4 5 0.9945999999999999,1.0,0.9898
5 5 0.993,1.0,0.9728
6 5 0.9888,1.0,0.0512
7 5 0.9792,0.019,0.0004
0 6 0.9803999999999998,1.0,0.998
1 6 0.9793999999999998,1.0,0.9986
2 6 0.9844000000000002,1.0,0.9968
3 6 0.991,1.0,0.9944
4 6 0.9916,1.0,0.9882
5 6 0.9907999999999999,0.9994,0.9713999999999998
6 6 0.9907999999999999,0.9998000000000001,0.0252
7 6 0.9815999999999999,0.016599999999999997,0.0002
0 7 0.9730000000000001,0.9996,0.9951999999999999
1 7 0.9648,0.9998000000000001,0.9955999999999999
2 7 0.9715999999999999,1.0,0.9940000000000001
3 7 0.982,1.0,0.9931999999999999
4 7 0.9848000000000001,1.0,0.986
5 7 0.983,0.9974000000000001,0.6424000000000001
6 7 0.9814,1.0,0.014000000000000002
7 7 0.9803999999999998,0.019,0.0004
0 8 0.7083999999999999,0.9836,0.9904
1 8 0.892,0.9896,0.9843999999999999
2 8 0.9363999999999999,0.9968,0.9752000000000001
3 8 0.8936,0.9978,0.9385999999999999
4 8 0.9238000000000002,0.9324,0.7788
5 8 0.9423999999999999,0.9987999999999999,0.033
6 8 0.9216000000000001,0.10300000000000001,0.0018000000000000002
7 8 0.9751999999999998,0.0184,0.0002
0 9 0.09419999999999999,0.5039999999999999,0.966
1 9 0.0618,0.5256000000000001,0.9292
2 9 0.0872,0.6826,0.8437999999999999
3 9 0.16620000000000001,0.8016,0.6408
4 9 0.5359999999999999,0.7964,0.1666
5 9 0.7496,0.25379999999999997,0.0066
6 9 0.8939999999999999,0.07680000000000001,0.0024000000000000002
7 9 0.9463999999999999,0.016,0.0
};

\nextgroupplot[title=sum, yticklabels={}]
\addplot[
matrix plot,
mesh/cols=8,
mesh/color input=explicit,
]
table[meta=rgb] {
x y rgb
0 0 0.9762000000000001,0.999,0.9972
1 0 0.9715999999999999,0.9992000000000001,0.9974000000000001
2 0 0.9843999999999999,1.0,0.994
3 0 0.9902000000000001,1.0,0.9906
4 0 0.9872,0.9998000000000001,0.9783999999999999
5 0 0.9843999999999999,0.9882,0.8240000000000001
6 0 0.9549999999999998,0.7944,0.0526
7 0 0.9569999999999999,0.0254,0.0006000000000000001
0 1 0.9766,1.0,0.9974000000000001
1 1 0.9763999999999999,1.0,0.9958
2 1 0.986,1.0,0.9949999999999999
3 1 0.9912000000000001,1.0,0.9917999999999999
4 1 0.9902,1.0,0.9833999999999999
5 1 0.9831999999999999,0.9974000000000001,0.9136000000000001
6 1 0.9773999999999999,0.9851999999999999,0.09179999999999999
7 1 0.977,0.024199999999999996,0.0004
0 2 0.9818,0.9998000000000001,0.998
1 2 0.9804,1.0,0.9982
2 2 0.9902000000000001,1.0,0.9964000000000001
3 2 0.9926,1.0,0.9944
4 2 0.991,1.0,0.9894000000000001
5 2 0.992,0.9978,0.9202
6 2 0.9870000000000001,0.9996,0.032200000000000006
7 2 0.9784,0.0206,0.0002
0 3 0.9842000000000001,0.9997999999999999,0.9978000000000001
1 3 0.9815999999999999,0.9998000000000001,0.9982
2 3 0.9890000000000001,1.0,0.9970000000000001
3 3 0.9879999999999999,1.0,0.9934
4 3 0.9895999999999999,1.0,0.9888
5 3 0.9912000000000001,0.9997999999999999,0.959
6 3 0.9882,0.9994,0.0206
7 3 0.9773999999999999,0.096,0.001
0 4 0.9833999999999999,0.9998000000000001,0.9986
1 4 0.978,1.0,0.9986
2 4 0.9902,0.9998000000000001,0.9970000000000001
3 4 0.991,1.0,0.9917999999999999
4 4 0.9917999999999999,1.0,0.9872
5 4 0.9865999999999999,0.9994,0.9527999999999999
6 4 0.985,0.9972,0.0188
7 4 0.977,0.0262,0.0006000000000000001
0 5 0.9804,1.0,0.9978000000000001
1 5 0.9783999999999999,1.0,0.9978
2 5 0.9863999999999999,1.0,0.9963999999999998
3 5 0.9890000000000001,1.0,0.9944000000000001
4 5 0.993,1.0,0.986
5 5 0.9852000000000001,0.9898,0.8081999999999999
6 5 0.9847999999999999,0.9896,0.046
7 5 0.9814,0.0196,0.0004
0 6 0.7786000000000001,0.9996,0.9955999999999999
1 6 0.9768000000000001,1.0,0.9959999999999999
2 6 0.9698,1.0,0.9958
3 6 0.9767999999999999,1.0,0.9869999999999999
4 6 0.985,1.0,0.9732
5 6 0.9798,0.9942,0.8956
6 6 0.975,0.9964000000000001,0.016599999999999997
7 6 0.9792,0.0256,0.0004
0 7 0.11640000000000002,0.9341999999999999,0.8176
1 7 0.2632,0.9586,0.9196
2 7 0.6976000000000001,1.0,0.9822
3 7 0.8666,1.0,0.9501999999999999
4 7 0.9469999999999998,0.9998000000000001,0.8998000000000002
5 7 0.9574,0.9912000000000001,0.6478
6 7 0.9452,0.5742,0.0098
7 7 0.9725999999999999,0.022600000000000002,0.0002
0 8 0.1582,0.7744,0.7121999999999999
1 8 0.09659999999999999,0.8374,0.6357999999999999
2 8 0.0688,0.958,0.6066
3 8 0.15539999999999998,0.9986,0.718
4 8 0.694,0.9986,0.1754
5 8 0.6327999999999999,0.333,0.0134
6 8 0.6844,0.09240000000000001,0.0014
7 8 0.791,0.0162,0.0002
0 9 0.1826,0.8652,0.7076
1 9 0.09440000000000001,0.8492000000000001,0.6157999999999999
2 9 0.07400000000000001,0.9802,0.21980000000000005
3 9 0.0422,1.0,0.176
4 9 0.1274,1.0,0.04560000000000001
5 9 0.5488,0.3032,0.0066
6 9 0.5186,0.08280000000000001,0.0016
7 9 0.3665999999999999,0.0062,0.0002
};
\nextgroupplot[title=max, yticklabels={}]
\addplot[
matrix plot,
mesh/cols=8,
mesh/color input=explicit,
]
table[meta=rgb] {
x y rgb
0 0 0.9682000000000001,1.0,0.9978
1 0 0.9782,1.0,0.9982000000000001
2 0 0.9852000000000001,1.0,0.9975999999999999
3 0 0.9945999999999999,1.0,0.9932000000000001
4 0 0.9916,1.0,0.9879999999999999
5 0 0.9898,1.0,0.9652000000000001
6 0 0.9843999999999999,0.9766,0.2096
7 0 0.9833999999999999,0.9922000000000001,0.0152
0 1 0.9815999999999999,1.0,0.9976
1 1 0.9835999999999998,1.0,0.9984
2 1 0.9904,1.0,0.9965999999999999
3 1 0.9936,1.0,0.9958
4 1 0.9914,1.0,0.9902
5 1 0.9926,1.0,0.9719999999999999
6 1 0.9882,0.9992000000000001,0.6938
7 1 0.9852000000000001,0.9934,0.023200000000000002
0 2 0.9805999999999999,1.0,0.9992000000000001
1 2 0.9805999999999999,1.0,0.9974000000000001
2 2 0.9885999999999999,1.0,0.9974000000000001
3 2 0.9922000000000001,1.0,0.9945999999999999
4 2 0.994,1.0,0.9914
5 2 0.9944,1.0,0.9794
6 2 0.9916,1.0,0.1796
7 2 0.991,0.9982,0.014799999999999999
0 3 0.9822,1.0,0.9987999999999999
1 3 0.9784,1.0,0.9987999999999999
2 3 0.9904,1.0,0.9972000000000001
3 3 0.9924,1.0,0.9958
4 3 0.9936,1.0,0.9914
5 3 0.9934,1.0,0.9753999999999999
6 3 0.99,1.0,0.9024000000000001
7 3 0.9875999999999999,0.9968,0.0204
0 4 0.9856,1.0,0.9982
1 4 0.9823999999999999,1.0,0.9987999999999999
2 4 0.9921999999999999,1.0,0.9958
3 4 0.993,1.0,0.9948
4 4 0.9924,1.0,0.9902000000000001
5 4 0.9944,1.0,0.975
6 4 0.9908000000000001,0.9997999999999999,0.6902
7 4 0.9869999999999999,1.0,0.012199999999999999
0 5 0.9856,1.0,0.9974000000000001
1 5 0.9846,1.0,0.9966000000000002
2 5 0.9867999999999999,1.0,0.9982
3 5 0.9908000000000001,1.0,0.9945999999999999
4 5 0.9942,1.0,0.9898
5 5 0.994,1.0,0.9720000000000001
6 5 0.9890000000000001,0.9994,0.321
7 5 0.9865999999999999,0.9972,0.014400000000000001
0 6 0.9688000000000001,1.0,0.9972
1 6 0.9730000000000001,1.0,0.9968
2 6 0.9869999999999999,1.0,0.9934
3 6 0.9894000000000001,1.0,0.9916
4 6 0.9914,1.0,0.9875999999999999
5 6 0.991,1.0,0.9644
6 6 0.9914,0.9998000000000001,0.014599999999999998
7 6 0.9862,0.018,0.0004
0 7 0.6372,0.9948,0.9827999999999999
1 7 0.7376,0.9958,0.9865999999999999
2 7 0.9158000000000002,0.9992000000000001,0.9802
3 7 0.9686,1.0,0.9712
4 7 0.9710000000000001,0.9974000000000001,0.9339999999999999
5 7 0.9823999999999999,0.991,0.27619999999999995
6 7 0.9823999999999999,0.9696,0.013399999999999999
7 7 0.9869999999999999,0.014599999999999998,0.0002
0 8 0.18619999999999998,0.8767999999999999,0.89
1 8 0.1018,0.8284,0.8942
2 8 0.067,0.8694,0.8382
3 8 0.1916,0.9349999999999999,0.6345999999999999
4 8 0.5474,0.726,0.18
5 8 0.7726,0.2572,0.0196
6 8 0.9126000000000001,0.06960000000000001,0.0036000000000000003
7 8 0.9746,0.014000000000000002,0.0004
0 9 0.1904,0.8596,0.7462000000000001
1 9 0.094,0.7081999999999999,0.696
2 9 0.0664,0.8676,0.5838
3 9 0.0678,0.9308,0.378
4 9 0.40099999999999997,0.725,0.09540000000000001
5 9 0.693,0.21799999999999997,0.0124
6 9 0.7689999999999999,0.0314,0.002
7 9 0.739,0.0112,0.0002
};
\end{groupplot}
\end{tikzpicture}

%% file: plots/cases/first_token_output/var_dim_cases.tex

\begin{tikzpicture}

\begin{groupplot}[
group style={
    group size=6 by 2,
    horizontal sep=0.1cm,
    vertical sep=0.85cm,
    xlabels at=edge bottom,
    ylabels at=edge left,
},
height=4.3cm, 
width=1.55*\columnwidth/6,
enlargelimits=false,
xtick=data, ytick=data, xticklabel style={font=\small, rotate=90}, yticklabel style={font=\small}, title style={at={(0.5,0.9)}, font=\small},]

\nextgroupplot[title=BERT (91.9\%), yticklabels={0.3,0.09,0.027,0.008,0.002,7.3e-4,2.2e-4,6.6e-5,2.0e-5,5.9e-6}, xticklabels={8,16,32,64,128,256,512,1024}]
\addplot[
matrix plot,
mesh/cols=8,
mesh/color input=explicit,
]
table[meta=rgb] {
x y rgb
0 0 0.0166,1.0,0.017400000000000002
1 0 0.0174,1.0,0.016
2 0 0.0162,1.0,0.017400000000000002
3 0 0.019200000000000002,1.0,0.016800000000000002
4 0 0.016800000000000002,1.0,0.017400000000000002
5 0 0.015,1.0,0.016
6 0 0.0152,1.0,0.0164
7 0 0.0126,1.0,0.015400000000000002
0 1 0.017599999999999998,1.0,0.018000000000000002
1 1 0.017,1.0,0.016399999999999998
2 1 0.017,1.0,0.017800000000000003
3 1 0.018400000000000003,1.0,0.0172
4 1 0.0166,1.0,0.019
5 1 0.0172,1.0,0.016399999999999998
6 1 0.0186,1.0,0.0184
7 1 0.016800000000000002,1.0,0.0168
0 2 0.0206,1.0,0.017400000000000002
1 2 0.017800000000000003,1.0,0.019200000000000002
2 2 0.0164,1.0,0.0172
3 2 0.017599999999999998,1.0,0.0166
4 2 0.018,1.0,0.017
5 2 0.0172,1.0,0.0164
6 2 0.0174,1.0,0.016
7 2 0.0156,1.0,0.0164
0 3 0.18780000000000002,1.0,0.018000000000000002
1 3 0.0262,1.0,0.0166
2 3 0.0172,1.0,0.0186
3 3 0.018,1.0,0.0172
4 3 0.019200000000000002,1.0,0.017599999999999998
5 3 0.017599999999999998,1.0,0.017
6 3 0.017400000000000002,1.0,0.016
7 3 0.016800000000000002,1.0,0.0166
0 4 0.039599999999999996,1.0,0.019200000000000002
1 4 0.1312,1.0,0.0212
2 4 0.22959999999999997,1.0,0.019
3 4 0.3676,1.0,0.0166
4 4 0.016,1.0,0.017400000000000002
5 4 0.0186,1.0,0.0176
6 4 0.019200000000000002,1.0,0.017
7 4 0.018000000000000002,1.0,0.017800000000000003
0 5 0.0298,1.0,0.0164
1 5 0.0306,1.0,0.0164
2 5 0.0278,1.0,0.030199999999999998
3 5 0.0648,1.0,0.022600000000000002
4 5 0.2242,1.0,0.0288
5 5 0.39,1.0,0.022600000000000002
6 5 0.2152,1.0,0.0156
7 5 0.094,1.0,0.0162
0 6 0.016599999999999997,1.0,0.0172
1 6 0.018199999999999997,1.0,0.0186
2 6 0.020999999999999998,1.0,0.0172
3 6 0.028599999999999997,1.0,0.019
4 6 0.039599999999999996,1.0,0.027800000000000002
5 6 0.2234,1.0,0.026600000000000002
6 6 0.993,1.0,0.0212
7 6 0.789,1.0,0.0172
0 7 0.0162,1.0,0.016
1 7 0.017,1.0,0.0166
2 7 0.019,1.0,0.0164
3 7 0.016,1.0,0.018
4 7 0.0308,1.0,0.017599999999999998
5 7 0.029600000000000005,1.0,0.0238
6 7 0.41080000000000005,1.0,0.0232
7 7 0.611,1.0,0.0222
0 8 0.018,1.0,0.0176
1 8 0.018,1.0,0.016599999999999997
2 8 0.0166,1.0,0.017800000000000003
3 8 0.0162,1.0,0.0176
4 8 0.0162,1.0,0.017400000000000002
5 8 0.027800000000000002,1.0,0.019
6 8 0.0434,1.0,0.0256
7 8 0.7954,1.0,0.022799999999999997
0 9 0.012,0.37100000000000005,0.0128
1 9 0.0158,0.983,0.0166
2 9 0.0166,1.0,0.016800000000000002
3 9 0.0196,1.0,0.0174
4 9 0.0188,1.0,0.0172
5 9 0.017,1.0,0.0182
6 9 0.027199999999999995,1.0,0.016800000000000002
7 9 0.1674,1.0,0.020399999999999998
};
\addplot[
only marks,
mark=x,
]
table {
x y
6 6
};

\nextgroupplot[title=MTE (75.4\%), yticklabels={}, xticklabels={8,16,32,64,128,256,512,1024}]
\addplot[
matrix plot,
mesh/cols=8,
mesh/color input=explicit,
]
table[meta=rgb] {
x y rgb
0 0 0.017400000000000002,1.0,0.017800000000000003
1 0 0.017599999999999998,1.0,0.016800000000000002
2 0 0.018,1.0,0.0164
3 0 0.017800000000000003,1.0,0.017
4 0 0.0206,0.9996,0.0156
5 0 0.020999999999999998,0.998,0.017800000000000003
6 0 0.0196,0.9984,0.0162
7 0 0.0212,1.0,0.017
0 1 0.0242,1.0,0.0184
1 1 0.0194,1.0,0.018400000000000003
2 1 0.020399999999999998,1.0,0.0174
3 1 0.027000000000000003,0.994,0.017
4 1 0.049600000000000005,0.9823999999999999,0.016800000000000002
5 1 0.0198,1.0,0.0198
6 1 0.0256,0.9917999999999999,0.0164
7 1 0.0216,0.9994,0.0172
0 2 0.44279999999999997,1.0,0.0182
1 2 0.4814,1.0,0.0166
2 2 0.5464,1.0,0.017
3 2 0.2016,1.0,0.021200000000000004
4 2 0.0482,1.0,0.0212
5 2 0.1242,0.9986,0.0164
6 2 0.079,0.9763999999999999,0.0172
7 2 0.0366,0.9729999999999999,0.017400000000000002
0 3 0.43260000000000004,1.0,0.0182
1 3 0.76,1.0,0.0194
2 3 0.7647999999999999,1.0,0.0224
3 3 0.5838,1.0,0.0196
4 3 0.6348,1.0,0.022399999999999996
5 3 0.31060000000000004,1.0,0.0206
6 3 0.14279999999999998,1.0,0.017800000000000003
7 3 0.3588,1.0,0.017400000000000002
0 4 0.0456,1.0,0.0184
1 4 0.11640000000000002,1.0,0.0188
2 4 0.0818,1.0,0.0354
3 4 0.392,1.0,0.023
4 4 0.5551999999999999,1.0,0.021800000000000003
5 4 0.20480000000000004,1.0,0.024
6 4 0.5996,1.0,0.02
7 4 0.198,1.0,0.0194
0 5 0.0258,1.0,0.017400000000000002
1 5 0.032,1.0,0.0188
2 5 0.07560000000000001,1.0,0.0188
3 5 0.152,1.0,0.017800000000000003
4 5 0.227,1.0,0.018999999999999996
5 5 0.0632,1.0,0.024399999999999998
6 5 0.127,1.0,0.024
7 5 0.0446,1.0,0.023799999999999998
0 6 0.016999999999999998,1.0,0.018400000000000003
1 6 0.02,1.0,0.018400000000000003
2 6 0.0186,1.0,0.0198
3 6 0.0228,1.0,0.018000000000000002
4 6 0.0274,1.0,0.017800000000000003
5 6 0.12799999999999997,1.0,0.016800000000000002
6 6 0.5426,1.0,0.0168
7 6 0.42399999999999993,1.0,0.0196
0 7 0.017,1.0,0.016800000000000002
1 7 0.0172,1.0,0.017800000000000003
2 7 0.0164,1.0,0.0166
3 7 0.025,1.0,0.0214
4 7 0.024,1.0,0.018800000000000004
5 7 0.0186,1.0,0.0188
6 7 0.019,1.0,0.017400000000000002
7 7 0.20260000000000003,1.0,0.017
0 8 0.0172,1.0,0.016800000000000002
1 8 0.016599999999999997,1.0,0.0196
2 8 0.018000000000000002,1.0,0.0178
3 8 0.019,1.0,0.0174
4 8 0.021,1.0,0.0164
5 8 0.018600000000000002,1.0,0.0182
6 8 0.017400000000000002,1.0,0.0174
7 8 0.017,1.0,0.018400000000000003
0 9 0.0158,1.0,0.018799999999999997
1 9 0.0166,1.0,0.0176
2 9 0.0182,1.0,0.0172
3 9 0.016800000000000002,1.0,0.0172
4 9 0.0176,1.0,0.0182
5 9 0.0164,1.0,0.017
6 9 0.018000000000000002,1.0,0.017400000000000002
7 9 0.018,1.0,0.0172
};
\addplot[
only marks,
mark=x,
]
table {
x y
2 3
};

\nextgroupplot[title=NAP (98.7\%), yticklabels={}, xticklabels={8,16,32,64,128,256,512,1024}]
\addplot[
matrix plot,
mesh/cols=8,
mesh/color input=explicit,
]
table[meta=rgb] {
x y rgb
0 0 0.019200000000000002,1.0,0.016
1 0 0.0178,1.0,0.0182
2 0 0.018000000000000002,1.0,0.0182
3 0 0.02,1.0,0.0166
4 0 0.0216,1.0,0.016999999999999998
5 0 0.0218,1.0,0.0172
6 0 0.021,1.0,0.0164
7 0 0.0206,1.0,0.0166
0 1 0.0172,1.0,0.0172
1 1 0.0162,1.0,0.0182
2 1 0.017400000000000002,1.0,0.0188
3 1 0.017599999999999998,1.0,0.017599999999999998
4 1 0.0216,1.0,0.018400000000000003
5 1 0.0196,1.0,0.018000000000000002
6 1 0.0196,1.0,0.0182
7 1 0.019999999999999997,1.0,0.016999999999999998
0 2 0.030199999999999998,1.0,0.028199999999999996
1 2 0.031,1.0,0.0274
2 2 0.0298,1.0,0.027000000000000003
3 2 0.028999999999999998,1.0,0.028999999999999998
4 2 0.0344,1.0,0.0294
5 2 0.4338,1.0,0.0242
6 2 0.6182,1.0,0.0224
7 2 0.6906,0.9955999999999999,0.016800000000000002
0 3 0.049400000000000006,1.0,0.0264
1 3 0.1392,1.0,0.031
2 3 0.13640000000000002,1.0,0.0268
3 3 0.08360000000000001,1.0,0.024399999999999998
4 3 0.29300000000000004,1.0,0.0262
5 3 0.22820000000000001,1.0,0.024399999999999998
6 3 0.4582,1.0,0.024
7 3 0.6952,1.0,0.024200000000000003
0 4 0.0334,1.0,0.023000000000000003
1 4 0.0752,1.0,0.023799999999999998
2 4 0.054200000000000005,1.0,0.0246
3 4 0.0464,1.0,0.028200000000000003
4 4 0.13620000000000002,1.0,0.0222
5 4 0.20859999999999998,1.0,0.026399999999999996
6 4 0.7806000000000001,1.0,0.0232
7 4 0.8634000000000001,1.0,0.0194
0 5 0.035199999999999995,1.0,0.017599999999999998
1 5 0.0536,1.0,0.036
2 5 0.08320000000000001,1.0,0.0306
3 5 0.1976,1.0,0.0368
4 5 0.3902,1.0,0.24459999999999998
5 5 0.9394,1.0,0.14600000000000002
6 5 0.9798,1.0,0.8809999999999999
7 5 0.984,1.0,0.43839999999999996
0 6 0.0162,1.0,0.0182
1 6 0.0172,1.0,0.018400000000000003
2 6 0.04440000000000001,1.0,0.029600000000000005
3 6 0.6834,1.0,0.31339999999999996
4 6 0.9602,1.0,0.67
5 6 0.9698,1.0,0.6542
6 6 0.983,1.0,0.93
7 6 0.9888,1.0,0.9501999999999999
0 7 0.018000000000000002,1.0,0.0176
1 7 0.0164,1.0,0.0164
2 7 0.0178,1.0,0.0174
3 7 0.061,1.0,0.0262
4 7 0.8522000000000001,1.0,0.38520000000000004
5 7 0.953,1.0,0.045
6 7 0.9810000000000001,1.0,0.43479999999999996
7 7 0.9823999999999999,1.0,0.8173999999999999
0 8 0.0188,1.0,0.0176
1 8 0.0162,1.0,0.0176
2 8 0.017400000000000002,1.0,0.0182
3 8 0.0212,1.0,0.017400000000000002
4 8 0.0332,1.0,0.021800000000000003
5 8 0.333,1.0,0.0288
6 8 0.9106,1.0,0.019799999999999998
7 8 0.9469999999999998,1.0,0.0688
0 9 0.0166,1.0,0.0166
1 9 0.0166,1.0,0.0188
2 9 0.0166,1.0,0.016
3 9 0.0176,1.0,0.019
4 9 0.0186,1.0,0.0168
5 9 0.0256,1.0,0.018400000000000003
6 9 0.028000000000000004,1.0,0.0186
7 9 0.2704,1.0,0.019
};
\addplot[
only marks,
mark=x,
]
table {
x y
7 6
};

\nextgroupplot[title=NON (88.4\%), yticklabels={}, xticklabels={8,16,32,64,128,256,512,1024}]
\addplot[
matrix plot,
mesh/cols=8,
mesh/color input=explicit,
]
table[meta=rgb] {
x y rgb
0 0 0.0164,1.0,0.017599999999999998
1 0 0.018,1.0,0.0164
2 0 0.017,1.0,0.018400000000000003
3 0 0.0186,1.0,0.0184
4 0 0.020399999999999998,1.0,0.0186
5 0 0.019200000000000002,1.0,0.019
6 0 0.019600000000000003,1.0,0.0172
7 0 0.022199999999999998,1.0,0.017599999999999998
0 1 0.0228,1.0,0.0196
1 1 0.0176,1.0,0.017
2 1 0.0166,1.0,0.0172
3 1 0.0182,1.0,0.0166
4 1 0.019,1.0,0.0184
5 1 0.021599999999999998,1.0,0.016999999999999998
6 1 0.020200000000000003,1.0,0.0164
7 1 0.021799999999999996,1.0,0.017599999999999998
0 2 0.0258,1.0,0.0258
1 2 0.029599999999999998,1.0,0.0242
2 2 0.027600000000000003,1.0,0.0254
3 2 0.0298,1.0,0.028599999999999997
4 2 0.024,1.0,0.0232
5 2 0.023600000000000003,1.0,0.0228
6 2 0.0242,1.0,0.019200000000000002
7 2 0.0246,1.0,0.0248
0 3 0.0328,1.0,0.031
1 3 0.06219999999999999,1.0,0.0256
2 3 0.0534,1.0,0.0298
3 3 0.029400000000000003,1.0,0.0268
4 3 0.0318,1.0,0.0306
5 3 0.0284,1.0,0.027000000000000003
6 3 0.0304,1.0,0.0236
7 3 0.028000000000000004,1.0,0.027200000000000002
0 4 0.035800000000000005,1.0,0.0174
1 4 0.0432,1.0,0.023200000000000002
2 4 0.08319999999999998,1.0,0.022
3 4 0.0982,1.0,0.0392
4 4 0.0404,1.0,0.027999999999999997
5 4 0.0308,1.0,0.0264
6 4 0.027800000000000002,1.0,0.0264
7 4 0.028800000000000003,1.0,0.030199999999999998
0 5 0.017,1.0,0.0182
1 5 0.0242,1.0,0.0172
2 5 0.0392,1.0,0.0236
3 5 0.029400000000000003,1.0,0.027000000000000003
4 5 0.038400000000000004,1.0,0.027399999999999997
5 5 0.0298,1.0,0.026799999999999997
6 5 0.0306,1.0,0.029599999999999998
7 5 0.5853999999999999,1.0,0.023
0 6 0.018,1.0,0.0162
1 6 0.019,1.0,0.0184
2 6 0.016999999999999998,1.0,0.0176
3 6 0.0298,1.0,0.0172
4 6 0.0258,1.0,0.027800000000000002
5 6 0.029199999999999997,1.0,0.025
6 6 0.028399999999999998,1.0,0.028999999999999998
7 6 0.24939999999999998,1.0,0.0256
0 7 0.015799999999999998,1.0,0.016800000000000002
1 7 0.019799999999999998,1.0,0.0164
2 7 0.0172,1.0,0.017599999999999998
3 7 0.027200000000000002,1.0,0.0174
4 7 0.027200000000000002,1.0,0.0218
5 7 0.028399999999999998,1.0,0.0312
6 7 0.0256,1.0,0.028399999999999998
7 7 0.3998,1.0,0.0254
0 8 0.0176,1.0,0.0178
1 8 0.0164,1.0,0.017
2 8 0.0166,1.0,0.0166
3 8 0.017,1.0,0.0166
4 8 0.0258,1.0,0.0188
5 8 0.030199999999999998,1.0,0.0228
6 8 0.9440000000000002,1.0,0.018
7 8 0.541,1.0,0.20220000000000002
0 9 0.0162,1.0,0.0164
1 9 0.0166,1.0,0.019200000000000002
2 9 0.0168,1.0,0.0182
3 9 0.0182,1.0,0.016999999999999998
4 9 0.0166,1.0,0.017
5 9 0.0204,1.0,0.0186
6 9 0.1062,1.0,0.017800000000000003
7 9 0.877,1.0,0.0164
};
\addplot[
only marks,
mark=x,
]
table {
x y
6 8
};

\nextgroupplot[title=sum (23.0\%), yticklabels={}, xticklabels={8,16,32,64,128,256,512,1024}]
\addplot[
matrix plot,
mesh/cols=8,
mesh/color input=explicit,
]
table[meta=rgb] {
x y rgb
0 0 0.019200000000000002,1.0,0.016
1 0 0.0184,1.0,0.0172
2 0 0.0186,1.0,0.0202
3 0 0.019200000000000002,1.0,0.0172
4 0 0.0182,1.0,0.017
5 0 0.024399999999999998,1.0,0.0186
6 0 0.0216,1.0,0.0172
7 0 0.0222,1.0,0.019
0 1 0.025,1.0,0.0182
1 1 0.0254,1.0,0.0176
2 1 0.019200000000000002,1.0,0.0196
3 1 0.027200000000000002,1.0,0.0178
4 1 0.0258,1.0,0.020800000000000003
5 1 0.023,1.0,0.018
6 1 0.026600000000000002,1.0,0.0176
7 1 0.022799999999999997,1.0,0.0182
0 2 0.0314,1.0,0.0234
1 2 0.0298,1.0,0.023
2 2 0.037,1.0,0.023600000000000003
3 2 0.027200000000000002,1.0,0.022
4 2 0.0304,1.0,0.02
5 2 0.027600000000000003,1.0,0.0182
6 2 0.0236,1.0,0.0178
7 2 0.023799999999999998,1.0,0.0166
0 3 0.0306,1.0,0.0198
1 3 0.030199999999999998,1.0,0.02
2 3 0.0314,1.0,0.022400000000000003
3 3 0.029200000000000004,1.0,0.0242
4 3 0.028999999999999998,1.0,0.018
5 3 0.0294,1.0,0.019
6 3 0.028000000000000004,1.0,0.0172
7 3 0.0258,1.0,0.017400000000000002
0 4 0.0236,1.0,0.0164
1 4 0.0224,1.0,0.0166
2 4 0.023,1.0,0.0166
3 4 0.028999999999999998,1.0,0.0172
4 4 0.0258,1.0,0.017599999999999998
5 4 0.026199999999999994,1.0,0.0206
6 4 0.025,1.0,0.018
7 4 0.0162,1.0,0.017599999999999998
0 5 0.018,1.0,0.018000000000000002
1 5 0.0194,1.0,0.016800000000000002
2 5 0.017599999999999998,1.0,0.017
3 5 0.018600000000000002,1.0,0.017800000000000003
4 5 0.016999999999999998,1.0,0.0174
5 5 0.0188,1.0,0.017800000000000003
6 5 0.0182,1.0,0.017400000000000002
7 5 0.02,1.0,0.017800000000000003
0 6 0.0162,1.0,0.017
1 6 0.016800000000000002,1.0,0.019
2 6 0.0166,1.0,0.016599999999999997
3 6 0.016599999999999997,1.0,0.016600000000000004
4 6 0.0166,1.0,0.017599999999999998
5 6 0.016,1.0,0.0156
6 6 0.016,1.0,0.0166
7 6 0.0172,1.0,0.017400000000000002
0 7 0.018000000000000002,1.0,0.017599999999999998
1 7 0.018000000000000002,1.0,0.016800000000000002
2 7 0.0188,1.0,0.0172
3 7 0.0154,1.0,0.019
4 7 0.016,1.0,0.0184
5 7 0.014600000000000002,1.0,0.016800000000000002
6 7 0.017800000000000003,1.0,0.016800000000000002
7 7 0.0162,1.0,0.016800000000000002
0 8 0.0158,1.0,0.0172
1 8 0.0196,1.0,0.019200000000000002
2 8 0.0172,1.0,0.0176
3 8 0.0172,1.0,0.0184
4 8 0.0182,1.0,0.018
5 8 0.0158,1.0,0.017
6 8 0.014599999999999998,1.0,0.0166
7 8 0.014599999999999998,1.0,0.019200000000000002
0 9 0.0172,1.0,0.0172
1 9 0.0164,1.0,0.018399999999999996
2 9 0.017,1.0,0.017
3 9 0.0166,1.0,0.0178
4 9 0.0166,1.0,0.0202
5 9 0.0162,1.0,0.0172
6 9 0.0166,1.0,0.017800000000000003
7 9 0.016,1.0,0.0186
};
\addplot[
only marks,
mark=x,
]
table {
x y
2 2
};

\nextgroupplot[title=max (98.6\%), yticklabels={}, xticklabels={8,16,32,64,128,256,512,1024}]
\addplot[
matrix plot,
mesh/cols=8,
mesh/color input=explicit,
]
table[meta=rgb] {
x y rgb
0 0 0.1146,0.984,0.016
1 0 0.4236,0.9266,0.016599999999999997
2 0 0.8328,0.8550000000000001,0.016
3 0 0.8991999999999999,0.8524,0.0156
4 0 0.9568,0.9996,0.7526
5 0 0.9753999999999999,1.0,0.8484
6 0 0.977,1.0,0.8553999999999998
7 0 0.9802,1.0,0.9004000000000001
0 1 0.6365999999999999,0.9998000000000001,0.020999999999999998
1 1 0.913,1.0,0.05840000000000001
2 1 0.9480000000000001,1.0,0.021799999999999996
3 1 0.9517999999999999,0.9998000000000001,0.27259999999999995
4 1 0.9644,0.9716000000000001,0.3424
5 1 0.9643999999999998,1.0,0.7874000000000001
6 1 0.9648,1.0,0.6831999999999999
7 1 0.9696,1.0,0.8094000000000001
0 2 0.759,1.0,0.08960000000000001
1 2 0.9064,1.0,0.0658
2 2 0.9713999999999998,1.0,0.1574
3 2 0.9784,1.0,0.5426
4 2 0.9843999999999999,1.0,0.8395999999999999
5 2 0.985,1.0,0.6162
6 2 0.966,1.0,0.124
7 2 0.9212,1.0,0.08880000000000002
0 3 0.8215999999999999,1.0,0.026000000000000002
1 3 0.9304,1.0,0.0412
2 3 0.9702,1.0,0.04379999999999999
3 3 0.9823999999999999,1.0,0.23159999999999997
4 3 0.9889999999999999,1.0,0.6644
5 3 0.9916,1.0,0.893
6 3 0.9917999999999999,1.0,0.8751999999999999
7 3 0.977,1.0,0.5606
0 4 0.79,1.0,0.018
1 4 0.9301999999999999,1.0,0.0252
2 4 0.9564,1.0,0.021200000000000004
3 4 0.9722,1.0,0.017599999999999998
4 4 0.9773999999999999,1.0,0.029599999999999998
5 4 0.9782,1.0,0.099
6 4 0.9756,1.0,0.4276
7 4 0.9712,1.0,0.641
0 5 0.41320000000000007,1.0,0.0194
1 5 0.825,1.0,0.021599999999999998
2 5 0.9405999999999999,1.0,0.022600000000000002
3 5 0.9608000000000001,1.0,0.025
4 5 0.9728,1.0,0.0218
5 5 0.9762000000000001,1.0,0.0208
6 5 0.9752000000000001,1.0,0.0246
7 5 0.9718,1.0,0.071
0 6 0.06040000000000001,1.0,0.019
1 6 0.23120000000000002,1.0,0.0196
2 6 0.6848,1.0,0.021399999999999995
3 6 0.8968,1.0,0.0164
4 6 0.9448000000000001,1.0,0.018799999999999997
5 6 0.9561999999999999,1.0,0.0162
6 6 0.9608000000000001,1.0,0.014599999999999998
7 6 0.9658,1.0,0.018400000000000003
0 7 0.017,1.0,0.016999999999999998
1 7 0.016999999999999998,1.0,0.0172
2 7 0.056400000000000006,1.0,0.0188
3 7 0.3146,1.0,0.0186
4 7 0.6788000000000001,1.0,0.0156
5 7 0.8118000000000001,1.0,0.019200000000000002
6 7 0.881,1.0,0.0156
7 7 0.8860000000000001,1.0,0.016
0 8 0.017400000000000002,1.0,0.0172
1 8 0.0158,1.0,0.017400000000000002
2 8 0.016800000000000002,1.0,0.017599999999999998
3 8 0.017599999999999998,1.0,0.0176
4 8 0.0354,1.0,0.0182
5 8 0.13040000000000002,1.0,0.0174
6 8 0.3776,1.0,0.016
7 8 0.5965999999999999,0.9984,0.016
0 9 0.016800000000000002,1.0,0.017599999999999998
1 9 0.016800000000000002,1.0,0.0172
2 9 0.0172,1.0,0.017
3 9 0.0162,1.0,0.0172
4 9 0.0158,1.0,0.017
5 9 0.0176,1.0,0.0186
6 9 0.0172,1.0,0.0164
7 9 0.0258,0.999,0.0188
};
\addplot[
only marks,
mark=x,
]
table {
x y
5 3
};

\end{groupplot}
\end{tikzpicture}

%% file: plots/cases/first_token_output/var_heads.tex

\begin{tikzpicture}

\begin{groupplot}[
group style={
    group size=4 by 1,
    horizontal sep=0.1cm,
    xlabels at=edge bottom,
    ylabels at=edge left,
},
height=4.3cm, 
width=1.55*\columnwidth/6,
enlargelimits=false,
xtick=data, ytick=data, xticklabel style={font=\small, rotate=90}, yticklabel style={font=\small}, xticklabels={1,2,4,8,16,32,64,128}, title style={at={(0.5,0.9)}, font=\small},]

\nextgroupplot[title=BERT, yticklabels={0.3,0.09,0.027,0.008,0.002,7.3e-4,2.2e-4,6.6e-5,2.0e-5,5.9e-6}]
\addplot[
matrix plot,
mesh/cols=8,
mesh/color input=explicit,
]
table[meta=rgb] {
x y rgb
0 0 0.20749999582767487,0.2123750001192093,0.2240625023841858
1 0 0.21062499284744263,0.21800000965595245,0.22875000536441803
2 0 0.20406250655651093,0.21449999511241913,0.22468750178813934
3 0 0.21125000715255737,0.2147500067949295,0.21812500059604645
4 0 0.20406250655651093,0.21581248939037323,0.22624999284744263
5 0 0.20218749344348907,0.2108749896287918,0.2212499976158142
6 0 0.20218749344348907,0.211124986410141,0.22093750536441803
7 0 0.1993750035762787,0.21181249618530273,0.21906250715255737
0 1 0.21187500655651093,0.21531251072883606,0.22968749701976776
1 1 0.21281249821186066,0.2186250239610672,0.22750000655651093
2 1 0.2109375,0.2172500193119049,0.22968749701976776
3 1 0.21531249582767487,0.21875,0.22937500476837158
4 1 0.2162500023841858,0.22087499499320984,0.23156249523162842
5 1 0.2162500023841858,0.22331249713897705,0.23468750715255737
6 1 0.2134374976158142,0.21787500381469727,0.23000000417232513
7 1 0.21250000596046448,0.21662500500679016,0.22750000655651093
0 2 0.21562500298023224,0.22037498652935028,0.22875000536441803
1 2 0.21875,0.22187499701976776,0.23218749463558197
2 2 0.21593749523162842,0.21724998950958252,0.23375000059604645
3 2 0.21062499284744263,0.21906252205371857,0.23281249403953552
4 2 0.2150000035762787,0.21968750655651093,0.22687500715255737
5 2 0.21375000476837158,0.22081248462200165,0.23250000178813934
6 2 0.20812499523162842,0.2172500193119049,0.23156249523162842
7 2 0.2146874964237213,0.21949999034404755,0.23281249403953552
0 3 0.21312500536441803,0.2175000011920929,0.2306250035762787
1 3 0.21375000476837158,0.2188750058412552,0.22968749701976776
2 3 0.21031250059604645,0.21649999916553497,0.2306250035762787
3 3 0.21281249821186066,0.2188750058412552,0.22875000536441803
4 3 0.21062499284744263,0.21568751335144043,0.2278124988079071
5 3 0.21406249701976776,0.21700000762939453,0.23250000178813934
6 3 0.2121874988079071,0.21912498772144318,0.23218749463558197
7 3 0.21187500655651093,0.21912498772144318,0.23218749463558197
0 4 0.2121874988079071,0.2202499806880951,0.2290624976158142
1 4 0.21312500536441803,0.21912498772144318,0.23375000059604645
2 4 0.2134374976158142,0.22300000488758087,0.2368749976158142
3 4 0.21187500655651093,0.22437500953674316,0.27781251072883606
4 4 0.21250000596046448,0.2226249873638153,0.2512499988079071
5 4 0.21375000476837158,0.2226875126361847,0.24156250059604645
6 4 0.21562500298023224,0.3486874997615814,0.8518750071525574
7 4 0.22343750298023224,0.4425000250339508,0.8921874761581421
0 5 0.21937499940395355,0.22843749821186066,0.24062499403953552
1 5 0.22312499582767487,0.37543749809265137,0.9150000214576721
2 5 0.22937500476837158,0.3688125014305115,0.9206249713897705
3 5 0.25,0.7720624804496765,0.9153125286102295
4 5 0.901562511920929,0.9114375114440918,0.9209374785423279
5 5 0.9134374856948853,0.9196249842643738,0.9303125143051147
6 5 0.9121875166893005,0.916812539100647,0.925000011920929
7 5 0.9168750047683716,0.9230625033378601,0.9337499737739563
0 6 0.22156250476837158,0.226749986410141,0.24250000715255737
1 6 0.2253125011920929,0.28450003266334534,0.5253124833106995
2 6 0.22843749821186066,0.2332499921321869,0.2603124976158142
3 6 0.2228125035762787,0.2486875057220459,0.33406248688697815
4 6 0.2356249988079071,0.27162498235702515,0.3746874928474426
5 6 0.24437500536441803,0.38987499475479126,0.5400000214576721
6 6 0.2878125011920929,0.4690000116825104,0.6862499713897705
7 6 0.2590624988079071,0.5377500057220459,0.7821875214576721
0 7 0.21906250715255737,0.2278124988079071,0.23656250536441803
1 7 0.21593749523162842,0.22193750739097595,0.2356249988079071
2 7 0.21781249344348907,0.2212500125169754,0.2475000023841858
3 7 0.2146874964237213,0.2200624942779541,0.23281249403953552
4 7 0.21250000596046448,0.22062499821186066,0.23156249523162842
5 7 0.21437500417232513,0.22056250274181366,0.23593750596046448
6 7 0.21281249821186066,0.21843750774860382,0.2331250011920929
7 7 0.2134374976158142,0.2201250046491623,0.23000000417232513
0 8 0.20999999344348907,0.21649999916553497,0.22968749701976776
1 8 0.21937499940395355,0.22374999523162842,0.22937500476837158
2 8 0.21375000476837158,0.22181248664855957,0.22937500476837158
3 8 0.2096875011920929,0.21649999916553497,0.23343749344348907
4 8 0.20999999344348907,0.22225001454353333,0.23375000059604645
5 8 0.21062499284744263,0.2186249941587448,0.23218749463558197
6 8 0.21156249940395355,0.21799997985363007,0.23593750596046448
7 8 0.21187500655651093,0.21931247413158417,0.2331250011920929
0 9 0.21156249940395355,0.21937501430511475,0.23281249403953552
1 9 0.21250000596046448,0.22093752026557922,0.23499999940395355
2 9 0.2175000011920929,0.22056250274181366,0.23250000178813934
3 9 0.21156249940395355,0.21924999356269836,0.23093749582767487
4 9 0.20874999463558197,0.21574997901916504,0.23125000298023224
5 9 0.21656249463558197,0.2198750078678131,0.2303124964237213
6 9 0.21312500536441803,0.22018751502037048,0.23250000178813934
7 9 0.21187500655651093,0.21637499332427979,0.23281249403953552
};
\addplot[
only marks,
mark=x,
]
table {
x y
7 5
};
\addplot[
only marks,
mark=x,
mark options={draw=red},
]
table {
x y
6 5
};

\nextgroupplot[title=MTE, yticklabels={}]
\addplot[
matrix plot,
mesh/cols=8,
mesh/color input=explicit,
]
table[meta=rgb] {
x y rgb
0 0 0.20906250178813934,0.21587499976158142,0.22562499344348907
1 0 0.203125,0.21643748879432678,0.24124999344348907
2 0 0.2121874988079071,0.21562500298023224,0.23281249403953552
3 0 0.20874999463558197,0.21681252121925354,0.23125000298023224
4 0 0.2084375023841858,0.21587499976158142,0.22593750059604645
5 0 0.21312500536441803,0.21793749928474426,0.22499999403953552
6 0 0.20624999701976776,0.21393749117851257,0.22718749940395355
7 0 0.20531250536441803,0.2148124873638153,0.22687500715255737
0 1 0.20812499523162842,0.2148124873638153,0.23125000298023224
1 1 0.2162500023841858,0.22112500667572021,0.2306250035762787
2 1 0.20812499523162842,0.21781249344348907,0.2368749976158142
3 1 0.20906250178813934,0.21912498772144318,0.2368749976158142
4 1 0.21062499284744263,0.23212499916553497,0.25968751311302185
5 1 0.21531249582767487,0.21900001168251038,0.2278124988079071
6 1 0.21312500536441803,0.21781249344348907,0.2303124964237213
7 1 0.21062499284744263,0.23987500369548798,0.2887499928474426
0 2 0.21562500298023224,0.22025001049041748,0.23250000178813934
1 2 0.2224999964237213,0.2291249930858612,0.24062499403953552
2 2 0.2368749976158142,0.3409999907016754,0.5821874737739563
3 2 0.25718748569488525,0.40549999475479126,0.6584374904632568
4 2 0.27656251192092896,0.4711250364780426,0.7003124952316284
5 2 0.2409375011920929,0.4754374921321869,0.7068750262260437
6 2 0.6440625190734863,0.6667500734329224,0.7281249761581421
7 2 0.6899999976158142,0.7118124961853027,0.746874988079071
0 3 0.22093750536441803,0.3324374854564667,0.7184374928474426
1 3 0.22718749940395355,0.49268755316734314,0.7956249713897705
2 3 0.23624999821186066,0.6854375004768372,0.8459374904632568
3 3 0.8181250095367432,0.8338750600814819,0.8784375190734863
4 3 0.8731250166893005,0.8868125081062317,0.9003124833106995
5 3 0.8928124904632568,0.901187539100647,0.9103124737739563
6 3 0.8996875286102295,0.9048125147819519,0.9106249809265137
7 3 0.9006249904632568,0.9054374694824219,0.917187511920929
0 4 0.22750000655651093,0.2447499930858612,0.3031249940395355
1 4 0.22968749701976776,0.3852500319480896,0.8153125047683716
2 4 0.5903124809265137,0.8129375576972961,0.890625
3 4 0.4346874952316284,0.5816875100135803,0.6712499856948853
4 4 0.7649999856948853,0.8558124303817749,0.9156249761581421
5 4 0.8003125190734863,0.8834375143051147,0.9196875095367432
6 4 0.737500011920929,0.8666250109672546,0.9118750095367432
7 4 0.8868749737739563,0.8981874585151672,0.9150000214576721
0 5 0.22499999403953552,0.26112499833106995,0.3400000035762787
1 5 0.2253125011920929,0.2744999825954437,0.3412500023841858
2 5 0.23624999821186066,0.37306249141693115,0.5531250238418579
3 5 0.34031251072883606,0.41318750381469727,0.5131250023841858
4 5 0.3034375011920929,0.4090000092983246,0.5246875286102295
5 5 0.3075000047683716,0.3590624928474426,0.4246875047683716
6 5 0.24124999344348907,0.3467499911785126,0.45625001192092896
7 5 0.31968748569488525,0.343562513589859,0.3553124964237213
0 6 0.21781249344348907,0.22568750381469727,0.2384375035762787
1 6 0.21656249463558197,0.22499999403953552,0.23593750596046448
2 6 0.22218750417232513,0.22487500309944153,0.23218749463558197
3 6 0.2146874964237213,0.22456249594688416,0.2368749976158142
4 6 0.21593749523162842,0.2199374884366989,0.23281249403953552
5 6 0.21531249582767487,0.22218748927116394,0.2368749976158142
6 6 0.21281249821186066,0.22124998271465302,0.23499999940395355
7 6 0.21312500536441803,0.2200624942779541,0.23343749344348907
0 7 0.21656249463558197,0.22081248462200165,0.23281249403953552
1 7 0.21718749403953552,0.22212500870227814,0.23250000178813934
2 7 0.22187499701976776,0.22806251049041748,0.23375000059604645
3 7 0.22156250476837158,0.22806251049041748,0.23781250417232513
4 7 0.21406249701976776,0.22318749129772186,0.23468750715255737
5 7 0.21718749403953552,0.2227499932050705,0.23468750715255737
6 7 0.2146874964237213,0.22081251442432404,0.2290624976158142
7 7 0.21875,0.22437500953674316,0.23531250655651093
0 8 0.2146874964237213,0.21924999356269836,0.2290624976158142
1 8 0.21718749403953552,0.21906249225139618,0.23406249284744263
2 8 0.2134374976158142,0.2199374884366989,0.23343749344348907
3 8 0.2150000035762787,0.2186874896287918,0.23156249523162842
4 8 0.22156250476837158,0.22343750298023224,0.22937500476837158
5 8 0.21250000596046448,0.218812495470047,0.23093749582767487
6 8 0.21406249701976776,0.22306248545646667,0.2331250011920929
7 8 0.21187500655651093,0.2173749953508377,0.2303124964237213
0 9 0.21187500655651093,0.21968750655651093,0.2318750023841858
1 9 0.21437500417232513,0.21812501549720764,0.23593750596046448
2 9 0.21250000596046448,0.22318749129772186,0.23375000059604645
3 9 0.21906250715255737,0.22275002300739288,0.23000000417232513
4 9 0.21531249582767487,0.21806249022483826,0.2290624976158142
5 9 0.21656249463558197,0.22068750858306885,0.22968749701976776
6 9 0.21375000476837158,0.21906249225139618,0.2290624976158142
7 9 0.20906250178813934,0.2173124998807907,0.2356249988079071
};
\addplot[
only marks,
mark=x,
]
table {
x y
7 3
};

\nextgroupplot[title=NAP, yticklabels={}]
\addplot[
matrix plot,
mesh/cols=8,
mesh/color input=explicit,
]
table[meta=rgb] {
x y rgb
0 0 0.20906250178813934,0.21912498772144318,0.23156249523162842
1 0 0.20874999463558197,0.21450002491474152,0.22593750059604645
2 0 0.21187500655651093,0.21875,0.22562499344348907
3 0 0.21031250059604645,0.2148124873638153,0.2240625023841858
4 0 0.20937499403953552,0.21393749117851257,0.22187499701976776
5 0 0.21156249940395355,0.2172500193119049,0.22750000655651093
6 0 0.20999999344348907,0.2146875113248825,0.22499999403953552
7 0 0.2084375023841858,0.21287500858306885,0.22437499463558197
0 1 0.21156249940395355,0.2163124978542328,0.22875000536441803
1 1 0.21031250059604645,0.21531251072883606,0.22687500715255737
2 1 0.21312500536441803,0.2172500193119049,0.22718749940395355
3 1 0.20937499403953552,0.21437498927116394,0.23093749582767487
4 1 0.2071875035762787,0.2186250239610672,0.23281249403953552
5 1 0.21062499284744263,0.21956248581409454,0.23531250655651093
6 1 0.21062499284744263,0.21681252121925354,0.22687500715255737
7 1 0.21250000596046448,0.21637499332427979,0.2253125011920929
0 2 0.21562500298023224,0.2240000069141388,0.2318750023841858
1 2 0.22031250596046448,0.2265625,0.23906250298023224
2 2 0.22031250596046448,0.22593748569488525,0.24187499284744263
3 2 0.22218750417232513,0.2744999825954437,0.4403125047683716
4 2 0.2265625,0.28712499141693115,0.3865624964237213
5 2 0.234375,0.3681875169277191,0.5193750262260437
6 2 0.2421875,0.49543753266334534,0.7506250143051147
7 2 0.2809374928474426,0.4126250147819519,0.6106250286102295
0 3 0.22374999523162842,0.22943750023841858,0.24031250178813934
1 3 0.22468750178813934,0.2342499941587448,0.2512499988079071
2 3 0.2381249964237213,0.2725624740123749,0.39781248569488525
3 3 0.23749999701976776,0.33287498354911804,0.640625
4 3 0.23281249403953552,0.5292500257492065,0.7568749785423279
5 3 0.26875001192092896,0.4750625193119049,0.6103125214576721
6 3 0.46281251311302185,0.6986249685287476,0.8534374833106995
7 3 0.30281248688697815,0.5878125429153442,0.7643749713897705
0 4 0.2240625023841858,0.22731252014636993,0.24031250178813934
1 4 0.2199999988079071,0.281125009059906,0.49156248569488525
2 4 0.22718749940395355,0.29506251215934753,0.3631249964237213
3 4 0.24187499284744263,0.46281251311302185,0.6399999856948853
4 4 0.23468750715255737,0.5134374499320984,0.8306249976158142
5 4 0.5887500047683716,0.7535625696182251,0.8809375166893005
6 4 0.38749998807907104,0.6148124933242798,0.8096874952316284
7 4 0.3721874952316284,0.7211250066757202,0.8678125143051147
0 5 0.21812500059604645,0.22456249594688416,0.24156250059604645
1 5 0.2265625,0.2695625126361847,0.34281250834465027
2 5 0.3409374952316284,0.5269374847412109,0.7528125047683716
3 5 0.2475000023841858,0.6770625114440918,0.9046875238418579
4 5 0.49687498807907104,0.7563124895095825,0.9212499856948853
5 5 0.7212499976158142,0.8634374737739563,0.9362499713897705
6 5 0.7878124713897705,0.8574374914169312,0.9296875
7 5 0.5246875286102295,0.7741875052452087,0.9659374952316284
0 6 0.8031250238418579,0.8814999461174011,0.9331250190734863
1 6 0.8515625,0.9212499856948853,0.9543750286102295
2 6 0.8862500190734863,0.9468124508857727,0.9821875095367432
3 6 0.887499988079071,0.9576250314712524,0.9781249761581421
4 6 0.9743750095367432,0.9796250462532043,0.9856250286102295
5 6 0.8690624833106995,0.9577500224113464,0.9881250262260437
6 6 0.8299999833106995,0.9513124227523804,0.9862499833106995
7 6 0.7206249833106995,0.7941874861717224,0.9853125214576721
0 7 0.8306249976158142,0.8526250123977661,0.8731250166893005
1 7 0.7684375047683716,0.8242500424385071,0.862500011920929
2 7 0.6303125023841858,0.7919999957084656,0.8603125214576721
3 7 0.7559375166893005,0.8475624918937683,0.9256250262260437
4 7 0.6349999904632568,0.8161875009536743,0.9434375166893005
5 7 0.4699999988079071,0.7903125882148743,0.9356250166893005
6 7 0.7225000262260437,0.8197500109672546,0.8881250023841858
7 7 0.22593750059604645,0.44593754410743713,0.8556249737739563
0 8 0.21968750655651093,0.22593751549720764,0.23781250417232513
1 8 0.22218750417232513,0.2278124988079071,0.23593750596046448
2 8 0.21875,0.22481250762939453,0.2396875023841858
3 8 0.21968750655651093,0.22424998879432678,0.23593750596046448
4 8 0.21968750655651093,0.22631248831748962,0.24187499284744263
5 8 0.22312499582767487,0.22831249237060547,0.23718750476837158
6 8 0.21843749284744263,0.22562499344348907,0.23874999582767487
7 8 0.21968750655651093,0.22312501072883606,0.23624999821186066
0 9 0.21406249701976776,0.21781249344348907,0.2278124988079071
1 9 0.21281249821186066,0.21843750774860382,0.23281249403953552
2 9 0.21531249582767487,0.22093749046325684,0.23093749582767487
3 9 0.21312500536441803,0.22100000083446503,0.2306250035762787
4 9 0.2096875011920929,0.21718749403953552,0.23218749463558197
5 9 0.21156249940395355,0.21831250190734863,0.23093749582767487
6 9 0.21187500655651093,0.21843750774860382,0.23499999940395355
7 9 0.2146874964237213,0.2200625240802765,0.23343749344348907
};
\addplot[
only marks,
mark=x,
]
table {
x y
4 6
};

\nextgroupplot[title=NON, yticklabels={}]
\addplot[
matrix plot,
mesh/cols=8,
mesh/color input=explicit,
]
table[meta=rgb] {
x y rgb
0 0 0.2096875011920929,0.2135624885559082,0.22156250476837158
1 0 0.20749999582767487,0.2136874943971634,0.22437499463558197
2 0 0.2068749964237213,0.21512500941753387,0.23218749463558197
3 0 0.2109375,0.21800000965595245,0.23375000059604645
4 0 0.2071875035762787,0.21400001645088196,0.23718750476837158
5 0 0.20781250298023224,0.21787500381469727,0.2290624976158142
6 0 0.2084375023841858,0.21331247687339783,0.2253125011920929
7 0 0.2096875011920929,0.21543750166893005,0.2278124988079071
0 1 0.2071875035762787,0.21406249701976776,0.2265625
1 1 0.21187500655651093,0.21712498366832733,0.22624999284744263
2 1 0.21281249821186066,0.2200624942779541,0.22624999284744263
3 1 0.20749999582767487,0.21268749237060547,0.2290624976158142
4 1 0.20749999582767487,0.21375000476837158,0.2290624976158142
5 1 0.20781250298023224,0.2149999886751175,0.22968749701976776
6 1 0.21250000596046448,0.21462500095367432,0.22312499582767487
7 1 0.21312500536441803,0.21975000202655792,0.23250000178813934
0 2 0.21062499284744263,0.2176249921321869,0.22843749821186066
1 2 0.2146874964237213,0.2186250239610672,0.23781250417232513
2 2 0.2121874988079071,0.2186249941587448,0.23906250298023224
3 2 0.2162500023841858,0.22474999725818634,0.23406249284744263
4 2 0.21875,0.2240000069141388,0.24062499403953552
5 2 0.22062499821186066,0.22850000858306885,0.2396875023841858
6 2 0.22031250596046448,0.2254374921321869,0.23937499523162842
7 2 0.2240625023841858,0.22931250929832458,0.25437501072883606
0 3 0.22156250476837158,0.22625000774860382,0.24124999344348907
1 3 0.22062499821186066,0.22481250762939453,0.2356249988079071
2 3 0.22031250596046448,0.22856250405311584,0.2409375011920929
3 3 0.21812500059604645,0.22837500274181366,0.2459374964237213
4 3 0.21843749284744263,0.22681251168251038,0.24031250178813934
5 3 0.22218750417232513,0.2264374941587448,0.24437500536441803
6 3 0.22343750298023224,0.2292499989271164,0.24531249701976776
7 3 0.22343750298023224,0.2276875078678131,0.24312500655651093
0 4 0.21968750655651093,0.2266875058412552,0.24156250059604645
1 4 0.21812500059604645,0.22587499022483826,0.2381249964237213
2 4 0.23125000298023224,0.24912500381469727,0.29093751311302185
3 4 0.2228125035762787,0.24787497520446777,0.30687499046325684
4 4 0.22062499821186066,0.242000013589859,0.2737500071525574
5 4 0.2318750023841858,0.25824999809265137,0.28125
6 4 0.2290624976158142,0.288937509059906,0.43281251192092896
7 4 0.24781249463558197,0.2641250193119049,0.29781249165534973
0 5 0.2199999988079071,0.22681251168251038,0.2384375035762787
1 5 0.21843749284744263,0.2253125011920929,0.24187499284744263
2 5 0.22312499582767487,0.22631248831748962,0.2384375035762787
3 5 0.2228125035762787,0.22881250083446503,0.25093749165534973
4 5 0.22562499344348907,0.2291875183582306,0.2475000023841858
5 5 0.23000000417232513,0.2318125069141388,0.23906250298023224
6 5 0.22750000655651093,0.2342500239610672,0.24312500655651093
7 5 0.22093750536441803,0.22824999690055847,0.24375000596046448
0 6 0.2228125035762787,0.22693748772144318,0.24031250178813934
1 6 0.21687500178813934,0.226749986410141,0.23999999463558197
2 6 0.21937499940395355,0.22431249916553497,0.23999999463558197
3 6 0.21843749284744263,0.22462499141693115,0.2381249964237213
4 6 0.21875,0.22487500309944153,0.2409375011920929
5 6 0.21937499940395355,0.22575001418590546,0.2356249988079071
6 6 0.22093750536441803,0.22606250643730164,0.23937499523162842
7 6 0.21718749403953552,0.22481250762939453,0.24187499284744263
0 7 0.22343750298023224,0.2303124964237213,0.2446874976158142
1 7 0.21875,0.22587499022483826,0.2396875023841858
2 7 0.21906250715255737,0.22706250846385956,0.24031250178813934
3 7 0.21812500059604645,0.22293749451637268,0.23999999463558197
4 7 0.22031250596046448,0.22712500393390656,0.2381249964237213
5 7 0.21968750655651093,0.22593751549720764,0.2421875
6 7 0.22093750536441803,0.22568750381469727,0.23499999940395355
7 7 0.21906250715255737,0.22556248307228088,0.2396875023841858
0 8 0.21875,0.22318749129772186,0.23343749344348907
1 8 0.21875,0.22362498939037323,0.2384375035762787
2 8 0.21875,0.22475001215934753,0.24406249821186066
3 8 0.21281249821186066,0.21968750655651093,0.2303124964237213
4 8 0.2109375,0.2176249921321869,0.23999999463558197
5 8 0.21281249821186066,0.21931250393390656,0.22937500476837158
6 8 0.21562500298023224,0.21962502598762512,0.23218749463558197
7 8 0.20937499403953552,0.2162499874830246,0.2303124964237213
0 9 0.21156249940395355,0.21649999916553497,0.22843749821186066
1 9 0.21125000715255737,0.21906249225139618,0.22968749701976776
2 9 0.21875,0.2213750183582306,0.22937500476837158
3 9 0.21062499284744263,0.21662500500679016,0.23406249284744263
4 9 0.2084375023841858,0.21774999797344208,0.23281249403953552
5 9 0.2134374976158142,0.21800000965595245,0.23125000298023224
6 9 0.21375000476837158,0.21849998831748962,0.22968749701976776
7 9 0.21187500655651093,0.21837501227855682,0.2306250035762787
};
\addplot[
only marks,
mark=x,
]
table {
x y
6 4
};
\addplot[
only marks,
mark=x,
mark options={draw=red},
]
table {
x y
3 4
};

\end{groupplot}
\end{tikzpicture}

%% file: plots/mode/var_dim.tex

\begin{tikzpicture}

\begin{groupplot}[
group style={
    group size=6 by 1,
    horizontal sep=0.1cm,
    vertical sep=0.85cm,
    xlabels at=edge bottom,
    ylabels at=edge left,
},
height=4.3cm, 
width=1.55*\columnwidth/6,
enlargelimits=false,
xtick=data, ytick=data, xticklabels={8,16,32,64,128,256,512,1024}, xticklabel style={font=\small, rotate=90}, yticklabel style={font=\small}, title style={at={(0.5,0.9)}, font=\small},]

\nextgroupplot[title=BERT (99.9\%), yticklabels={0.3,0.09,0.027,0.008,0.002,7.3e-4,2.2e-4,6.6e-5,2.0e-5,5.9e-6}]
\addplot[
matrix plot,
mesh/cols=8,
mesh/color input=explicit,
]
table[meta=rgb] {
x y rgb
0 0 0.11218749731779099,0.12325000017881393,0.1393750011920929
1 0 0.11375000327825546,0.12112500220537185,0.1290625035762787
2 0 0.11312499642372131,0.11918749958276749,0.1287499964237213
3 0 0.11500000208616257,0.12168750166893005,0.12968750298023224
4 0 0.11968749761581421,0.12331250011920929,0.13375000655651093
5 0 0.11156249791383743,0.11937500089406967,0.12656250596046448
6 0 0.11249999701976776,0.1190625011920929,0.12968750298023224
7 0 0.11124999821186066,0.11656250208616256,0.12343750149011612
0 1 0.11749999970197678,0.1261250004172325,0.13718749582767487
1 1 0.11625000089406967,0.12300000041723251,0.13562500476837158
2 1 0.11687500029802322,0.12074999958276748,0.13343749940395355
3 1 0.12312500178813934,0.12893750369548798,0.13312500715255737
4 1 0.11437500268220901,0.12037500143051147,0.13093750178813934
5 1 0.11375000327825546,0.12262499928474427,0.1303125023841858
6 1 0.11625000089406967,0.12075000256299973,0.12718750536441803
7 1 0.1184374988079071,0.12424999922513962,0.13218750059604645
0 2 0.11687500029802322,0.1629375010728836,0.3215624988079071
1 2 0.11874999850988388,0.12399999797344208,0.13437500596046448
2 2 0.11906249821186066,0.12675000131130218,0.14000000059604645
3 2 0.11781249940395355,0.12231249809265136,0.13500000536441803
4 2 0.11937499791383743,0.12943749874830246,0.1418749988079071
5 2 0.12687499821186066,0.14168749749660492,0.15531249344348907
6 2 0.11656250059604645,0.131624998152256,0.15031249821186066
7 2 0.11874999850988388,0.12462500035762787,0.13343749940395355
0 3 0.2503125071525574,0.38962500393390653,0.5321875214576721
1 3 0.12531250715255737,0.3465625077486038,0.6575000286102295
2 3 0.13375000655651093,0.2949374973773956,0.46875
3 3 0.18312500417232513,0.2751250073313713,0.4781250059604645
4 3 0.2553125023841858,0.309749998152256,0.4671874940395355
5 3 0.22437499463558197,0.26387499570846557,0.33281248807907104
6 3 0.25,0.2584999978542328,0.27000001072883606
7 3 0.11999999731779099,0.15706249922513962,0.20343750715255737
0 4 0.2199999988079071,0.30700000524520876,0.3687500059604645
1 4 0.4690625071525574,0.4993125021457672,0.5603125095367432
2 4 0.6137499809265137,0.673562490940094,0.7418749928474426
3 4 0.6681249737739563,0.779562509059906,0.8596875071525574
4 4 0.8696874976158142,0.9694999933242798,0.9965624809265137
5 4 0.9903125166893005,0.9961875081062317,0.9984375238418579
6 4 0.5478125214576721,0.831374990940094,0.9571874737739563
7 4 0.3890624940395355,0.44456249475479126,0.6700000166893005
0 5 0.2096875011920929,0.23199999332427979,0.2734375
1 5 0.2828125059604645,0.33637499809265137,0.37156251072883606
2 5 0.44593751430511475,0.48412500619888305,0.5309374928474426
3 5 0.6834375262260437,0.7628124952316284,0.8443750143051147
4 5 0.9940624833106995,0.9956249952316284,0.996874988079071
5 5 0.9950000047683716,0.9958124995231629,0.996874988079071
6 5 0.9909374713897705,0.993500006198883,0.9959375262260437
7 5 0.996874988079071,0.998812484741211,0.9996874928474426
0 6 0.13906249403953552,0.14393749833106995,0.15781250596046448
1 6 0.19249999523162842,0.24287500381469726,0.28812500834465027
2 6 0.2956250011920929,0.33406250476837157,0.4012500047683716
3 6 0.5550000071525574,0.5837499976158143,0.621874988079071
4 6 0.9568750262260437,0.9794374942779541,0.9884374737739563
5 6 0.995312511920929,0.9965000033378602,0.9975000023841858
6 6 0.9959375262260437,0.9974375009536743,0.9987499713897705
7 6 0.9975000023841858,0.998312509059906,0.9990624785423279
0 7 0.125,0.12925000190734864,0.13906249403953552
1 7 0.13312500715255737,0.1403750002384186,0.15000000596046448
2 7 0.17749999463558197,0.19274999797344208,0.21656249463558197
3 7 0.2984375059604645,0.32825000286102296,0.38624998927116394
4 7 0.5953124761581421,0.6542500138282776,0.7106249928474426
5 7 0.9834374785423279,0.9849375009536743,0.9865624904632568
6 7 0.9956250190734863,0.996750009059906,0.9981250166893005
7 7 0.996874988079071,0.9976249814033509,0.9987499713897705
0 8 0.10875000059604645,0.11793750077486038,0.13218750059604645
1 8 0.13187499344348907,0.13531250059604644,0.14281250536441803
2 8 0.12843750417232513,0.13400000035762788,0.1459375023841858
3 8 0.15000000596046448,0.15568750202655793,0.16906249523162842
4 8 0.2290624976158142,0.26537499725818636,0.29343751072883606
5 8 0.682812511920929,0.7024999976158142,0.7271875143051147
6 8 0.9865624904632568,0.9888125061988831,0.9921875
7 8 0.995312511920929,0.9966250061988831,0.9975000023841858
0 9 0.10593750327825546,0.1092500001192093,0.11687500029802322
1 9 0.11093749850988388,0.1196250006556511,0.13593749701976776
2 9 0.11218749731779099,0.11856250166893005,0.13468749821186066
3 9 0.12062499672174454,0.12850000113248825,0.14093750715255737
4 9 0.13343749940395355,0.15199999809265136,0.17031249403953552
5 9 0.1993750035762787,0.21687499880790712,0.25343748927116394
6 9 0.6368749737739563,0.6732499957084656,0.7106249928474426
7 9 0.9828125238418579,0.9855000019073487,0.9881250262260437
};
\addplot[
only marks,
mark=x,
]
table {
x y
7 5
};

\nextgroupplot[title=MTE (100\%), yticklabels={}]
\addplot[
matrix plot,
mesh/cols=8,
mesh/color input=explicit,
]
table[meta=rgb] {
x y rgb
0 0 0.12531250715255737,0.12956250309944153,0.13562500476837158
1 0 0.1262499988079071,0.13274999856948852,0.13875000178813934
2 0 0.11999999731779099,0.1292500004172325,0.14343750476837158
3 0 0.1237500011920929,0.13081250041723252,0.1365624964237213
4 0 0.12312500178813934,0.13581250011920928,0.15031249821186066
5 0 0.12843750417232513,0.1365625023841858,0.14249999821186066
6 0 0.12656250596046448,0.1326875001192093,0.15000000596046448
7 0 0.11093749850988388,0.12974999994039535,0.14749999344348907
0 1 0.12968750298023224,0.13543749749660491,0.1484375
1 1 0.1315625011920929,0.13774999976158142,0.14374999701976776
2 1 0.13062499463558197,0.1357499986886978,0.14499999582767487
3 1 0.12531250715255737,0.1364999994635582,0.1418749988079071
4 1 0.13062499463558197,0.13656249940395354,0.14249999821186066
5 1 0.13625000417232513,0.14025000035762786,0.1459375023841858
6 1 0.1365624964237213,0.13943750262260438,0.14343750476837158
7 1 0.13249999284744263,0.1380624994635582,0.1496874988079071
0 2 0.4581249952316284,0.5279374957084656,0.5718749761581421
1 2 0.5584375262260437,0.6733749985694886,0.7912499904632568
2 2 0.6890624761581421,0.8292499899864196,0.989062488079071
3 2 0.14499999582767487,0.24893749952316285,0.42406249046325684
4 2 0.13124999403953552,0.14237499833106995,0.15468749403953552
5 2 0.13375000655651093,0.13993749916553497,0.15531249344348907
6 2 0.13718749582767487,0.14337500035762787,0.1599999964237213
7 2 0.1328125,0.13887500166893005,0.15312500298023224
0 3 0.3840624988079071,0.5241249978542328,0.6015625
1 3 0.6465625166893005,0.8068750023841857,0.9862499833106995
2 3 0.7162500023841858,0.9036875009536743,0.9984375238418579
3 3 0.9543750286102295,0.9841875076293946,0.9965624809265137
4 3 0.901562511920929,0.9746875047683716,0.9975000023841858
5 3 0.9828125238418579,0.9926874995231628,0.9978125095367432
6 3 0.13312500715255737,0.47706249058246614,0.9809374809265137
7 3 0.14000000059604645,0.19193749725818635,0.38093748688697815
0 4 0.3293749988079071,0.40731250047683715,0.4846875071525574
1 4 0.6118749976158142,0.6441874980926514,0.6615625023841858
2 4 0.8362500071525574,0.9555625081062317,0.9981250166893005
3 4 0.9906250238418579,0.9969375014305115,0.9996874928474426
4 4 0.9965624809265137,0.9975625038146972,0.9990624785423279
5 4 0.9906250238418579,0.9961875081062317,0.9981250166893005
6 4 0.9965624809265137,0.9980624914169312,0.9990624785423279
7 4 0.9515625238418579,0.9831250071525574,0.9978125095367432
0 5 0.2709375023841858,0.28956250548362733,0.3199999928474426
1 5 0.3956249952316284,0.4397499918937683,0.4621874988079071
2 5 0.65625,0.7165625095367432,0.7549999952316284
3 5 0.9946874976158142,0.996875011920929,0.9981250166893005
4 5 0.9978125095367432,0.9981875061988831,0.9987499713897705
5 5 0.9978125095367432,0.998562490940094,0.9990624785423279
6 5 0.9971874952316284,0.9984375,0.9996874928474426
7 5 0.9759374856948853,0.9905624866485596,0.9993749856948853
0 6 0.20781250298023224,0.2610000044107437,0.2915624976158142
1 6 0.38968750834465027,0.43625000715255735,0.5087500214576721
2 6 0.7712500095367432,0.8526250004768372,0.9731249809265137
3 6 0.9962499737739563,0.9971874952316284,0.9984375238418579
4 6 0.9965624809265137,0.9980000138282776,0.9984375238418579
5 6 0.9978125095367432,0.9993749976158142,1.0
6 6 0.9984375238418579,0.9991250038146973,1.0
7 6 0.9984375238418579,0.9989999890327453,0.9996874928474426
0 7 0.19968749582767487,0.2275000035762787,0.29468750953674316
1 7 0.3828125,0.41774999499320986,0.46656250953674316
2 7 0.8209375143051147,0.9046875,0.9665625095367432
3 7 0.9921875,0.993874990940094,0.9946874976158142
4 7 0.9959375262260437,0.9979999899864197,0.9990624785423279
5 7 0.9987499713897705,0.9993749856948853,1.0
6 7 0.9993749856948853,0.9994999885559082,1.0
7 7 0.9990624785423279,0.9993749856948853,1.0
0 8 0.14687499403953552,0.17962500154972078,0.22062499821186066
1 8 0.26249998807907104,0.31868749260902407,0.3490625023841858
2 8 0.7056249976158142,0.7683750033378601,0.8199999928474426
3 8 0.965624988079071,0.9722499966621398,0.979687511920929
4 8 0.9909374713897705,0.9932500004768372,0.995312511920929
5 8 0.9971874952316284,0.9982499957084656,0.9987499713897705
6 8 0.9987499713897705,0.9992499828338623,0.9996874928474426
7 8 0.9993749856948853,0.9997499942779541,1.0
0 9 0.11656250059604645,0.12362499982118606,0.15312500298023224
1 9 0.27437499165534973,0.2965624988079071,0.33937498927116394
2 9 0.5340625047683716,0.5590000033378602,0.6012499928474426
3 9 0.848437488079071,0.865874993801117,0.8818749785423279
4 9 0.9587500095367432,0.9630000114440918,0.9690625071525574
5 9 0.9865624904632568,0.9894374966621399,0.9915624856948853
6 9 0.9965624809265137,0.996999990940094,0.9984375238418579
7 9 0.9984375238418579,0.9991249918937684,1.0
};
\addplot[
only marks,
mark=x,
]
table {
x y
7 8
};

\nextgroupplot[title=NAP (99.7\%), yticklabels={}]
\addplot[
matrix plot,
mesh/cols=8,
mesh/color input=explicit,
]
table[meta=rgb] {
x y rgb
0 0 0.13124999403953552,0.179749995470047,0.28843748569488525
1 0 0.14249999821186066,0.3193749785423279,0.5237500071525574
2 0 0.13062499463558197,0.22318749129772186,0.34968748688697815
3 0 0.1368750035762787,0.3296250104904175,0.5581250190734863
4 0 0.19437499344348907,0.3865000009536743,0.5853124856948853
5 0 0.21937499940395355,0.3361875116825104,0.484375
6 0 0.21687500178813934,0.35493746399879456,0.5143749713897705
7 0 0.140625,0.3149999976158142,0.44749999046325684
0 1 0.12281250208616257,0.22212500870227814,0.3853124976158142
1 1 0.3199999928474426,0.3803125023841858,0.42093750834465027
2 1 0.14281250536441803,0.5608125329017639,0.996874988079071
3 1 0.7562500238418579,0.9212499856948853,0.9981250166893005
4 1 0.9896875023841858,0.9935000538825989,0.995312511920929
5 1 0.9962499737739563,0.9968124628067017,0.9978125095367432
6 1 0.4012500047683716,0.875374972820282,0.9956250190734863
7 1 0.3578124940395355,0.8611875772476196,0.9962499737739563
0 2 0.31718748807907104,0.4416874945163727,0.5371875166893005
1 2 0.6650000214576721,0.7908124327659607,0.8843749761581421
2 2 0.9543750286102295,0.9770625233650208,0.9931250214576721
3 2 0.9634374976158142,0.9794999957084656,0.9959375262260437
4 2 0.9918749928474426,0.9957500696182251,0.9987499713897705
5 2 0.9950000047683716,0.9965624809265137,0.9978125095367432
6 2 0.995312511920929,0.9970625042915344,0.9984375238418579
7 2 0.9950000047683716,0.9967500567436218,0.9978125095367432
0 3 0.42875000834465027,0.460812509059906,0.5115625262260437
1 3 0.551562488079071,0.7398124933242798,0.8496875166893005
2 3 0.8421875238418579,0.9032500386238098,0.9696875214576721
3 3 0.981249988079071,0.9861874580383301,0.9931250214576721
4 3 0.9868749976158142,0.9903749227523804,0.9937499761581421
5 3 0.9825000166893005,0.9893125295639038,0.9928125143051147
6 3 0.9915624856948853,0.9966875314712524,0.9990624785423279
7 3 0.9896875023841858,0.9940624237060547,0.9993749856948853
0 4 0.2734375,0.3761875033378601,0.46406251192092896
1 4 0.5028125047683716,0.6109374761581421,0.6528124809265137
2 4 0.8084375262260437,0.8806875348091125,0.9659374952316284
3 4 0.9831249713897705,0.9860624074935913,0.9887499809265137
4 4 0.9809374809265137,0.9830624461174011,0.9859374761581421
5 4 0.9837499856948853,0.987250030040741,0.9903125166893005
6 4 0.981249988079071,0.9863125085830688,0.9903125166893005
7 4 0.9800000190734863,0.984125018119812,0.987500011920929
0 5 0.3284375071525574,0.390999972820282,0.4690625071525574
1 5 0.7446874976158142,0.7892500162124634,0.8500000238418579
2 5 0.9671875238418579,0.9817500114440918,0.9884374737739563
3 5 0.9846875071525574,0.9864374995231628,0.9900000095367432
4 5 0.9871875047683716,0.9888750314712524,0.9918749928474426
5 5 0.9856250286102295,0.9876874685287476,0.9896875023841858
6 5 0.9840624928474426,0.9876874685287476,0.9940624833106995
7 5 0.9809374809265137,0.9854375123977661,0.9871875047683716
0 6 0.3971875011920929,0.46287497878074646,0.518750011920929
1 6 0.7484375238418579,0.8446874618530273,0.9281250238418579
2 6 0.979687511920929,0.981624960899353,0.9837499856948853
3 6 0.984375,0.987125039100647,0.9931250214576721
4 6 0.9896875023841858,0.9919999837875366,0.9940624833106995
5 6 0.9906250238418579,0.9944999814033508,0.9971874952316284
6 6 0.9871875047683716,0.9921249151229858,0.9962499737739563
7 6 0.987500011920929,0.991812527179718,0.9946874976158142
0 7 0.2784374952316284,0.36168748140335083,0.4556249976158142
1 7 0.7271875143051147,0.7949374914169312,0.8262500166893005
2 7 0.9524999856948853,0.9614375233650208,0.9678124785423279
3 7 0.9806249737739563,0.9849374890327454,0.9900000095367432
4 7 0.9909374713897705,0.9919999837875366,0.9946874976158142
5 7 0.9909374713897705,0.9941250085830688,0.9965624809265137
6 7 0.9940624833106995,0.9951249957084656,0.9965624809265137
7 7 0.9931250214576721,0.9944999814033508,0.9962499737739563
0 8 0.22718749940395355,0.27012497186660767,0.3318749964237213
1 8 0.5615624785423279,0.6068124771118164,0.6712499856948853
2 8 0.8374999761581421,0.8728124499320984,0.8928124904632568
3 8 0.9403125047683716,0.9501250386238098,0.9612500071525574
4 8 0.9703124761581421,0.9744375348091125,0.9781249761581421
5 8 0.9856250286102295,0.9875624775886536,0.9906250238418579
6 8 0.9878125190734863,0.9910625219345093,0.9946874976158142
7 8 0.9934375286102295,0.9952500462532043,0.9981250166893005
0 9 0.1862500011920929,0.20481249690055847,0.23125000298023224
1 9 0.3440625071525574,0.3760624825954437,0.4399999976158142
2 9 0.6103125214576721,0.6350624561309814,0.6868749856948853
3 9 0.7796875238418579,0.8326875567436218,0.8628125190734863
4 9 0.9153125286102295,0.9196249842643738,0.9290624856948853
5 9 0.9556249976158142,0.9590624570846558,0.9646875262260437
6 9 0.9768750071525574,0.9801249504089355,0.9828125238418579
7 9 0.9865624904632568,0.9898125529289246,0.9931250214576721
};
\addplot[
only marks,
mark=x,
]
table {
x y
6 2
};

\nextgroupplot[title=NON (99.9\%), yticklabels={}]
\addplot[
matrix plot,
mesh/cols=8,
mesh/color input=explicit,
]
table[meta=rgb] {
x y rgb
0 0 0.1993750035762787,0.24012500047683716,0.3021875023841858
1 0 0.2006250023841858,0.2459374964237213,0.33156248927116394
2 0 0.13062499463558197,0.20062499940395356,0.23593750596046448
3 0 0.15062500536441803,0.2636250048875809,0.4243749976158142
4 0 0.20468750596046448,0.301500004529953,0.41593751311302185
5 0 0.22343750298023224,0.33793750405311584,0.5418750047683716
6 0 0.19843749701976776,0.3392500013113022,0.4675000011920929
7 0 0.2043749988079071,0.2674374967813492,0.3096874952316284
0 1 0.13218750059604645,0.237437504529953,0.31437501311302185
1 1 0.20999999344348907,0.2645625025033951,0.3321875035762787
2 1 0.20250000059604645,0.2626249998807907,0.41062501072883606
3 1 0.21593749523162842,0.28649999499320983,0.43187499046325684
4 1 0.14624999463558197,0.2961875021457672,0.4481250047683716
5 1 0.20499999821186066,0.41612500250339507,0.6565625071525574
6 1 0.14374999701976776,0.49449999928474425,0.9246875047683716
7 1 0.20125000178813934,0.4541250079870224,0.8103125095367432
0 2 0.1990624964237213,0.325874999165535,0.41999998688697815
1 2 0.38374999165534973,0.4794374942779541,0.5584375262260437
2 2 0.6274999976158142,0.6918750047683716,0.8031250238418579
3 2 0.5262500047683716,0.6664375066757202,0.9787499904632568
4 2 0.4424999952316284,0.6664374947547913,0.817187488079071
5 2 0.625,0.7290624976158142,0.9453125
6 2 0.5931249856948853,0.8289375066757202,0.9700000286102295
7 2 0.7737500071525574,0.8969374895095825,0.9865624904632568
0 3 0.3050000071525574,0.4124374985694885,0.5453125238418579
1 3 0.41749998927116394,0.690562492609024,0.8353124856948853
2 3 0.7746875286102295,0.8835624933242798,0.9815624952316284
3 3 0.7831249833106995,0.9281875014305114,0.9853125214576721
4 3 0.8534374833106995,0.9478125095367431,0.9868749976158142
5 3 0.8615624904632568,0.9549375057220459,0.9840624928474426
6 3 0.7628124952316284,0.8874374985694885,0.9909374713897705
7 3 0.7471874952316284,0.9354375004768372,0.9946874976158142
0 4 0.24124999344348907,0.3179375022649765,0.4753125011920929
1 4 0.2931250035762787,0.36843750476837156,0.535937488079071
2 4 0.49812498688697815,0.5876874983310699,0.7553125023841858
3 4 0.6818749904632568,0.8103750109672546,0.9928125143051147
4 4 0.7221875190734863,0.8350624918937684,0.9921875
5 4 0.8518750071525574,0.9608749985694885,0.9950000047683716
6 4 0.9918749928474426,0.9934374928474426,0.9943749904632568
7 4 0.9915624856948853,0.9938124895095826,0.9962499737739563
0 5 0.18968750536441803,0.21356249749660491,0.28312501311302185
1 5 0.19718749821186066,0.2451249986886978,0.32624998688697815
2 5 0.2874999940395355,0.30668750405311584,0.36937499046325684
3 5 0.3256250023841858,0.4034375071525574,0.47874999046325684
4 5 0.542187511920929,0.6449999928474426,0.7649999856948853
5 5 0.8740624785423279,0.9484999895095825,0.9971874952316284
6 5 0.9959375262260437,0.9975625038146972,0.9987499713897705
7 5 0.9871875047683716,0.989187502861023,0.9909374713897705
0 6 0.13187499344348907,0.16968750059604645,0.1993750035762787
1 6 0.21687500178813934,0.2375,0.28843748569488525
2 6 0.22843749821186066,0.3251249998807907,0.37843748927116394
3 6 0.36937499046325684,0.4745625078678131,0.5637500286102295
4 6 0.6443750262260437,0.7019375085830688,0.7699999809265137
5 6 0.8481249809265137,0.9684999942779541,0.9987499713897705
6 6 0.9975000023841858,0.9987499833106994,0.9993749856948853
7 6 0.9971874952316284,0.9980000019073486,0.9990624785423279
0 7 0.12250000238418579,0.1274375021457672,0.1365624964237213
1 7 0.14281250536441803,0.19987500160932542,0.2721875011920929
2 7 0.27125000953674316,0.28212499618530273,0.3071874976158142
3 7 0.6384375095367432,0.7128750085830688,0.7946875095367432
4 7 0.979687511920929,0.9868124842643737,0.9909374713897705
5 7 0.9846875071525574,0.9890625,0.9912499785423279
6 7 0.9896875023841858,0.9918750047683715,0.9984375238418579
7 7 0.9978125095367432,0.998812484741211,0.9996874928474426
0 8 0.12062499672174454,0.12381250113248825,0.1365624964237213
1 8 0.12218750268220901,0.12993750125169753,0.1393750011920929
2 8 0.14093750715255737,0.14793750047683715,0.1599999964237213
3 8 0.3128125071525574,0.45074999928474424,0.5306249856948853
4 8 0.9215624928474426,0.9320624947547913,0.9440624713897705
5 8 0.9787499904632568,0.9846874952316285,0.9878125190734863
6 8 0.9893749952316284,0.9908750057220459,0.9931250214576721
7 8 0.9915624856948853,0.9933125019073487,0.9950000047683716
0 9 0.11625000089406967,0.12331250160932541,0.13468749821186066
1 9 0.11625000089406967,0.12668750137090684,0.1431249976158142
2 9 0.13718749582767487,0.14112499952316285,0.1525000035762787
3 9 0.14624999463558197,0.1521875023841858,0.15906250476837158
4 9 0.5071874856948853,0.606375002861023,0.723437488079071
5 9 0.9478124976158142,0.9569374918937683,0.9631249904632568
6 9 0.9853125214576721,0.9863125085830688,0.9881250262260437
7 9 0.9896875023841858,0.9915624976158142,0.9931250214576721
};
\addplot[
only marks,
mark=x,
]
table {
x y
7 7
};

\nextgroupplot[title=sum (100\%), yticklabels={}]
\addplot[
matrix plot,
mesh/cols=8,
mesh/color input=explicit,
]
table[meta=rgb] {
x y rgb
0 0 0.41843751072883606,0.47587499022483826,0.5378124713897705
1 0 0.3956249952316284,0.6120000064373017,0.7981250286102295
2 0 0.6384375095367432,0.8602499961853027,0.9599999785423279
3 0 0.22156250476837158,0.8206249952316285,0.9834374785423279
4 0 0.96875,0.9827499985694885,0.9909374713897705
5 0 0.9787499904632568,0.9898750066757203,0.995312511920929
6 0 0.9934375286102295,0.9961249947547912,0.9987499713897705
7 0 0.9790624976158142,0.9911875009536744,0.9981250166893005
0 1 0.6528124809265137,0.6934999942779541,0.7209374904632568
1 1 0.9865624904632568,0.9886250019073486,0.9918749928474426
2 1 0.9962499737739563,0.996999990940094,0.9978125095367432
3 1 0.9946874976158142,0.9966250061988831,0.9984375238418579
4 1 0.9946874976158142,0.9955000162124634,0.9959375262260437
5 1 0.9925000071525574,0.9957499980926514,0.9990624785423279
6 1 0.9946874976158142,0.9956874966621398,0.9965624809265137
7 1 0.9956250190734863,0.9963124990463257,0.996874988079071
0 2 0.706250011920929,0.7160625100135803,0.7278125286102295
1 2 0.9399999976158142,0.9461875081062316,0.9540625214576721
2 2 0.9615625143051147,0.9785624980926514,0.9912499785423279
3 2 0.9950000047683716,0.9969375133514404,0.9984375238418579
4 2 0.9962499737739563,0.9975624918937683,0.9987499713897705
5 2 0.9950000047683716,0.9961249947547912,0.9971874952316284
6 2 0.9943749904632568,0.9961249947547912,0.9971874952316284
7 2 0.9934375286102295,0.9955625057220459,0.9975000023841858
0 3 0.7037500143051147,0.7216875076293945,0.7443749904632568
1 3 0.9700000286102295,0.9709375143051148,0.9737499952316284
2 3 0.9681249856948853,0.9707499980926514,0.9737499952316284
3 3 0.9653124809265137,0.9687499880790711,0.9731249809265137
4 3 0.9565625190734863,0.9660624980926513,0.9737499952316284
5 3 0.9668750166893005,0.9681874990463257,0.9706249833106995
6 3 0.9606249928474426,0.9691250085830688,0.9750000238418579
7 3 0.9637500047683716,0.9662500023841858,0.9681249856948853
0 4 0.7162500023841858,0.7309999942779541,0.7471874952316284
1 4 0.9881250262260437,0.9897500038146972,0.9909374713897705
2 4 0.987500011920929,0.990874993801117,0.9943749904632568
3 4 0.9893749952316284,0.9916249871253967,0.9937499761581421
4 4 0.9859374761581421,0.990874993801117,0.9943749904632568
5 4 0.9884374737739563,0.9906875014305114,0.9925000071525574
6 4 0.9859374761581421,0.9904999971389771,0.9934375286102295
7 4 0.9856250286102295,0.9888125061988831,0.9906250238418579
0 5 0.7106249928474426,0.7333125114440918,0.7496874928474426
1 5 0.9946874976158142,0.9961874842643738,0.996874988079071
2 5 0.9943749904632568,0.9961874961853028,0.9978125095367432
3 5 0.9959375262260437,0.9969375014305115,0.9975000023841858
4 5 0.9978125095367432,0.998312509059906,0.9996874928474426
5 5 0.9975000023841858,0.9981250166893005,0.9984375238418579
6 5 0.9962499737739563,0.9977499961853027,0.9990624785423279
7 5 0.995312511920929,0.9959375023841858,0.996874988079071
0 6 0.6918749809265137,0.7001874923706055,0.7231249809265137
1 6 0.9943749904632568,0.9953125,0.9962499737739563
2 6 0.9971874952316284,0.9981249928474426,0.9987499713897705
3 6 0.9978125095367432,0.9986875057220459,0.9996874928474426
4 6 0.9978125095367432,0.9983749985694885,0.9993749856948853
5 6 0.9987499713897705,0.9994374871253967,1.0
6 6 0.9990624785423279,0.9994374871253967,1.0
7 6 0.9990624785423279,0.9996874928474426,1.0
0 7 0.6256250143051147,0.6372500061988831,0.6496875286102295
1 7 0.9471874833106995,0.9578750014305115,0.9787499904632568
2 7 0.9887499809265137,0.9910624861717224,0.9937499761581421
3 7 0.995312511920929,0.9965625047683716,0.9984375238418579
4 7 0.995312511920929,0.9966874957084656,0.9978125095367432
5 7 0.9956250190734863,0.9961875200271606,0.9978125095367432
6 7 0.9962499737739563,0.9969374895095825,0.9981250166893005
7 7 0.9940624833106995,0.9958750009536743,0.9978125095367432
0 8 0.38968750834465027,0.44756250381469725,0.5146874785423279
1 8 0.690625011920929,0.7013750076293945,0.753125011920929
2 8 0.9165624976158142,0.9201249957084656,0.9331250190734863
3 8 0.9643750190734863,0.9681875109672546,0.973437488079071
4 8 0.9846875071525574,0.9866250157356262,0.9921875
5 8 0.9865624904632568,0.988937497138977,0.9903125166893005
6 8 0.9900000095367432,0.9914374947547913,0.9928125143051147
7 8 0.9918749928474426,0.9928750038146973,0.9946874976158142
0 9 0.2718749940395355,0.29899999499320984,0.3462499976158142
1 9 0.42625001072883606,0.5135000050067902,0.578125
2 9 0.698437511920929,0.7276874899864196,0.7528125047683716
3 9 0.8943750262260437,0.9065624952316285,0.926562488079071
4 9 0.9624999761581421,0.9664375066757203,0.9731249809265137
5 9 0.9787499904632568,0.9823750019073486,0.9862499833106995
6 9 0.9856250286102295,0.9874375104904175,0.9900000095367432
7 9 0.9893749952316284,0.9904374957084656,0.9925000071525574
};
\addplot[
only marks,
mark=x,
]
table {
x y
7 6
};

\nextgroupplot[title=max (14.9\%), yticklabels={}]
\addplot[
matrix plot,
mesh/cols=8,
mesh/color input=explicit,
]
table[meta=rgb] {
x y rgb
0 0 0.125,0.1359375,0.14374999701976776
1 0 0.12343750149011612,0.1339375004172325,0.1393750011920929
2 0 0.12562499940395355,0.1356249988079071,0.15156249701976776
3 0 0.1274999976158142,0.13431250154972077,0.14249999821186066
4 0 0.12656250596046448,0.13218750059604645,0.13593749701976776
5 0 0.12656250596046448,0.1290624976158142,0.13593749701976776
6 0 0.13562500476837158,0.1408749997615814,0.14562499523162842
7 0 0.13093750178813934,0.13568750023841858,0.1418749988079071
0 1 0.12999999523162842,0.13337499797344207,0.1459375023841858
1 1 0.13187499344348907,0.13812499940395356,0.14374999701976776
2 1 0.1353124976158142,0.14256249964237214,0.14624999463558197
3 1 0.1328125,0.13724999725818635,0.14124999940395355
4 1 0.1353124976158142,0.13850000202655793,0.1446875035762787
5 1 0.1315625011920929,0.1376875013113022,0.14499999582767487
6 1 0.1287499964237213,0.14118750095367433,0.15281249582767487
7 1 0.13781249523162842,0.14287500083446503,0.14781250059604645
0 2 0.13187499344348907,0.13793750405311583,0.1587499976158142
1 2 0.14031249284744263,0.14874999821186066,0.15968750417232513
2 2 0.1365624964237213,0.14225000143051147,0.1574999988079071
3 2 0.1353124976158142,0.1421249955892563,0.15000000596046448
4 2 0.13843749463558197,0.14424999952316284,0.1496874988079071
5 2 0.13437500596046448,0.14268749952316284,0.15312500298023224
6 2 0.13218750059604645,0.1380625009536743,0.1471875011920929
7 2 0.13437500596046448,0.14031250178813934,0.15343749523162842
0 3 0.1340624988079071,0.14118749797344207,0.15468749403953552
1 3 0.1381250023841858,0.14150000214576722,0.1525000035762787
2 3 0.14281250536441803,0.14687500298023223,0.15468749403953552
3 3 0.14031249284744263,0.14424999952316284,0.15687499940395355
4 3 0.13781249523162842,0.14149999916553496,0.15406249463558197
5 3 0.1381250023841858,0.14506250321865083,0.15718750655651093
6 3 0.1340624988079071,0.1393750011920929,0.14906249940395355
7 3 0.13593749701976776,0.14237500131130218,0.1484375
0 4 0.12062499672174454,0.12687500119209288,0.13562500476837158
1 4 0.12218750268220901,0.1276875004172325,0.14093750715255737
2 4 0.12406250089406967,0.13006250113248824,0.13968749344348907
3 4 0.12218750268220901,0.13193750232458115,0.14906249940395355
4 4 0.13062499463558197,0.13712499737739564,0.14781250059604645
5 4 0.12999999523162842,0.13506250083446503,0.14374999701976776
6 4 0.13343749940395355,0.13756250143051146,0.14749999344348907
7 4 0.12812499701976776,0.13506250083446503,0.14281250536441803
0 5 0.12312500178813934,0.12649999856948851,0.1328125
1 5 0.12250000238418579,0.12712500095367432,0.14156250655651093
2 5 0.12218750268220901,0.12731249779462814,0.13249999284744263
3 5 0.11968749761581421,0.1252499997615814,0.13312500715255737
4 5 0.11812499910593033,0.12537500113248826,0.13718749582767487
5 5 0.11999999731779099,0.12681250274181366,0.13312500715255737
6 5 0.12562499940395355,0.12693749964237214,0.13375000655651093
7 5 0.12656250596046448,0.13106250166893005,0.14124999940395355
0 6 0.11999999731779099,0.12650000005960466,0.13593749701976776
1 6 0.12406250089406967,0.12918749898672105,0.14031249284744263
2 6 0.11812499910593033,0.12481249868869781,0.1315625011920929
3 6 0.11999999731779099,0.1265625,0.1315625011920929
4 6 0.12125000357627869,0.1250000014901161,0.1353124976158142
5 6 0.11906249821186066,0.1235624983906746,0.13500000536441803
6 6 0.12250000238418579,0.12718749940395355,0.13437500596046448
7 6 0.12531250715255737,0.13206250071525574,0.14374999701976776
0 7 0.12343750149011612,0.12568750083446503,0.13625000417232513
1 7 0.11937499791383743,0.12650000154972077,0.13875000178813934
2 7 0.12093749642372131,0.12556250095367433,0.13468749821186066
3 7 0.11906249821186066,0.12537499964237214,0.13593749701976776
4 7 0.12031249701976776,0.12556249797344207,0.14031249284744263
5 7 0.11749999970197678,0.12393750250339508,0.13562500476837158
6 7 0.12218750268220901,0.12706250101327896,0.13218750059604645
7 7 0.125,0.13062500059604645,0.13906249403953552
0 8 0.12156250327825546,0.12518749982118607,0.13468749821186066
1 8 0.12031249701976776,0.12462500035762787,0.1381250023841858
2 8 0.12156250327825546,0.12612499594688414,0.13718749582767487
3 8 0.12062499672174454,0.12468749731779098,0.14031249284744263
4 8 0.12218750268220901,0.12668750137090684,0.1353124976158142
5 8 0.12156250327825546,0.12556250095367433,0.13218750059604645
6 8 0.12593750655651093,0.13081249892711638,0.13906249403953552
7 8 0.1303125023841858,0.1398125022649765,0.14656250178813934
0 9 0.11812499910593033,0.12537500113248826,0.13500000536441803
1 9 0.12218750268220901,0.12462500035762787,0.13218750059604645
2 9 0.12531250715255737,0.12850000262260436,0.13437500596046448
3 9 0.12187500298023224,0.1264375016093254,0.13437500596046448
4 9 0.12125000357627869,0.1261875033378601,0.13468749821186066
5 9 0.12406250089406967,0.13112500160932541,0.14562499523162842
6 9 0.1290625035762787,0.13599999845027924,0.1521874964237213
7 9 0.13093750178813934,0.1357499986886978,0.14374999701976776
};
\addplot[
only marks,
mark=x,
color=white,
]
table {
x y
1 2
};

\end{groupplot}
\end{tikzpicture}

%% file: plots/cases/var_batch_cases+cases_val.tex

\begin{tikzpicture}

\begin{groupplot}[
group style={
    group size=6 by 2,
    horizontal sep=0.1cm,
    vertical sep=0.85cm,
    xlabels at=edge bottom,
    ylabels at=edge left,
},
height=4.3cm, 
width=1.55*\columnwidth/6,
enlargelimits=false,
xtick=data, ytick=data, xticklabel style={font=\small, rotate=90}, yticklabel style={font=\small}, title style={at={(0.5,0.9)}, font=\small}]

\nextgroupplot[title=NAP, yticklabels={0.3,0.09,0.027,0.008,0.002,7.3e-4,2.2e-4,6.6e-5,2.0e-5,5.9e-6}, xticklabels={1,2,4,8,16,32,64,128}]
\addplot[
matrix plot,
mesh/cols=8,
mesh/color input=explicit,
]
table[meta=rgb] {
x y rgb
0 0 0.9874,1.0,0.968
1 0 0.9880000000000001,0.9998000000000001,0.959
2 0 0.9902,1.0,0.9633999999999998
3 0 0.9890000000000001,1.0,0.959
4 0 0.9856,1.0,0.9658
5 0 0.9873999999999998,0.9996,0.9488
6 0 0.938,0.8392,0.7146000000000001
7 0 0.8398,0.5282,0.1842
0 1 0.9885999999999999,1.0,0.9763999999999999
1 1 0.9890000000000001,1.0,0.977
2 1 0.99,1.0,0.9746
3 1 0.9927999999999999,1.0,0.9761999999999998
4 1 0.9926,1.0,0.9722
5 1 0.9914,0.9998000000000001,0.9638
6 1 0.9892,1.0,0.9614
7 1 0.8539999999999999,0.538,0.3596
0 2 0.9911999999999999,1.0,0.9799999999999999
1 2 0.9958,1.0,0.9743999999999999
2 2 0.9934,1.0,0.9808
3 2 0.9932000000000001,1.0,0.9798
4 2 0.9914,1.0,0.9789999999999999
5 2 0.9925999999999998,1.0,0.9818
6 2 0.9914,1.0,0.9773999999999999
7 2 0.9902,1.0,0.9757999999999999
0 3 0.9949999999999999,1.0,0.9856
1 3 0.9944000000000001,1.0,0.9856
2 3 0.9945999999999999,1.0,0.986
3 3 0.994,1.0,0.9847999999999999
4 3 0.9922000000000001,1.0,0.983
5 3 0.9917999999999999,1.0,0.977
6 3 0.9924,1.0,0.9762000000000001
7 3 0.991,1.0,0.9739999999999999
0 4 0.9940000000000001,1.0,0.9885999999999999
1 4 0.9942,1.0,0.9847999999999999
2 4 0.9924,1.0,0.9875999999999999
3 4 0.9932000000000001,1.0,0.9826
4 4 0.9904,1.0,0.974
5 4 0.9904,1.0,0.9728000000000001
6 4 0.9927999999999999,1.0,0.9771999999999998
7 4 0.9886000000000001,1.0,0.9764000000000002
0 5 0.993,1.0,0.9865999999999999
1 5 0.9936,1.0,0.9843999999999999
2 5 0.9926,1.0,0.9822
3 5 0.9916,1.0,0.9812
4 5 0.9936,1.0,0.9789999999999999
5 5 0.993,1.0,0.9755999999999998
6 5 0.9906,1.0,0.9792
7 5 0.9914,1.0,0.9736
0 6 0.9936,1.0,0.9827999999999999
1 6 0.992,1.0,0.9815999999999999
2 6 0.9934,1.0,0.9816
3 6 0.9949999999999999,1.0,0.9782
4 6 0.9936,1.0,0.9767999999999999
5 6 0.9914,1.0,0.9789999999999999
6 6 0.9917999999999999,1.0,0.9718
7 6 0.9887999999999998,1.0,0.9634
0 7 0.9938,1.0,0.9816
1 7 0.9936,1.0,0.9788
2 7 0.9936,1.0,0.9783999999999999
3 7 0.991,1.0,0.9736
4 7 0.9863999999999999,1.0,0.9646000000000001
5 7 0.9822,0.9997999999999999,0.9555999999999999
6 7 0.9795999999999999,1.0,0.9352
7 7 0.9655999999999999,0.999,0.9002000000000001
0 8 0.99,1.0,0.9725999999999999
1 8 0.9862,1.0,0.9607999999999999
2 8 0.9842000000000001,1.0,0.9551999999999999
3 8 0.9751999999999998,0.9996,0.9352
4 8 0.9648,0.999,0.9132
5 8 0.951,0.9964000000000001,0.8708
6 8 0.9258,0.9810000000000001,0.7446
7 8 0.8734,0.9621999999999999,0.21800000000000003
0 9 0.9728,0.9992000000000001,0.9276
1 9 0.9606,0.9982,0.8968
2 9 0.9452,0.9949999999999999,0.8457999999999999
3 9 0.9254,0.9782,0.6236
4 9 0.8846,0.9541999999999999,0.3408
5 9 0.8253999999999999,0.9352,0.1456
6 9 0.7323999999999999,0.5284,0.0216
7 9 0.6404,0.2358,0.010600000000000002
};
\nextgroupplot[title=no norm, yticklabels={}, xticklabels={1,2,4,8,16,32,64,128}]
\addplot[
matrix plot,
mesh/cols=8,
mesh/color input=explicit,
]
table[meta=rgb] {
x y rgb
0 0 0.8637999999999998,0.9464,0.697
1 0 0.9792,0.9996,0.9419999999999998
2 0 0.9308,0.8326,0.5948
3 0 0.9784,0.9486000000000001,0.6522
4 0 0.9030000000000001,0.6696000000000001,0.4154
5 0 0.921,0.8039999999999999,0.23019999999999996
6 0 0.9321999999999999,0.8253999999999999,0.26199999999999996
7 0 0.7647999999999999,0.2648,0.0804
0 1 0.9894000000000001,1.0,0.9727999999999998
1 1 0.9894000000000001,1.0,0.9562000000000002
2 1 0.9882,1.0,0.9541999999999999
3 1 0.9888,0.9994,0.9517999999999999
4 1 0.943,0.8446,0.6006
5 1 0.9884000000000001,0.9997999999999999,0.9545999999999999
6 1 0.9865999999999999,0.9960000000000001,0.7784
7 1 0.9405999999999999,0.9356,0.1908
0 2 0.993,1.0,0.9715999999999999
1 2 0.991,0.9998000000000001,0.9667999999999999
2 2 0.9928000000000001,1.0,0.9743999999999999
3 2 0.9907999999999999,1.0,0.9785999999999999
4 2 0.9926,1.0,0.9788
5 2 0.9902,0.9996,0.9468
6 2 0.9904,0.9996,0.7844
7 2 0.9468,0.8448,0.754
0 3 0.9934,1.0,0.9814
1 3 0.9949999999999999,1.0,0.982
2 3 0.9928000000000001,1.0,0.9808
3 3 0.9926,0.999,0.7882
4 3 0.9928000000000001,1.0,0.9807999999999998
5 3 0.9932000000000001,1.0,0.9795999999999999
6 3 0.991,1.0,0.9708
7 3 0.9898,1.0,0.9669999999999999
0 4 0.9936,1.0,0.9843999999999999
1 4 0.9947999999999999,1.0,0.9869999999999999
2 4 0.9959999999999999,1.0,0.9809999999999999
3 4 0.9934,1.0,0.9755999999999998
4 4 0.9931999999999999,1.0,0.9818000000000001
5 4 0.991,1.0,0.9762000000000001
6 4 0.9926,1.0,0.9751999999999998
7 4 0.9907999999999999,1.0,0.968
0 5 0.9932000000000001,1.0,0.9852000000000001
1 5 0.9936,1.0,0.9838000000000001
2 5 0.993,1.0,0.9846
3 5 0.9922000000000001,1.0,0.9748000000000001
4 5 0.993,0.9998000000000001,0.9752000000000001
5 5 0.992,0.9996,0.9738
6 5 0.9906,0.9998000000000001,0.9648
7 5 0.9880000000000001,0.999,0.961
0 6 0.9934000000000001,1.0,0.9795999999999999
1 6 0.9922000000000001,1.0,0.9766
2 6 0.9924,1.0,0.9782
3 6 0.9932000000000001,1.0,0.9763999999999999
4 6 0.9918000000000001,0.9998000000000001,0.974
5 6 0.9918000000000001,0.9996,0.9667999999999999
6 6 0.9865999999999999,0.9998000000000001,0.9558
7 6 0.9848000000000001,0.9992000000000001,0.7496
0 7 0.9936,0.9997999999999999,0.9800000000000001
1 7 0.9928000000000001,1.0,0.974
2 7 0.991,0.9996,0.9718
3 7 0.9914,0.9996,0.968
4 7 0.9898,0.9996,0.9597999999999999
5 7 0.9853999999999999,0.9984,0.874
6 7 0.9728,0.9984,0.3574
7 7 0.9491999999999999,0.9998000000000001,0.032999999999999995
0 8 0.991,0.9996,0.974
1 8 0.9904,0.9998000000000001,0.9592
2 8 0.9826,0.9994,0.8928
3 8 0.9781999999999998,0.9996,0.363
4 8 0.967,0.9974000000000001,0.0848
5 8 0.9356,0.9992000000000001,0.022000000000000002
6 8 0.7916000000000001,0.40700000000000003,0.0092
7 8 0.767,0.24559999999999998,0.0076
0 9 0.9734,0.998,0.0728
1 9 0.9551999999999999,1.0,0.0206
2 9 0.9117999999999998,1.0,0.0188
3 9 0.7776,0.2832,0.006999999999999999
4 9 0.7788,0.24779999999999996,0.006599999999999999
5 9 0.7448,0.26280000000000003,0.007400000000000001
6 9 0.7190000000000001,0.2452,0.008
7 9 0.6656000000000001,0.2556,0.006199999999999999
};
\nextgroupplot[title=softmax, yticklabels={}, xticklabels={1,2,4,8,16,32,64,128}]
\addplot[
matrix plot,
mesh/cols=8,
mesh/color input=explicit,
]
table[meta=rgb] {
x y rgb
0 0 0.813,0.5568,0.15940000000000004
1 0 0.8048,0.47520000000000007,0.11660000000000001
2 0 0.8764,0.5938,0.164
3 0 0.8059999999999998,0.37579999999999997,0.08979999999999999
4 0 0.7488,0.32939999999999997,0.06720000000000001
5 0 0.7458,0.2872,0.06899999999999999
6 0 0.7409999999999999,0.2772,0.06440000000000001
7 0 0.7578,0.35,0.0638
0 1 0.8901999999999999,0.6462,0.1874
1 1 0.8793999999999998,0.6539999999999999,0.15920000000000004
2 1 0.8009999999999999,0.3806,0.101
3 1 0.8699999999999999,0.6196,0.1102
4 1 0.7984,0.39840000000000003,0.06640000000000001
5 1 0.8089999999999999,0.4008,0.06760000000000002
6 1 0.8620000000000001,0.40259999999999996,0.058399999999999994
7 1 0.7706,0.28259999999999996,0.038
0 2 0.9688000000000001,0.959,0.23120000000000002
1 2 0.9038,0.6848,0.09940000000000002
2 2 0.8855999999999999,0.6544,0.07780000000000001
3 2 0.8178000000000001,0.409,0.057600000000000005
4 2 0.817,0.5067999999999999,0.0368
5 2 0.7910000000000001,0.3758,0.023
6 2 0.9363999999999999,0.8236000000000001,0.2212
7 2 0.9714,0.9844000000000002,0.1558
0 3 0.9818,0.9872,0.1068
1 3 0.9832000000000001,0.9945999999999999,0.11359999999999999
2 3 0.9874,0.9890000000000001,0.1722
3 3 0.9682000000000001,0.9346,0.1608
4 3 0.9862,0.9904,0.29359999999999997
5 3 0.9853999999999999,0.9942,0.2072
6 3 0.9924000000000002,0.9996,0.9597999999999999
7 3 0.9875999999999999,0.9994,0.9446
0 4 0.9890000000000001,0.993,0.358
1 4 0.9894000000000001,0.9984,0.568
2 4 0.99,0.9987999999999999,0.587
3 4 0.9936,0.9998000000000001,0.968
4 4 0.9953999999999998,1.0,0.9730000000000001
5 4 0.9927999999999999,0.9998000000000001,0.9698
6 4 0.9922000000000001,0.9998000000000001,0.9621999999999999
7 4 0.9892,0.9945999999999999,0.41000000000000003
0 5 0.9926,0.9996,0.969
1 5 0.9931999999999999,0.9996,0.9718
2 5 0.9914,0.999,0.9693999999999999
3 5 0.992,0.9996,0.9715999999999999
4 5 0.9927999999999999,0.9997999999999999,0.974
5 5 0.9898,1.0,0.9695999999999998
6 5 0.9886000000000001,0.9987999999999999,0.9418
7 5 0.9826,0.9962,0.44279999999999997
0 6 0.9916,0.9998000000000001,0.982
1 6 0.9934000000000001,1.0,0.9766
2 6 0.994,1.0,0.9823999999999999
3 6 0.9940000000000001,1.0,0.9767999999999999
4 6 0.993,1.0,0.9742
5 6 0.992,1.0,0.9611999999999998
6 6 0.9898,0.9996,0.942
7 6 0.9875999999999999,0.9996,0.9
0 7 0.9948,1.0,0.9785999999999999
1 7 0.9942,1.0,0.9751999999999998
2 7 0.9937999999999999,1.0,0.975
3 7 0.9914,1.0,0.9612
4 7 0.9924,1.0,0.9480000000000001
5 7 0.9875999999999999,1.0,0.9304
6 7 0.977,0.9998000000000001,0.8960000000000001
7 7 0.9638,0.9906,0.7222000000000001
0 8 0.9922000000000001,1.0,0.9572
1 8 0.9911999999999999,1.0,0.9548
2 8 0.9865999999999999,1.0,0.9378
3 8 0.9837999999999999,0.9998000000000001,0.915
4 8 0.9751999999999998,0.9998000000000001,0.8876
5 8 0.9652,0.9868,0.7726000000000001
6 8 0.9410000000000001,0.9931999999999999,0.22799999999999998
7 8 0.8379999999999999,0.9328,0.0274
0 9 0.977,0.9998000000000001,0.909
1 9 0.9742000000000001,0.999,0.8860000000000001
2 9 0.9587999999999999,0.9863999999999999,0.805
3 9 0.9380000000000001,0.9823999999999999,0.5716000000000001
4 9 0.906,0.9954000000000001,0.07640000000000001
5 9 0.8074,0.7884,0.0284
6 9 0.7226000000000001,0.24739999999999998,0.0178
7 9 0.6462000000000001,0.25,0.0098
};
\nextgroupplot[title=BERT, yticklabels={}, xticklabels={1,2,4,8,16,32,64,128}]
\addplot[
matrix plot,
mesh/cols=8,
mesh/color input=explicit,
]
table[meta=rgb] {
x y rgb
0 0 0.0214,0.9338000000000001,0.018000000000000002
1 0 0.0262,0.9044000000000001,0.0182
2 0 0.0182,0.8792,0.018199999999999997
3 0 0.0256,0.5814,0.019
4 0 0.21020000000000003,0.5568000000000001,0.0216
5 0 0.0514,0.43220000000000003,0.0202
6 0 0.0204,0.6122000000000001,0.0194
7 0 0.031199999999999995,0.26480000000000004,0.0242
0 1 0.0752,0.9526,0.0164
1 1 0.0316,0.7777999999999999,0.0206
2 1 0.030800000000000004,0.5372,0.0198
3 1 0.04660000000000001,0.5820000000000001,0.0226
4 1 0.1224,0.39139999999999997,0.0206
5 1 0.18639999999999998,0.5504,0.0162
6 1 0.11280000000000001,0.3592,0.016000000000000004
7 1 0.252,0.1488,0.0508
0 2 0.16100000000000003,0.6786000000000001,0.0182
1 2 0.33840000000000003,0.606,0.0164
2 2 0.335,0.315,0.019200000000000002
3 2 0.2284,0.24800000000000005,0.015599999999999998
4 2 0.2652,0.3096,0.0398
5 2 0.6462,0.19440000000000002,0.054000000000000006
6 2 0.6058,0.30839999999999995,0.03420000000000001
7 2 0.5346,0.13319999999999999,0.0308
0 3 0.5322,0.31539999999999996,0.06140000000000001
1 3 0.4728,0.3348,0.054400000000000004
2 3 0.6444,0.2852,0.0218
3 3 0.6262000000000001,0.5940000000000001,0.08880000000000002
4 3 0.6726000000000001,0.4354000000000001,0.0524
5 3 0.8099999999999999,0.9296,0.2684
6 3 0.6742000000000001,0.5096,0.1998
7 3 0.9004,0.8099999999999999,0.5411999999999999
0 4 0.7442,0.8532,0.1712
1 4 0.9648,0.9978,0.7664
2 4 0.99,1.0,0.9745999999999999
3 4 0.9942,1.0,0.9753999999999999
4 4 0.9934,1.0,0.9715999999999999
5 4 0.9944,1.0,0.9719999999999999
6 4 0.9938,1.0,0.9714
7 4 0.9907999999999999,1.0,0.9544
0 5 0.993,1.0,0.978
1 5 0.9932000000000001,1.0,0.9795999999999999
2 5 0.9922000000000001,1.0,0.9812
3 5 0.9926,1.0,0.9800000000000001
4 5 0.9936,1.0,0.9753999999999999
5 5 0.9938,1.0,0.9756
6 5 0.9914,1.0,0.9635999999999999
7 5 0.9889999999999999,0.9994,0.8288
0 6 0.9907999999999999,1.0,0.9795999999999999
1 6 0.9914,1.0,0.9738
2 6 0.9890000000000001,0.9998000000000001,0.9698
3 6 0.9904,1.0,0.9591999999999998
4 6 0.9904,0.9996,0.9554
5 6 0.9898,0.9998000000000001,0.9468
6 6 0.985,0.999,0.921
7 6 0.9738,0.9984,0.6748000000000001
0 7 0.9886000000000001,0.9987999999999999,0.9532
1 7 0.9858,1.0,0.767
2 7 0.9818,0.9986,0.727
3 7 0.9809999999999999,0.9992000000000001,0.8816
4 7 0.9815999999999999,0.99,0.5082
5 7 0.9791999999999998,0.9934,0.5948
6 7 0.8757999999999999,0.7832,0.0304
7 7 0.8596,0.27699999999999997,0.0632
0 8 0.9720000000000001,0.9982,0.859
1 8 0.9587999999999999,0.9938,0.6736000000000001
2 8 0.9432,0.9593999999999999,0.1518
3 8 0.9581999999999999,0.7986,0.16019999999999998
4 8 0.8728000000000001,0.2702,0.05740000000000001
5 8 0.8606,0.2632,0.0262
6 8 0.757,0.2594,0.027200000000000002
7 8 0.6841999999999999,0.256,0.014599999999999998
0 9 0.915,0.8564,0.029599999999999998
1 9 0.865,0.28040000000000004,0.0434
2 9 0.7622,0.2688,0.013600000000000001
3 9 0.7702,0.2698,0.0064
4 9 0.698,0.261,0.0398
5 9 0.651,0.27680000000000005,0.027200000000000002
6 9 0.6718,0.19459999999999997,0.0466
7 9 0.621,0.0668,0.05419999999999999
};
\nextgroupplot[title=max pooling, yticklabels={}, xticklabels={1,2,4,8,16,32,64,128}]
\addplot[
matrix plot,
mesh/cols=8,
mesh/color input=explicit,
]
table[meta=rgb] {
x y rgb
0 0 0.9902000000000001,1.0,0.9667999999999999
1 0 0.9882,1.0,0.9654
2 0 0.9907999999999999,1.0,0.9757999999999999
3 0 0.9898000000000001,1.0,0.9672000000000001
4 0 0.9912000000000001,1.0,0.9650000000000001
5 0 0.9904,0.9810000000000001,0.6752
6 0 0.9902,1.0,0.9621999999999999
7 0 0.9853999999999999,0.9964000000000001,0.8535999999999999
0 1 0.9894000000000001,1.0,0.9694
1 1 0.9904,0.9996,0.9366
2 1 0.9916,0.9827999999999999,0.7796
3 1 0.9902,0.9964000000000001,0.7854000000000001
4 1 0.9909999999999999,1.0,0.9645999999999999
5 1 0.9904,1.0,0.9687999999999999
6 1 0.992,0.9982,0.7842
7 1 0.9898,1.0,0.9654
0 2 0.9924,1.0,0.9686
1 2 0.992,1.0,0.9753999999999999
2 2 0.9916,1.0,0.9758000000000001
3 2 0.9926,1.0,0.9838000000000001
4 2 0.9922000000000001,1.0,0.9786000000000001
5 2 0.9945999999999999,1.0,0.9795999999999999
6 2 0.9936,1.0,0.983
7 2 0.9924,1.0,0.9686
0 3 0.9940000000000001,1.0,0.9816
1 3 0.9949999999999999,1.0,0.9838000000000001
2 3 0.9942,1.0,0.9825999999999999
3 3 0.9936,1.0,0.9798
4 3 0.9945999999999999,1.0,0.9784
5 3 0.9938,1.0,0.9783999999999999
6 3 0.9928000000000001,1.0,0.9772000000000001
7 3 0.9926000000000001,1.0,0.9767999999999999
0 4 0.994,1.0,0.9867999999999999
1 4 0.9949999999999999,1.0,0.9826
2 4 0.9952000000000002,1.0,0.9814
3 4 0.992,1.0,0.9804
4 4 0.994,1.0,0.9730000000000001
5 4 0.9938,1.0,0.9752000000000001
6 4 0.9914,1.0,0.9728
7 4 0.9898,1.0,0.9709999999999999
0 5 0.9940000000000001,1.0,0.9812
1 5 0.993,1.0,0.9748000000000001
2 5 0.9948,1.0,0.9752000000000001
3 5 0.9945999999999999,1.0,0.9728
4 5 0.9949999999999999,1.0,0.975
5 5 0.992,1.0,0.9724
6 5 0.9930000000000001,1.0,0.9692000000000001
7 5 0.9921999999999999,0.9998000000000001,0.967
0 6 0.992,1.0,0.9730000000000001
1 6 0.9948,1.0,0.9702
2 6 0.994,1.0,0.9757999999999999
3 6 0.9938,1.0,0.9688000000000001
4 6 0.9948,1.0,0.9682000000000001
5 6 0.9922000000000001,1.0,0.9652000000000001
6 6 0.9904,0.999,0.9339999999999999
7 6 0.9862,0.9917999999999999,0.7976
0 7 0.9934,0.9998000000000001,0.9728
1 7 0.994,1.0,0.9673999999999999
2 7 0.9932000000000001,1.0,0.9686
3 7 0.9916,1.0,0.9523999999999999
4 7 0.9888,0.9964000000000001,0.8934000000000001
5 7 0.9841999999999999,0.9934,0.3088
6 7 0.9480000000000001,0.9564,0.0218
7 7 0.791,0.24779999999999996,0.0176
0 8 0.9916,0.9994,0.9450000000000001
1 8 0.9878,0.998,0.9064
2 8 0.9808,0.9932000000000001,0.6277999999999999
3 8 0.9638,0.9875999999999999,0.0516
4 8 0.8646,0.698,0.017599999999999998
5 8 0.7798,0.2448,0.0172
6 8 0.7222000000000001,0.25739999999999996,0.015200000000000002
7 8 0.7312000000000001,0.213,0.01
0 9 0.9551999999999999,0.9832000000000001,0.056400000000000006
1 9 0.8734,0.8863999999999999,0.0186
2 9 0.7672,0.25860000000000005,0.016800000000000002
3 9 0.7380000000000001,0.24720000000000003,0.014799999999999999
4 9 0.687,0.2598,0.015
5 9 0.6985999999999999,0.2092,0.0122
6 9 0.6768,0.187,0.0108
7 9 0.6026,0.23600000000000004,0.012
};
\nextgroupplot[title=sum pooling, yticklabels={}, xticklabels={1,2,4,8,16,32,64,128}]
\addplot[
matrix plot,
mesh/cols=8,
mesh/color input=explicit,
]
table[meta=rgb] {
x y rgb
0 0 0.9772000000000001,0.9985999999999999,0.8662000000000001
1 0 0.9822000000000001,0.9986,0.8736
2 0 0.9742,0.9911999999999999,0.8432000000000001
3 0 0.9812,0.998,0.8804000000000001
4 0 0.9800000000000001,0.9946000000000002,0.8468
5 0 0.9808,0.9873999999999998,0.8218
6 0 0.9277999999999998,0.8396000000000001,0.6566
7 0 0.9054,0.7766,0.33440000000000003
0 1 0.9841999999999999,0.999,0.8986000000000001
1 1 0.9848000000000001,1.0,0.908
2 1 0.9858,0.9996,0.9
3 1 0.9858,0.9970000000000001,0.8922000000000001
4 1 0.9884000000000001,0.9992000000000001,0.9076000000000001
5 1 0.9846,0.991,0.8827999999999999
6 1 0.9841999999999999,0.9865999999999999,0.8058
7 1 0.9714,0.9571999999999999,0.2138
0 2 0.9898,1.0,0.9304
1 2 0.9928000000000001,0.9996,0.9376
2 2 0.9890000000000001,0.999,0.9346
3 2 0.9926,1.0,0.9485999999999999
4 2 0.9922000000000001,0.9991999999999999,0.9458
5 2 0.9869999999999999,0.999,0.9385999999999999
6 2 0.9875999999999999,0.9974000000000001,0.9308
7 2 0.9856,0.9722,0.769
0 3 0.994,1.0,0.9652
1 3 0.9928000000000001,1.0,0.9676
2 3 0.9934,1.0,0.9722
3 3 0.9919999999999998,1.0,0.9667999999999999
4 3 0.9914,0.9992000000000001,0.9486000000000001
5 3 0.9924,0.9994,0.9549999999999998
6 3 0.9862,0.9894000000000001,0.9132
7 3 0.9814,0.9678000000000001,0.46840000000000004
0 4 0.993,1.0,0.9734
1 4 0.9945999999999999,1.0,0.9667999999999999
2 4 0.9936,1.0,0.9658
3 4 0.9926,1.0,0.9648
4 4 0.9916,0.9998000000000001,0.959
5 4 0.9890000000000001,0.9998000000000001,0.95
6 4 0.9869999999999999,0.9896,0.8374
7 4 0.9768000000000001,0.9526,0.4544
0 5 0.9937999999999999,0.9998000000000001,0.9696
1 5 0.993,1.0,0.9710000000000001
2 5 0.9912000000000001,1.0,0.9662
3 5 0.9912000000000001,0.9996,0.9616
4 5 0.9927999999999999,0.9998000000000001,0.9658
5 5 0.985,0.9968,0.8926000000000001
6 5 0.9822,0.9916,0.8748000000000001
7 5 0.9705999999999999,0.9718,0.7496
0 6 0.9931999999999999,0.9998000000000001,0.9738
1 6 0.9932000000000001,1.0,0.9712
2 6 0.9907999999999999,1.0,0.9668000000000001
3 6 0.9926,1.0,0.9642
4 6 0.9888,0.9994,0.9517999999999999
5 6 0.9846,0.9987999999999999,0.9212
6 6 0.9697999999999999,0.9958,0.8308
7 6 0.9346,0.9638,0.36819999999999997
0 7 0.9896,0.9998000000000001,0.9676
1 7 0.991,1.0,0.9598000000000001
2 7 0.9878,0.9991999999999999,0.9483999999999998
3 7 0.9800000000000001,0.9998000000000001,0.9146000000000001
4 7 0.9709999999999999,0.9974000000000001,0.86
5 7 0.9442,0.9875999999999999,0.6494
6 7 0.8539999999999999,0.9,0.023799999999999998
7 7 0.715,0.313,0.0164
0 8 0.9843999999999999,0.9994,0.9246000000000001
1 8 0.974,0.9996,0.8902000000000001
2 8 0.9667999999999999,0.9978,0.834
3 8 0.9461999999999999,0.9869999999999999,0.5186
4 8 0.8392,0.818,0.018
5 8 0.6526,0.3014,0.013999999999999999
6 8 0.5863999999999999,0.3094,0.008199999999999999
7 8 0.5992,0.2994,0.0128
0 9 0.9414,0.9895999999999999,0.5346
1 9 0.8775999999999999,0.9410000000000001,0.030799999999999994
2 9 0.783,0.5349999999999999,0.0182
3 9 0.6028,0.3082,0.0112
4 9 0.5668,0.3164,0.007200000000000001
5 9 0.5668,0.29519999999999996,0.006
6 9 0.5736000000000001,0.2944,0.008200000000000002
7 9 0.5668,0.2926,0.0132
};
\nextgroupplot[yticklabels={0.3,0.09,0.027,0.008,0.002,7.3e-4,2.2e-4,6.6e-5,2.0e-5,5.9e-6}, xticklabels={}]
\addplot[
matrix plot,
mesh/cols=8,
mesh/color input=explicit,
]
table[meta=rgb] {
x y rgb
0 0 0.9732,0.9994,0.9288000000000001
1 0 0.9748000000000001,0.999,0.9120000000000001
2 0 0.9718,0.9998000000000001,0.9328
3 0 0.969,0.9996,0.9206000000000001
4 0 0.9724,1.0,0.9098
5 0 0.9747999999999999,0.9994,0.915
6 0 0.8896000000000001,0.868,0.641
7 0 0.7298,0.6426000000000001,0.18040000000000003
0 1 0.9794,0.9996,0.931
1 1 0.9778,1.0,0.9391999999999999
2 1 0.9766000000000001,1.0,0.9322000000000001
3 1 0.9803999999999998,1.0,0.9414
4 1 0.978,1.0,0.9268000000000001
5 1 0.9729999999999999,0.9998000000000001,0.8954000000000001
6 1 0.9734,0.999,0.9118
7 1 0.7548,0.6521999999999999,0.33199999999999996
0 2 0.9789999999999999,1.0,0.943
1 2 0.9804,1.0,0.9362
2 2 0.9795999999999999,1.0,0.944
3 2 0.9795999999999999,1.0,0.9299999999999999
4 2 0.9745999999999999,1.0,0.9422
5 2 0.9762000000000001,0.9987999999999999,0.9328
6 2 0.9732,1.0,0.9328000000000001
7 2 0.967,0.9996,0.9179999999999999
0 3 0.9827999999999999,0.9998000000000001,0.9454
1 3 0.9816,1.0,0.9527999999999999
2 3 0.9804,0.9968,0.9470000000000001
3 3 0.9794,0.9996,0.9427999999999999
4 3 0.9778,0.9982,0.9188000000000001
5 3 0.9777999999999999,0.9998000000000001,0.9324
6 3 0.9745999999999999,1.0,0.9216000000000001
7 3 0.9628,1.0,0.9114000000000001
0 4 0.9782,0.9998000000000001,0.9380000000000001
1 4 0.982,0.9998000000000001,0.9373999999999999
2 4 0.9852000000000001,0.9992000000000001,0.9454
3 4 0.9767999999999999,0.9952,0.9182
4 4 0.966,0.9997999999999999,0.917
5 4 0.9673999999999999,0.9996,0.9178
6 4 0.9663999999999999,0.9996,0.9146000000000001
7 4 0.9575999999999999,0.9998000000000001,0.9202
0 5 0.9812,0.9987999999999999,0.9372
1 5 0.9785999999999999,1.0,0.9404
2 5 0.9784,0.9994,0.9339999999999999
3 5 0.9677999999999999,1.0,0.9391999999999999
4 5 0.9728,1.0,0.9349999999999999
5 5 0.9732,1.0,0.931
6 5 0.9692000000000001,1.0,0.9304
7 5 0.9698,0.9998000000000001,0.9306000000000001
0 6 0.9766,1.0,0.9453999999999999
1 6 0.9782,1.0,0.9352
2 6 0.9772000000000001,0.9998000000000001,0.942
3 6 0.9774,1.0,0.9434000000000001
4 6 0.977,0.9998000000000001,0.9384
5 6 0.9730000000000001,0.9998000000000001,0.9311999999999999
6 6 0.9751999999999998,0.9996,0.9212
7 6 0.9663999999999999,0.9998000000000001,0.8928
0 7 0.9799999999999999,1.0,0.9398
1 7 0.9757999999999999,1.0,0.9427999999999999
2 7 0.9783999999999999,0.9998000000000001,0.9312000000000001
3 7 0.9709999999999999,0.9998000000000001,0.9138
4 7 0.9686,1.0,0.9134
5 7 0.9526,0.9996,0.8865999999999999
6 7 0.9458,0.9994,0.8378
7 7 0.9279999999999999,0.9986,0.7712
0 8 0.969,1.0,0.9214
1 8 0.962,0.9996,0.9022
2 8 0.9562000000000002,0.9994,0.8836
3 8 0.9411999999999999,0.9987999999999999,0.8385999999999999
4 8 0.9241999999999999,0.9992000000000001,0.7762
5 8 0.898,0.9958,0.7302000000000002
6 8 0.8676,0.9926,0.5282
7 8 0.8071999999999999,0.9318,0.2238
0 9 0.9251999999999999,0.9994,0.8286
1 9 0.9088,0.9997999999999999,0.7642
2 9 0.8964000000000001,0.9950000000000001,0.6921999999999999
3 9 0.8624,0.982,0.4544
4 9 0.8214,0.9422,0.2794
5 9 0.7772,0.8938,0.2218
6 9 0.6265999999999999,0.6374000000000001,0.0496
7 9 0.5214000000000001,0.46399999999999997,0.028200000000000003
};
\nextgroupplot[yticklabels={}, xticklabels={}]
\addplot[
matrix plot,
mesh/cols=8,
mesh/color input=explicit,
]
table[meta=rgb] {
x y rgb
0 0 0.8240000000000001,0.9612,0.7182000000000001
1 0 0.9612,1.0,0.9109999999999999
2 0 0.885,0.8692,0.5812
3 0 0.9658,0.9578,0.6458
4 0 0.8173999999999999,0.7300000000000001,0.3984
5 0 0.8737999999999999,0.8220000000000001,0.20540000000000003
6 0 0.8802,0.8613999999999999,0.25800000000000006
7 0 0.611,0.42779999999999996,0.11580000000000001
0 1 0.9765999999999998,0.9998000000000001,0.9398
1 1 0.9762000000000001,1.0,0.9263999999999999
2 1 0.9758000000000001,0.9998000000000001,0.9216
3 1 0.9712,0.9998000000000001,0.9178000000000001
4 1 0.9038,0.8800000000000001,0.5862
5 1 0.9693999999999999,0.998,0.8979999999999999
6 1 0.9673999999999999,0.9880000000000001,0.7292
7 1 0.8889999999999999,0.9362,0.081
0 2 0.9806000000000001,1.0,0.9342
1 2 0.9757999999999999,0.9998000000000001,0.9266000000000002
2 2 0.9802,1.0,0.9464
3 2 0.9818000000000001,1.0,0.9116
4 2 0.9785999999999999,0.9889999999999999,0.7786
5 2 0.9743999999999999,0.9872,0.6971999999999999
6 2 0.9714,0.9978,0.73
7 2 0.8927999999999999,0.8692,0.4958
0 3 0.9784,0.993,0.9212
1 3 0.9808,0.9996,0.9342
2 3 0.9837999999999999,0.978,0.852
3 3 0.9753999999999999,0.9972,0.7602
4 3 0.9785999999999999,0.999,0.9248
5 3 0.9776,0.9987999999999999,0.9278000000000001
6 3 0.97,0.9987999999999999,0.8969999999999999
7 3 0.966,0.9976,0.8386000000000001
0 4 0.9850000000000001,0.9898,0.8851999999999999
1 4 0.9822,0.999,0.9228
2 4 0.9818,0.9982,0.9202
3 4 0.9760000000000002,0.9966000000000002,0.8958
4 4 0.9708,0.9890000000000001,0.9016
5 4 0.969,0.984,0.8204
6 4 0.966,0.9965999999999999,0.8328
7 4 0.961,0.9617999999999999,0.5734
0 5 0.9809999999999999,0.9996,0.9084000000000001
1 5 0.9705999999999999,1.0,0.9234
2 5 0.9751999999999998,0.9970000000000001,0.9304
3 5 0.9687999999999999,0.9978,0.8722
4 5 0.9702,0.9934,0.7954
5 5 0.968,0.9974000000000001,0.7798
6 5 0.9702000000000002,0.9884000000000001,0.7314
7 5 0.969,0.9968,0.3714
0 6 0.9733999999999998,0.9955999999999999,0.913
1 6 0.9720000000000001,0.9843999999999999,0.771
2 6 0.9746,0.998,0.8373999999999999
3 6 0.9728,0.9992000000000001,0.9196
4 6 0.9728,0.9986,0.8338000000000001
5 6 0.9714,0.9998000000000001,0.865
6 6 0.9704,0.999,0.8130000000000001
7 6 0.9648,0.9986,0.5316
0 7 0.9774,0.9917999999999999,0.5856000000000001
1 7 0.9756,0.9982,0.6118
2 7 0.977,0.9987999999999999,0.7385999999999999
3 7 0.9732,0.9997999999999999,0.841
4 7 0.9698,0.9991999999999999,0.8475999999999999
5 7 0.9632,0.9998000000000001,0.5997999999999999
6 7 0.9318,1.0,0.12160000000000001
7 7 0.8785999999999999,1.0,0.024800000000000003
0 8 0.9715999999999999,0.9987999999999999,0.8113999999999999
1 8 0.9709999999999999,0.9994,0.8438000000000001
2 8 0.9558,0.9996,0.5222
3 8 0.9461999999999999,0.9998000000000001,0.2486
4 8 0.9256,0.9998000000000001,0.0522
5 8 0.8622,0.9997999999999999,0.0268
6 8 0.6013999999999999,0.5888,0.016800000000000002
7 8 0.5811999999999999,0.4768,0.0166
0 9 0.9269999999999999,1.0,0.0352
1 9 0.9018,1.0,0.0294
2 9 0.8358000000000001,1.0,0.0246
3 9 0.5894,0.4992000000000001,0.015199999999999997
4 9 0.5871999999999999,0.47300000000000003,0.016800000000000002
5 9 0.566,0.479,0.0158
6 9 0.5448,0.4782,0.02
7 9 0.5188,0.4788,0.0158
};
\nextgroupplot[yticklabels={}, xticklabels={}]
\addplot[
matrix plot,
mesh/cols=8,
mesh/color input=explicit,
]
table[meta=rgb] {
x y rgb
0 0 0.76,0.6314,0.17139999999999997
1 0 0.7071999999999999,0.5668,0.1416
2 0 0.7784000000000001,0.6852000000000001,0.1718
3 0 0.6869999999999999,0.4648,0.11599999999999999
4 0 0.5918,0.43500000000000005,0.1078
5 0 0.5988,0.39180000000000004,0.0964
6 0 0.5912,0.4048,0.0774
7 0 0.6056,0.4786,0.0754
0 1 0.8224,0.7186,0.209
1 1 0.8078,0.7163999999999999,0.16699999999999998
2 1 0.6701999999999999,0.5056,0.1386
3 1 0.7849999999999999,0.6842,0.123
4 1 0.688,0.5142,0.091
5 1 0.6883999999999999,0.5198,0.072
6 1 0.7388,0.5138,0.0688
7 1 0.6362,0.4264,0.0596
0 2 0.943,0.9550000000000001,0.2234
1 2 0.8210000000000001,0.7487999999999999,0.11539999999999999
2 2 0.7949999999999999,0.7354,0.082
3 2 0.6928,0.5332,0.07020000000000001
4 2 0.698,0.609,0.048
5 2 0.6577999999999999,0.5314,0.0422
6 2 0.8778,0.8560000000000001,0.21259999999999998
7 2 0.925,0.9814,0.14100000000000001
0 3 0.9536,0.9869999999999999,0.121
1 3 0.9578,0.99,0.1146
2 3 0.9652,0.9872,0.13699999999999998
3 3 0.9296,0.9423999999999999,0.15119999999999997
4 3 0.9538,0.9911999999999999,0.23220000000000002
5 3 0.9558,0.9853999999999999,0.1948
6 3 0.9725999999999999,0.9982000000000001,0.897
7 3 0.9642,0.9946000000000002,0.7844
0 4 0.9714,0.9934,0.3292
1 4 0.9677999999999999,0.9948,0.5404
2 4 0.9718,0.9965999999999999,0.5426
3 4 0.9753999999999999,0.9946000000000002,0.9022
4 4 0.9795999999999999,0.9994,0.9099999999999999
5 4 0.9724,0.9932000000000001,0.9022
6 4 0.9642,0.9792,0.8602000000000001
7 4 0.9692000000000001,0.9730000000000001,0.3564
0 5 0.9763999999999999,0.9978,0.9082000000000001
1 5 0.9736,0.9972,0.9192
2 5 0.969,0.9955999999999999,0.8939999999999999
3 5 0.9673999999999999,0.9921999999999999,0.7512000000000001
4 5 0.9602,0.9698,0.6698000000000001
5 5 0.9649999999999999,0.9968,0.5332000000000001
6 5 0.9564,0.9882,0.4152
7 5 0.8862,0.9896,0.042
0 6 0.974,0.9474,0.9347999999999999
1 6 0.9652,0.9987999999999999,0.8478
2 6 0.9805999999999999,1.0,0.9266000000000002
3 6 0.9785999999999999,1.0,0.9362
4 6 0.9789999999999999,1.0,0.9023999999999999
5 6 0.9742000000000001,1.0,0.372
6 6 0.9621999999999999,1.0,0.5728
7 6 0.9436,0.9965999999999999,0.1392
0 7 0.9790000000000001,1.0,0.9388
1 7 0.985,1.0,0.9364000000000001
2 7 0.9823999999999999,1.0,0.8924
3 7 0.9767999999999999,1.0,0.7070000000000001
4 7 0.969,0.9858,0.40919999999999995
5 7 0.9481999999999999,1.0,0.5740000000000001
6 7 0.9228,0.9904,0.46399999999999997
7 7 0.7906000000000001,0.9314,0.3274
0 8 0.9732,1.0,0.6586000000000001
1 8 0.9663999999999999,1.0,0.5572
2 8 0.9571999999999999,0.9960000000000001,0.5444
3 8 0.9517999999999999,0.826,0.317
4 8 0.9342,0.7986,0.44279999999999997
5 8 0.876,0.9490000000000001,0.168
6 8 0.728,0.9638,0.0608
7 8 0.5824,0.7619999999999999,0.0268
0 9 0.9422,0.8950000000000001,0.45499999999999996
1 9 0.9226000000000001,0.7752000000000001,0.28500000000000003
2 9 0.8716000000000002,0.966,0.2866
3 9 0.8448,0.9396000000000001,0.2966
4 9 0.8160000000000001,0.7125999999999999,0.0398
5 9 0.6642,0.5684000000000001,0.027200000000000002
6 9 0.6024,0.4760000000000001,0.0304
7 9 0.543,0.4800000000000001,0.0258
};
\nextgroupplot[yticklabels={}, xticklabels={}]
\addplot[
matrix plot,
mesh/cols=8,
mesh/color input=explicit,
]
table[meta=rgb] {
x y rgb
0 0 0.0356,0.9301999999999999,0.0256
1 0 0.033600000000000005,0.9057999999999999,0.027200000000000002
2 0 0.0284,0.882,0.0244
3 0 0.03899999999999999,0.6138,0.027600000000000003
4 0 0.1958,0.5733999999999999,0.0346
5 0 0.05600000000000001,0.43879999999999997,0.0262
6 0 0.038,0.6146,0.028200000000000003
7 0 0.0442,0.23420000000000002,0.0356
0 1 0.0712,0.9696,0.0262
1 1 0.0372,0.7702,0.026999999999999996
2 1 0.041800000000000004,0.5494,0.026000000000000002
3 1 0.0698,0.5806,0.033999999999999996
4 1 0.1296,0.40619999999999995,0.0358
5 1 0.18739999999999998,0.5728,0.027200000000000002
6 1 0.1056,0.3868,0.024200000000000003
7 1 0.29100000000000004,0.25880000000000003,0.07560000000000001
0 2 0.16240000000000002,0.6965999999999999,0.0258
1 2 0.31379999999999997,0.6284,0.029199999999999997
2 2 0.264,0.33020000000000005,0.03
3 2 0.24620000000000003,0.30460000000000004,0.0236
4 2 0.27240000000000003,0.4034,0.057600000000000005
5 2 0.5316000000000001,0.418,0.06720000000000001
6 2 0.493,0.43499999999999994,0.034600000000000006
7 2 0.5292000000000001,0.2876,0.047799999999999995
0 3 0.5134000000000001,0.36339999999999995,0.06360000000000002
1 3 0.39039999999999997,0.40259999999999996,0.046
2 3 0.5608000000000001,0.4492,0.034999999999999996
3 3 0.5601999999999999,0.6152,0.061
4 3 0.5488000000000001,0.5271999999999999,0.0704
5 3 0.7552000000000001,0.9394,0.24359999999999998
6 3 0.5486,0.5362,0.2028
7 3 0.8762000000000001,0.8486,0.42779999999999996
0 4 0.6782,0.7827999999999999,0.1532
1 4 0.9244,0.993,0.6821999999999999
2 4 0.967,1.0,0.914
3 4 0.977,0.9998000000000001,0.9241999999999999
4 4 0.9719999999999999,1.0,0.8751999999999999
5 4 0.9712,1.0,0.9194000000000001
6 4 0.9782,0.9998000000000001,0.9034000000000001
7 4 0.968,0.9994,0.8587999999999999
0 5 0.9702,1.0,0.9198000000000001
1 5 0.9597999999999999,0.9998000000000001,0.8808
2 5 0.9404,0.9958,0.8648
3 5 0.9006000000000001,0.8775999999999999,0.6702
4 5 0.9252,0.977,0.6946
5 5 0.9324,0.9734,0.725
6 5 0.9408000000000001,0.9949999999999999,0.7232000000000001
7 5 0.9364000000000001,0.9970000000000001,0.657
0 6 0.9343999999999999,1.0,0.8181999999999998
1 6 0.8907999999999999,0.9963999999999998,0.5892
2 6 0.8703999999999998,0.9248,0.6156
3 6 0.8710000000000001,0.8266,0.5896
4 6 0.953,0.9308,0.7527999999999999
5 6 0.9368000000000001,0.9702,0.42619999999999997
6 6 0.9391999999999999,0.9858,0.5252
7 6 0.9254,0.9957999999999998,0.4795999999999999
0 7 0.8906000000000001,0.9952,0.8038000000000001
1 7 0.9378,0.9982000000000001,0.6256
2 7 0.8622,0.991,0.49239999999999995
3 7 0.9246000000000001,0.9841999999999999,0.6568
4 7 0.9088,0.9778,0.3224
5 7 0.9354000000000001,0.9916,0.41180000000000005
6 7 0.8433999999999999,0.6772,0.056999999999999995
7 7 0.7512000000000001,0.5062,0.11100000000000002
0 8 0.9188000000000001,0.9944000000000001,0.6796
1 8 0.8782,0.9944,0.2644
2 8 0.7326,0.978,0.17020000000000002
3 8 0.8694000000000001,0.8198000000000001,0.1378
4 8 0.7948000000000001,0.5022,0.06400000000000002
5 8 0.7474000000000001,0.48460000000000003,0.04859999999999999
6 8 0.6004,0.4885999999999999,0.0442
7 8 0.583,0.4880000000000001,0.0372
0 9 0.8573999999999999,0.7174,0.041400000000000006
1 9 0.7522,0.5456000000000001,0.0518
2 9 0.6304000000000001,0.517,0.0366
3 9 0.6364,0.5016,0.0166
4 9 0.5432,0.49639999999999995,0.0506
5 9 0.5309999999999999,0.5112,0.0374
6 9 0.5484,0.4082,0.079
7 9 0.5576,0.1424,0.0908
};
\nextgroupplot[yticklabels={}, xticklabels={}]
\addplot[
matrix plot,
mesh/cols=8,
mesh/color input=explicit,
]
table[meta=rgb] {
x y rgb
0 0 0.9742,1.0,0.9198000000000001
1 0 0.9739999999999999,1.0,0.9267999999999998
2 0 0.9753999999999999,1.0,0.9367999999999999
3 0 0.9788,1.0,0.9199999999999999
4 0 0.9756,1.0,0.9277999999999998
5 0 0.9747999999999999,0.9827999999999999,0.6418000000000001
6 0 0.9728,0.9878,0.9108
7 0 0.9607999999999999,0.9468,0.7847999999999999
0 1 0.9733999999999998,1.0,0.9339999999999999
1 1 0.9743999999999999,1.0,0.8958
2 1 0.9757999999999999,0.9827999999999999,0.7534000000000001
3 1 0.9802,0.9952,0.781
4 1 0.9738,1.0,0.9284000000000001
5 1 0.9742000000000001,1.0,0.9318
6 1 0.9751999999999998,0.9960000000000001,0.7526
7 1 0.9725999999999999,0.999,0.9026
0 2 0.9734,1.0,0.9212
1 2 0.9783999999999999,1.0,0.9234000000000002
2 2 0.977,1.0,0.9405999999999999
3 2 0.9799999999999999,1.0,0.9478
4 2 0.9745999999999999,1.0,0.9308
5 2 0.9693999999999999,0.9946000000000002,0.9092
6 2 0.9719999999999999,0.999,0.922
7 2 0.9652,0.9936,0.8942
0 3 0.9803999999999998,1.0,0.9460000000000001
1 3 0.9813999999999998,1.0,0.9396000000000001
2 3 0.9823999999999999,1.0,0.9332
3 3 0.9802,0.9996,0.9193999999999999
4 3 0.9648,0.9970000000000001,0.8822000000000001
5 3 0.9683999999999999,0.9984,0.8926000000000001
6 3 0.9062000000000001,0.9882000000000002,0.8493999999999999
7 3 0.9478,0.9972,0.9074
0 4 0.9743999999999999,0.9958,0.8953999999999999
1 4 0.9736,0.9998000000000001,0.9146000000000001
2 4 0.9516,0.9936,0.7928
3 4 0.915,0.9843999999999999,0.7576
4 4 0.907,0.9224,0.664
5 4 0.8495999999999999,0.8998000000000002,0.7134
6 4 0.8486,0.9186,0.769
7 4 0.8368,0.9212,0.7482
0 5 0.8724000000000001,0.9974000000000001,0.6754
1 5 0.8865999999999999,0.9776,0.7104
2 5 0.865,0.9708,0.6986
3 5 0.8873999999999999,0.9501999999999999,0.6976
4 5 0.8981999999999999,0.9192,0.7538
5 5 0.9098,0.9882,0.7864
6 5 0.9182,0.9914,0.8054
7 5 0.9578,0.9746,0.8596
0 6 0.8876,0.9496,0.7358
1 6 0.9193999999999999,0.9892,0.7702
2 6 0.9186,0.9994,0.8151999999999999
3 6 0.9328,0.9663999999999999,0.8305999999999999
4 6 0.9426,0.9952,0.8582000000000001
5 6 0.9730000000000001,0.9898,0.8560000000000001
6 6 0.9703999999999999,0.9948,0.7896000000000001
7 6 0.9613999999999999,0.9666,0.4088
0 7 0.951,0.999,0.8468
1 7 0.9523999999999999,0.998,0.8699999999999999
2 7 0.9739999999999999,0.984,0.8698
3 7 0.9735999999999999,0.9908000000000001,0.8366
4 7 0.9662000000000001,0.9655999999999999,0.6352
5 7 0.9436,0.9494,0.1492
6 7 0.8779999999999999,0.829,0.0418
7 7 0.6418,0.43920000000000003,0.023600000000000003
0 8 0.9714,0.9876000000000001,0.7874000000000001
1 8 0.966,0.9814,0.7006
2 8 0.954,0.9621999999999999,0.4114000000000001
3 8 0.9022,0.9432,0.051000000000000004
4 8 0.7448,0.633,0.0294
5 8 0.6344,0.4444,0.0206
6 8 0.5780000000000001,0.45199999999999996,0.033600000000000005
7 8 0.584,0.40099999999999997,0.0162
0 9 0.899,0.9349999999999999,0.0682
1 9 0.7811999999999999,0.7315999999999999,0.0288
2 9 0.6198,0.46299999999999997,0.022600000000000002
3 9 0.591,0.44939999999999997,0.021200000000000004
4 9 0.5401999999999999,0.46740000000000004,0.0202
5 9 0.5597999999999999,0.4138,0.028000000000000004
6 9 0.5618000000000001,0.39259999999999995,0.0238
7 9 0.48360000000000003,0.46020000000000005,0.027200000000000002
};
\nextgroupplot[yticklabels={}, xticklabels={}]
\addplot[
matrix plot,
mesh/cols=8,
mesh/color input=explicit,
]
table[meta=rgb] {
x y rgb
0 0 0.9596,1.0,0.7236
1 0 0.9666,0.9966000000000002,0.7846
2 0 0.9506,0.9747999999999999,0.4772
3 0 0.9612,0.9996,0.7308
4 0 0.9632,0.9802,0.5428
5 0 0.9570000000000001,0.9126,0.44960000000000006
6 0 0.8709999999999999,0.73,0.2956
7 0 0.8314,0.4928,0.1312
0 1 0.9597999999999999,1.0,0.8051999999999999
1 1 0.9671999999999998,1.0,0.695
2 1 0.9642,0.9978,0.3316
3 1 0.9494,0.9196,0.4418000000000001
4 1 0.9372,0.8788,0.4
5 1 0.9448000000000001,0.5258,0.39239999999999997
6 1 0.9282,0.47459999999999997,0.1788
7 1 0.9256,0.4298,0.0528
0 2 0.9598000000000001,0.998,0.4502
1 2 0.9623999999999999,0.93,0.5958
2 2 0.9678000000000001,0.882,0.7406
3 2 0.9228000000000002,0.7064,0.4114
4 2 0.9056,0.7611999999999999,0.33680000000000004
5 2 0.8785999999999999,0.41979999999999995,0.0732
6 2 0.8493999999999999,0.5267999999999999,0.1732
7 2 0.8662000000000001,0.46719999999999995,0.1442
0 3 0.9457999999999999,0.7074,0.48840000000000006
1 3 0.9256,0.6734,0.5094
2 3 0.8919999999999998,0.6275999999999999,0.49379999999999996
3 3 0.8615999999999999,0.8173999999999999,0.4814
4 3 0.8842000000000001,0.8422000000000001,0.6024
5 3 0.853,0.9457999999999999,0.40840000000000004
6 3 0.8446,0.8952,0.42560000000000003
7 3 0.8528,0.8615999999999999,0.1926
0 4 0.8512000000000001,0.7036,0.4514
1 4 0.8896000000000001,0.7352000000000001,0.5912
2 4 0.8798,0.8916000000000001,0.7024000000000001
3 4 0.9058000000000002,0.9570000000000001,0.5726
4 4 0.8474,0.874,0.6500000000000001
5 4 0.7100000000000001,0.9863999999999999,0.5436
6 4 0.6578000000000002,0.8766,0.4295999999999999
7 4 0.6851999999999999,0.7559999999999999,0.3338
0 5 0.873,0.9269999999999999,0.5349999999999999
1 5 0.8606000000000001,0.9306000000000001,0.6208
2 5 0.7904,0.9753999999999999,0.6876
3 5 0.6694,0.7978,0.29019999999999996
4 5 0.744,0.795,0.2576
5 5 0.6407999999999999,0.6437999999999999,0.3172
6 5 0.6544000000000001,0.6746,0.2024
7 5 0.7376000000000001,0.6638,0.11420000000000001
0 6 0.9186,0.7949999999999999,0.4298
1 6 0.9116,0.5631999999999999,0.2372
2 6 0.9016,0.6178000000000001,0.1472
3 6 0.9044000000000001,0.6976,0.15760000000000002
4 6 0.9044000000000001,0.6106,0.14740000000000003
5 6 0.8718,0.6048,0.23200000000000004
6 6 0.6994,0.546,0.25420000000000004
7 6 0.7182000000000001,0.7444000000000001,0.17160000000000003
0 7 0.8380000000000001,0.5513999999999999,0.2472
1 7 0.8911999999999999,0.7802,0.20420000000000002
2 7 0.6548,0.6126,0.38079999999999997
3 7 0.8444,0.543,0.26159999999999994
4 7 0.8058,0.6980000000000001,0.21760000000000002
5 7 0.6504,0.7140000000000001,0.11200000000000002
6 7 0.4434,0.6214000000000001,0.0508
7 7 0.48360000000000003,0.5734,0.022800000000000004
0 8 0.8904,0.558,0.47000000000000003
1 8 0.8019999999999999,0.5690000000000001,0.353
2 8 0.7804,0.6517999999999999,0.15580000000000002
3 8 0.6531999999999999,0.744,0.10900000000000001
4 8 0.42960000000000004,0.8446,0.025
5 8 0.4358000000000001,0.5803999999999999,0.0206
6 8 0.5018,0.5589999999999999,0.0144
7 8 0.5218,0.546,0.019200000000000002
0 9 0.7224,0.8398,0.1462
1 9 0.5016,0.5704,0.043000000000000003
2 9 0.4546,0.571,0.018
3 9 0.4266,0.5646000000000001,0.0188
4 9 0.454,0.5558,0.017
5 9 0.45920000000000005,0.5492000000000001,0.016
6 9 0.4596,0.5512,0.0176
7 9 0.4574,0.5509999999999999,0.025
};

\end{groupplot}
\end{tikzpicture}

%% file: plots/cases/first_token_output/var_layers.tex

\begin{tikzpicture}

\begin{groupplot}[
group style={
    group size=6 by 2,
    horizontal sep=0.1cm,
    vertical sep=0.85cm,
    xlabels at=edge bottom,
    ylabels at=edge left,
},
height=4.3cm, 
width=1.55*\columnwidth/6,
enlargelimits=false,
xtick=data, ytick=data, xticklabels={1,2,4,8,16,32,64}, xticklabel style={font=\small, rotate=90}, yticklabel style={font=\small}, title style={at={(0.5,0.9)}, font=\small},]

\nextgroupplot[title=BERT (76.8\%), 
yticklabels={0.3,0.09,0.027,0.008,0.002,7.3e-4,2.2e-4,6.6e-5,2.0e-5,5.9e-6}]
\addplot[
matrix plot,
mesh/cols=7,
mesh/color input=explicit,
]
table[meta=rgb] {
x y rgb
0 0 0.2096875011920929,0.2151249974966049,0.22499999403953552
1 0 0.20593750476837158,0.21368750035762787,0.2228125035762787
2 0 0.20781250298023224,0.21243749856948851,0.2303124964237213
3 0 0.20374999940395355,0.21031249761581422,0.22156250476837158
4 0 0.20812499523162842,0.21287499666213988,0.2253125011920929
5 0 0.20499999821186066,0.21212500035762788,0.21968750655651093
6 0 0.2018750011920929,0.2122500032186508,0.2240625023841858
0 1 0.20906250178813934,0.21843750178813934,0.23125000298023224
1 1 0.21375000476837158,0.22043750286102295,0.22812500596046448
2 1 0.21437500417232513,0.2188749998807907,0.23156249523162842
3 1 0.21187500655651093,0.21737499833106994,0.22843749821186066
4 1 0.2134374976158142,0.21806249916553497,0.2278124988079071
5 1 0.21187500655651093,0.21843750178813934,0.2356249988079071
6 1 0.21531249582767487,0.2201875001192093,0.23125000298023224
0 2 0.21156249940395355,0.21781249940395356,0.2318750023841858
1 2 0.21062499284744263,0.21806249916553497,0.22812500596046448
2 2 0.21250000596046448,0.21924999952316285,0.23218749463558197
3 2 0.21187500655651093,0.21793749928474426,0.23281249403953552
4 2 0.2121874988079071,0.21943749785423278,0.2306250035762787
5 2 0.21250000596046448,0.21962500214576722,0.23250000178813934
6 2 0.20999999344348907,0.21681250035762786,0.23375000059604645
0 3 0.21406249701976776,0.22062500119209288,0.23718750476837158
1 3 0.21031250059604645,0.21781249940395356,0.2331250011920929
2 3 0.21593749523162842,0.21943749785423278,0.2318750023841858
3 3 0.2109375,0.21512500047683716,0.22750000655651093
4 3 0.21250000596046448,0.2176250010728836,0.23499999940395355
5 3 0.21187500655651093,0.21674999892711638,0.22875000536441803
6 3 0.2121874988079071,0.21825000047683715,0.23281249403953552
0 4 0.2240625023841858,0.22899999916553498,0.2356249988079071
1 4 0.21375000476837158,0.2320000022649765,0.3021875023841858
2 4 0.21312500536441803,0.21812500059604645,0.23093749582767487
3 4 0.21281249821186066,0.21631250381469727,0.22750000655651093
4 4 0.21156249940395355,0.21687499582767486,0.2290624976158142
5 4 0.2134374976158142,0.22087500095367432,0.23125000298023224
6 4 0.21062499284744263,0.2176250010728836,0.22937500476837158
0 5 0.22312499582767487,0.22981249690055847,0.23937499523162842
1 5 0.22750000655651093,0.43093750178813933,0.9168750047683716
2 5 0.9206249713897705,0.9373749852180481,0.9631249904632568
3 5 0.21437500417232513,0.2229374974966049,0.2628124952316284
4 5 0.2146874964237213,0.225,0.23593750596046448
5 5 0.21125000715255737,0.2173124998807907,0.2318750023841858
6 5 0.2150000035762787,0.21949999928474426,0.23406249284744263
0 6 0.22093750536441803,0.22818749845027925,0.23937499523162842
1 6 0.23000000417232513,0.36356250047683714,0.551562488079071
2 6 0.9181249737739563,0.9537500143051147,0.9793750047683716
3 6 0.3540624976158142,0.8531875014305115,0.984375
4 6 0.2224999964237213,0.6319999933242798,0.9346874952316284
5 6 0.22812500596046448,0.5976250052452088,0.9087499976158142
6 6 0.21125000715255737,0.21956250369548796,0.2303124964237213
0 7 0.21187500655651093,0.2173750013113022,0.22937500476837158
1 7 0.21656249463558197,0.22362499833106994,0.234375
2 7 0.2175000011920929,0.22481250166893005,0.23874999582767487
3 7 0.21531249582767487,0.2262499988079071,0.2515625059604645
4 7 0.21906250715255737,0.22443750202655793,0.23937499523162842
5 7 0.22312499582767487,0.22543749809265137,0.23406249284744263
6 7 0.21656249463558197,0.22612500190734863,0.2384375035762787
0 8 0.21250000596046448,0.21668750047683716,0.22499999403953552
1 8 0.21312500536441803,0.22200000286102295,0.23093749582767487
2 8 0.2121874988079071,0.2176874965429306,0.2356249988079071
3 8 0.2096875011920929,0.2165624976158142,0.2290624976158142
4 8 0.2109375,0.2163750022649765,0.2265625
5 8 0.2109375,0.22018749713897706,0.2290624976158142
6 8 0.20937499403953552,0.2158125013113022,0.22875000536441803
0 9 0.21125000715255737,0.22062500119209288,0.23281249403953552
1 9 0.21187500655651093,0.22212500274181365,0.23125000298023224
2 9 0.2096875011920929,0.2146249979734421,0.23000000417232513
3 9 0.21250000596046448,0.21837500035762786,0.22843749821186066
4 9 0.2150000035762787,0.21812500059604645,0.22843749821186066
5 9 0.21718749403953552,0.22062500119209288,0.23656250536441803
6 9 0.21281249821186066,0.21800000071525574,0.23218749463558197
};

\nextgroupplot[title=MTE (51.2\%), 
yticklabels={}]
\addplot[
matrix plot,
mesh/cols=7,
mesh/color input=explicit,
]
table[meta=rgb] {
x y rgb
0 0 0.20624999701976776,0.2150000035762787,0.2265625
1 0 0.20812499523162842,0.21712498366832733,0.22937500476837158
2 0 0.20531250536441803,0.21306249499320984,0.2265625
3 0 0.2056249976158142,0.21375000476837158,0.23000000417232513
4 0 0.062187500298023224,0.08118750154972076,0.15937499701976776
5 0 0.06062500178813934,0.0858749970793724,0.18187500536441803
6 0 0.05406250059604645,0.05949999764561653,0.07093749940395355
0 1 0.2071875035762787,0.21512499451637268,0.22593750059604645
1 1 0.20874999463558197,0.21287500858306885,0.22468750178813934
2 1 0.21718749403953552,0.2330625057220459,0.27250000834465027
3 1 0.22374999523162842,0.2303124964237213,0.2540625035762787
4 1 0.21156249940395355,0.21975000202655792,0.23250000178813934
5 1 0.11500000208616257,0.1549375057220459,0.20781250298023224
6 1 0.05718750134110451,0.06168749928474426,0.07593750208616257
0 2 0.22437499463558197,0.29475000500679016,0.5665624737739563
1 2 0.21718749403953552,0.31856250762939453,0.48374998569488525
2 2 0.2434374988079071,0.31181249022483826,0.5840625166893005
3 2 0.21968750655651093,0.32587501406669617,0.53125
4 2 0.2265625,0.24950000643730164,0.3128125071525574
5 2 0.21812500059604645,0.25200000405311584,0.35843750834465027
6 2 0.19687500596046448,0.20537500083446503,0.21906250715255737
0 3 0.2368749976158142,0.4922500550746918,0.8678125143051147
1 3 0.23281249403953552,0.5815625190734863,0.8403124809265137
2 3 0.23343749344348907,0.3776875138282776,0.7612500190734863
3 3 0.23593750596046448,0.406562477350235,0.6087499856948853
4 3 0.21843749284744263,0.37550002336502075,0.6365625262260437
5 3 0.21937499940395355,0.27149999141693115,0.4490624964237213
6 3 0.22218750417232513,0.2396874874830246,0.27125000953674316
0 4 0.22687500715255737,0.4274374842643738,0.9018750190734863
1 4 0.23343749344348907,0.5506249666213989,0.8443750143051147
2 4 0.23499999940395355,0.4243749678134918,0.8596875071525574
3 4 0.23156249523162842,0.6306250095367432,0.9046875238418579
4 4 0.25687500834465027,0.39918750524520874,0.6065624952316284
5 4 0.21843749284744263,0.27950000762939453,0.43562498688697815
6 4 0.21156249940395355,0.22075000405311584,0.23593750596046448
0 5 0.21781249344348907,0.22731249034404755,0.24062499403953552
1 5 0.24156250059604645,0.37406250834465027,0.46937501430511475
2 5 0.2224999964237213,0.29225000739097595,0.3578124940395355
3 5 0.28718748688697815,0.36262500286102295,0.49906250834465027
4 5 0.23937499523162842,0.28831249475479126,0.3609375059604645
5 5 0.22062499821186066,0.2397499978542328,0.26218751072883606
6 5 0.21312500536441803,0.21825000643730164,0.22875000536441803
0 6 0.22062499821186066,0.22443750500679016,0.23937499523162842
1 6 0.22093750536441803,0.22456249594688416,0.23593750596046448
2 6 0.21875,0.22356250882148743,0.23531250655651093
3 6 0.21781249344348907,0.2225625067949295,0.23406249284744263
4 6 0.21687500178813934,0.21975000202655792,0.23656250536441803
5 6 0.2150000035762787,0.2200624942779541,0.2318750023841858
6 6 0.2134374976158142,0.2173124998807907,0.2303124964237213
0 7 0.2121874988079071,0.218562513589859,0.23093749582767487
1 7 0.2162500023841858,0.22206249833106995,0.234375
2 7 0.2150000035762787,0.22349998354911804,0.23906250298023224
3 7 0.21312500536441803,0.22343750298023224,0.23874999582767487
4 7 0.22062499821186066,0.22556249797344208,0.2356249988079071
5 7 0.21843749284744263,0.22350001335144043,0.24250000715255737
6 7 0.21937499940395355,0.22356250882148743,0.23999999463558197
0 8 0.21156249940395355,0.21718749403953552,0.2278124988079071
1 8 0.2146874964237213,0.21843750774860382,0.2303124964237213
2 8 0.21250000596046448,0.2199375182390213,0.2306250035762787
3 8 0.2134374976158142,0.22218751907348633,0.23749999701976776
4 8 0.2162500023841858,0.2228125035762787,0.2331250011920929
5 8 0.2199999988079071,0.22374999523162842,0.2306250035762787
6 8 0.21531249582767487,0.2239374816417694,0.2331250011920929
0 9 0.2150000035762787,0.21843750774860382,0.22968749701976776
1 9 0.21375000476837158,0.218562513589859,0.22875000536441803
2 9 0.2150000035762787,0.22025001049041748,0.24281249940395355
3 9 0.2121874988079071,0.21875,0.2318750023841858
4 9 0.2146874964237213,0.21825000643730164,0.23125000298023224
5 9 0.21562500298023224,0.21843750774860382,0.22937500476837158
6 9 0.2134374976158142,0.2201875001192093,0.23468750715255737
};

\nextgroupplot[title=NAP (90.8\%), 
yticklabels={}]
\addplot[
matrix plot,
mesh/cols=7,
mesh/color input=explicit,
]
table[meta=rgb] {
x y rgb
0 0 0.20374999940395355,0.21312499046325684,0.23250000178813934
1 0 0.21156249940395355,0.21712501347064972,0.2228125035762787
2 0 0.2084375023841858,0.21268749237060547,0.2228125035762787
3 0 0.2228125035762787,0.3628750145435333,0.5237500071525574
4 0 0.1875,0.4168124794960022,0.4884375035762787
5 0 0.0949999988079071,0.18287499248981476,0.4215624928474426
6 0 0.07000000029802322,0.0806874930858612,0.09593749791383743
0 1 0.21125000715255737,0.21568751335144043,0.22562499344348907
1 1 0.2068749964237213,0.21518750488758087,0.23499999940395355
2 1 0.21031250059604645,0.2965624928474426,0.43656250834465027
3 1 0.21562500298023224,0.4768125116825104,0.9024999737739563
4 1 0.8174999952316284,0.8715000152587891,0.9078124761581421
5 1 0.20624999701976776,0.5512499809265137,0.7884374856948853
6 1 0.19875000417232513,0.20475001633167267,0.21375000476837158
0 2 0.22187499701976776,0.25699999928474426,0.3865624964237213
1 2 0.22093750536441803,0.24431248009204865,0.33500000834465027
2 2 0.2240625023841858,0.2574999928474426,0.3371874988079071
3 2 0.2265625,0.4178125262260437,0.8118749856948853
4 2 0.22031250596046448,0.45887500047683716,0.8412500023841858
5 2 0.2224999964237213,0.5083125233650208,0.9649999737739563
6 2 0.2199999988079071,0.32100000977516174,0.6509374976158142
0 3 0.23093749582767487,0.35468751192092896,0.6018750071525574
1 3 0.2278124988079071,0.28968751430511475,0.4803124964237213
2 3 0.22624999284744263,0.2838749885559082,0.5021874904632568
3 3 0.22499999403953552,0.22687499225139618,0.23749999701976776
4 3 0.23375000059604645,0.26249998807907104,0.3400000035762787
5 3 0.21843749284744263,0.30399999022483826,0.5546875
6 3 0.22156250476837158,0.35737499594688416,0.4762499928474426
0 4 0.2565625011920929,0.3733749985694885,0.5687500238418579
1 4 0.2593750059604645,0.3414374887943268,0.40062499046325684
2 4 0.23406249284744263,0.37687501311302185,0.6137499809265137
3 4 0.28187501430511475,0.3932500183582306,0.5774999856948853
4 4 0.22156250476837158,0.2731874883174896,0.34843748807907104
5 4 0.2150000035762787,0.31325000524520874,0.6850000023841858
6 4 0.21812500059604645,0.26356250047683716,0.43937501311302185
0 5 0.8506249785423279,0.8606875538825989,0.8765624761581421
1 5 0.38062500953674316,0.710562527179718,0.8440625071525574
2 5 0.23000000417232513,0.6369999647140503,0.9318749904632568
3 5 0.24437500536441803,0.7038124799728394,0.932812511920929
4 5 0.22750000655651093,0.5593125224113464,0.9278125166893005
5 5 0.23156249523162842,0.40018749237060547,0.9118750095367432
6 5 0.21687500178813934,0.5175625085830688,0.9765625
0 6 0.8890625238418579,0.8939374685287476,0.8978124856948853
1 6 0.9228125214576721,0.9457499384880066,0.9637500047683716
2 6 0.9737499952316284,0.9798124432563782,0.9859374761581421
3 6 0.8778125047683716,0.9481874704360962,0.9837499856948853
4 6 0.8409374952316284,0.9379374384880066,0.9837499856948853
5 6 0.34687501192092896,0.7520624995231628,0.9746875166893005
6 6 0.23343749344348907,0.7164374589920044,0.9384375214576721
0 7 0.5843750238418579,0.7540000081062317,0.8534374833106995
1 7 0.770312488079071,0.8220000267028809,0.8650000095367432
2 7 0.8046875,0.8841250538825989,0.9278125166893005
3 7 0.7903125286102295,0.8571251034736633,0.9424999952316284
4 7 0.8040624856948853,0.8941249847412109,0.9384375214576721
5 7 0.7615625262260437,0.8491250276565552,0.9159374833106995
6 7 0.22843749821186066,0.581125020980835,0.8659374713897705
0 8 0.21937499940395355,0.22675001621246338,0.2356249988079071
1 8 0.21812500059604645,0.2240000218153,0.24500000476837158
2 8 0.22374999523162842,0.22868749499320984,0.2421875
3 8 0.2240625023841858,0.23274998366832733,0.25218749046325684
4 8 0.22218750417232513,0.23206250369548798,0.24843749403953552
5 8 0.22218750417232513,0.23006248474121094,0.24968749284744263
6 8 0.22093750536441803,0.22437497973442078,0.2396875023841858
0 9 0.21250000596046448,0.2173749953508377,0.23156249523162842
1 9 0.2150000035762787,0.22193749248981476,0.2384375035762787
2 9 0.21656249463558197,0.22206251323223114,0.23375000059604645
3 9 0.21031250059604645,0.21968750655651093,0.23906250298023224
4 9 0.21281249821186066,0.21768751740455627,0.2290624976158142
5 9 0.2134374976158142,0.2188125103712082,0.23000000417232513
6 9 0.21031250059604645,0.2176249921321869,0.2356249988079071
};

\nextgroupplot[title=NON (31.2\%), 
yticklabels={}]
\addplot[
matrix plot,
mesh/cols=7,
mesh/color input=explicit,
]
table[meta=rgb] {
x y rgb
0 0 0.20593750476837158,0.21250000596046448,0.22031250596046448
1 0 0.21125000715255737,0.21556250751018524,0.22718749940395355
2 0 0.20468750596046448,0.21150000393390656,0.22843749821186066
3 0 0.20250000059604645,0.20612499117851257,0.21250000596046448
4 0 0.17749999463558197,0.19212499260902405,0.203125
5 0 0.0871874988079071,0.1016250029206276,0.11124999821186066
6 0 0.06093750149011612,0.06575000286102295,0.07093749940395355
0 1 0.2146874964237213,0.21649999916553497,0.2253125011920929
1 1 0.20874999463558197,0.21518750488758087,0.2240625023841858
2 1 0.203125,0.22100000083446503,0.22843749821186066
3 1 0.21312500536441803,0.2163124978542328,0.22968749701976776
4 1 0.21156249940395355,0.21806249022483826,0.24062499403953552
5 1 0.2056249976158142,0.21793749928474426,0.22843749821186066
6 1 0.18906250596046448,0.20024999976158142,0.21187500655651093
0 2 0.2199999988079071,0.2304374873638153,0.23999999463558197
1 2 0.21812500059604645,0.22450001537799835,0.2368749976158142
2 2 0.2134374976158142,0.2226875126361847,0.23656250536441803
3 2 0.2162500023841858,0.25212499499320984,0.3921875059604645
4 2 0.22031250596046448,0.22618749737739563,0.2356249988079071
5 2 0.2212499976158142,0.23237499594688416,0.27125000953674316
6 2 0.21406249701976776,0.21824999153614044,0.23531250655651093
0 3 0.21843749284744263,0.22906248271465302,0.23656250536441803
1 3 0.21906250715255737,0.2278749942779541,0.23999999463558197
2 3 0.22093750536441803,0.2290000021457672,0.2462500035762787
3 3 0.21906250715255737,0.2292499989271164,0.24843749403953552
4 3 0.22343750298023224,0.23768749833106995,0.28437501192092896
5 3 0.22093750536441803,0.23225000500679016,0.2553125023841858
6 3 0.2212499976158142,0.22699999809265137,0.24562500417232513
0 4 0.2212499976158142,0.23474998772144318,0.24906249344348907
1 4 0.22374999523162842,0.2524375021457672,0.31187498569488525
2 4 0.22187499701976776,0.234562486410141,0.2681249976158142
3 4 0.21906250715255737,0.2250625193119049,0.24281249940395355
4 4 0.2253125011920929,0.23593750596046448,0.2578125
5 4 0.21937499940395355,0.2278749942779541,0.24531249701976776
6 4 0.22187499701976776,0.22737498581409454,0.2409375011920929
0 5 0.22343750298023224,0.22550001740455627,0.24500000476837158
1 5 0.2212499976158142,0.22737500071525574,0.2396875023841858
2 5 0.22062499821186066,0.23081250488758087,0.24062499403953552
3 5 0.22031250596046448,0.22687499225139618,0.24281249940395355
4 5 0.22031250596046448,0.22687499225139618,0.23906250298023224
5 5 0.2240625023841858,0.22843751311302185,0.23718750476837158
6 5 0.21312500536441803,0.21943750977516174,0.23375000059604645
0 6 0.22031250596046448,0.2254374921321869,0.24187499284744263
1 6 0.22187499701976776,0.22506248950958252,0.23906250298023224
2 6 0.21937499940395355,0.22474999725818634,0.24031250178813934
3 6 0.2240625023841858,0.22712500393390656,0.23593750596046448
4 6 0.22187499701976776,0.22756250202655792,0.23593750596046448
5 6 0.21937499940395355,0.2254374921321869,0.2356249988079071
6 6 0.21812500059604645,0.2251874953508377,0.2318750023841858
0 7 0.21656249463558197,0.22306248545646667,0.2381249964237213
1 7 0.2175000011920929,0.2251249998807907,0.23656250536441803
2 7 0.22468750178813934,0.2280624806880951,0.2384375035762787
3 7 0.22343750298023224,0.22943750023841858,0.23937499523162842
4 7 0.2212499976158142,0.22612500190734863,0.2384375035762787
5 7 0.21781249344348907,0.2250625193119049,0.24312500655651093
6 7 0.21937499940395355,0.2265625,0.23468750715255737
0 8 0.21281249821186066,0.2161875069141388,0.2278124988079071
1 8 0.21812500059604645,0.22450001537799835,0.234375
2 8 0.2175000011920929,0.22181248664855957,0.23781250417232513
3 8 0.22343750298023224,0.2276875078678131,0.24281249940395355
4 8 0.21968750655651093,0.22737500071525574,0.2537499964237213
5 8 0.2253125011920929,0.23674997687339783,0.30406248569488525
6 8 0.22437499463558197,0.22756250202655792,0.23718750476837158
0 9 0.21281249821186066,0.2199375182390213,0.22843749821186066
1 9 0.21156249940395355,0.21656250953674316,0.22750000655651093
2 9 0.2134374976158142,0.218812495470047,0.2290624976158142
3 9 0.21968750655651093,0.22606250643730164,0.2396875023841858
4 9 0.2175000011920929,0.22581250965595245,0.2421875
5 9 0.21937499940395355,0.22637498378753662,0.2396875023841858
6 9 0.21718749403953552,0.22293750941753387,0.23874999582767487
};

\nextgroupplot[title=sum (29.2\%), 
yticklabels={}]
\addplot[
matrix plot,
mesh/cols=7,
mesh/color input=explicit,
]
table[meta=rgb] {
x y rgb
0 0 0.20468750596046448,0.2122499942779541,0.22343750298023224
1 0 0.20874999463558197,0.2148750126361847,0.23156249523162842
2 0 0.20812499523162842,0.21637499332427979,0.22875000536441803
3 0 0.20593750476837158,0.21800000965595245,0.2459374964237213
4 0 0.20906250178813934,0.2134999930858612,0.23343749344348907
5 0 0.21062499284744263,0.21837499737739563,0.2265625
6 0 0.20468750596046448,0.2121249884366989,0.2240625023841858
0 1 0.20999999344348907,0.21718749403953552,0.2290624976158142
1 1 0.22031250596046448,0.22374999523162842,0.23281249403953552
2 1 0.21718749403953552,0.22493748366832733,0.23906250298023224
3 1 0.21437500417232513,0.2224999964237213,0.23656250536441803
4 1 0.21406249701976776,0.21837499737739563,0.23000000417232513
5 1 0.20999999344348907,0.21700000762939453,0.22843749821186066
6 1 0.21250000596046448,0.21681249141693115,0.2278124988079071
0 2 0.21875,0.22218748927116394,0.23999999463558197
1 2 0.22218750417232513,0.2264999896287918,0.23624999821186066
2 2 0.21937499940395355,0.2253750115633011,0.2384375035762787
3 2 0.22218750417232513,0.22581247985363007,0.2356249988079071
4 2 0.2199999988079071,0.22756250202655792,0.23343749344348907
5 2 0.22156250476837158,0.22550001740455627,0.23375000059604645
6 2 0.21250000596046448,0.21849998831748962,0.23125000298023224
0 3 0.2228125035762787,0.22631248831748962,0.2459374964237213
1 3 0.22093750536441803,0.22693748772144318,0.23999999463558197
2 3 0.21968750655651093,0.2254999876022339,0.23906250298023224
3 3 0.21812500059604645,0.2241249978542328,0.23749999701976776
4 3 0.21406249701976776,0.22187499701976776,0.23749999701976776
5 3 0.21375000476837158,0.2186249941587448,0.23656250536441803
6 3 0.20874999463558197,0.21518750488758087,0.22875000536441803
0 4 0.21843749284744263,0.22462499141693115,0.23406249284744263
1 4 0.2199999988079071,0.22443750500679016,0.2381249964237213
2 4 0.2146874964237213,0.21968750655651093,0.23250000178813934
3 4 0.21250000596046448,0.22187499701976776,0.23749999701976776
4 4 0.21125000715255737,0.218562513589859,0.2290624976158142
5 4 0.21156249940395355,0.22087499499320984,0.23125000298023224
6 4 0.21156249940395355,0.21774999797344208,0.22499999403953552
0 5 0.21281249821186066,0.22093749046325684,0.23749999701976776
1 5 0.2109375,0.21812501549720764,0.22687500715255737
2 5 0.21281249821186066,0.22056250274181366,0.2278124988079071
3 5 0.2134374976158142,0.2173124998807907,0.22937500476837158
4 5 0.2109375,0.21787500381469727,0.2368749976158142
5 5 0.2121874988079071,0.2175624817609787,0.22374999523162842
6 5 0.20812499523162842,0.21431250870227814,0.22687500715255737
0 6 0.2109375,0.21937501430511475,0.2318750023841858
1 6 0.21125000715255737,0.21787500381469727,0.22875000536441803
2 6 0.2109375,0.21800000965595245,0.2303124964237213
3 6 0.21031250059604645,0.2173124998807907,0.2306250035762787
4 6 0.2146874964237213,0.2199374884366989,0.2306250035762787
5 6 0.21375000476837158,0.21962499618530273,0.22718749940395355
6 6 0.2121874988079071,0.21556250751018524,0.23000000417232513
0 7 0.21375000476837158,0.21849998831748962,0.22937500476837158
1 7 0.20999999344348907,0.2201250046491623,0.23250000178813934
2 7 0.20749999582767487,0.2160625010728836,0.23125000298023224
3 7 0.21250000596046448,0.2176874876022339,0.23468750715255737
4 7 0.21562500298023224,0.2176874876022339,0.23125000298023224
5 7 0.2134374976158142,0.2186249941587448,0.23375000059604645
6 7 0.21031250059604645,0.2147500067949295,0.22687500715255737
0 8 0.20999999344348907,0.2188125103712082,0.23593750596046448
1 8 0.2134374976158142,0.2175624817609787,0.23499999940395355
2 8 0.20999999344348907,0.21681252121925354,0.23000000417232513
3 8 0.20999999344348907,0.21575000882148743,0.22843749821186066
4 8 0.2109375,0.21781249344348907,0.23093749582767487
5 8 0.21187500655651093,0.2199999988079071,0.23156249523162842
6 8 0.21281249821186066,0.2188125103712082,0.2278124988079071
0 9 0.21375000476837158,0.21831250190734863,0.23156249523162842
1 9 0.2109375,0.2162499874830246,0.2303124964237213
2 9 0.2121874988079071,0.21681252121925354,0.2331250011920929
3 9 0.2071875035762787,0.2160000056028366,0.22687500715255737
4 9 0.20999999344348907,0.21681249141693115,0.2303124964237213
5 9 0.2134374976158142,0.21637502312660217,0.22843749821186066
6 9 0.2084375023841858,0.21800000965595245,0.2318750023841858
};

\nextgroupplot[title=max (92.7\%), 
yticklabels={}]
\addplot[
matrix plot,
mesh/cols=7,
mesh/color input=explicit,
]
table[meta=rgb] {
x y rgb
0 0 0.8128125071525574,0.8228124380111694,0.8340625166893005
1 0 0.8571875095367432,0.9345625042915344,0.9678124785423279
2 0 0.9721875190734863,0.9756875038146973,0.9840624928474426
3 0 0.981249988079071,0.9832500219345093,0.984375
4 0 0.9784374833106995,0.981499969959259,0.984375
5 0 0.9806249737739563,0.9824374914169312,0.9840624928474426
6 0 0.9753124713897705,0.9780000448226929,0.9806249737739563
0 1 0.8606250286102295,0.8761875033378601,0.8909375071525574
1 1 0.9446874856948853,0.9557499885559082,0.9637500047683716
2 1 0.9596874713897705,0.9655624628067017,0.9712499976158142
3 1 0.9731249809265137,0.9760624766349792,0.9818750023841858
4 1 0.9771875143051147,0.9789375066757202,0.9837499856948853
5 1 0.9793750047683716,0.981249988079071,0.9837499856948853
6 1 0.9696875214576721,0.979687511920929,0.9840624928474426
0 2 0.9103124737739563,0.9132499694824219,0.918749988079071
1 2 0.9137499928474426,0.9465624690055847,0.9887499809265137
2 2 0.9156249761581421,0.9234374761581421,0.9362499713897705
3 2 0.9106249809265137,0.9162500500679016,0.9284374713897705
4 2 0.9043750166893005,0.9078124761581421,0.9165624976158142
5 2 0.890625,0.8985625505447388,0.9068750143051147
6 2 0.9556249976158142,0.9627499580383301,0.9731249809265137
0 3 0.9168750047683716,0.9223750233650208,0.9337499737739563
1 3 0.9540625214576721,0.9696249961853027,0.9790624976158142
2 3 0.9234374761581421,0.9439999461174011,0.9806249737739563
3 3 0.9184374809265137,0.9206250309944153,0.9271875023841858
4 3 0.9084374904632568,0.9160000085830688,0.9256250262260437
5 3 0.8993750214576721,0.9053750038146973,0.9150000214576721
6 3 0.8893749713897705,0.895937442779541,0.9071875214576721
0 4 0.909375011920929,0.9141875505447388,0.9225000143051147
1 4 0.9134374856948853,0.9158750772476196,0.9181249737739563
2 4 0.9087499976158142,0.9151874780654907,0.9209374785423279
3 4 0.9106249809265137,0.9141250848770142,0.9228125214576721
4 4 0.9112499952316284,0.913312554359436,0.9256250262260437
5 4 0.8968750238418579,0.9030000567436218,0.9165624976158142
6 4 0.8915625214576721,0.8972499966621399,0.90625
0 5 0.9028124809265137,0.9081250429153442,0.9131249785423279
1 5 0.8962500095367432,0.9027500152587891,0.9103124737739563
2 5 0.8856250047683716,0.89781254529953,0.9056249856948853
3 5 0.8978124856948853,0.9060624837875366,0.9137499928474426
4 5 0.8968750238418579,0.90562504529953,0.9118750095367432
5 5 0.8828125,0.8944999575614929,0.90625
6 5 0.887499988079071,0.8959375619888306,0.9040625095367432
0 6 0.8793749809265137,0.8823124766349792,0.8943750262260437
1 6 0.8746874928474426,0.8854374885559082,0.9009374976158142
2 6 0.8634374737739563,0.8718125224113464,0.879687488079071
3 6 0.8653125166893005,0.872249960899353,0.8774999976158142
4 6 0.859375,0.8650625348091125,0.8725000023841858
5 6 0.8403124809265137,0.8506250381469727,0.8578125238418579
6 6 0.8259375095367432,0.8357499837875366,0.8509374856948853
0 7 0.4296875,0.49562495946884155,0.5821874737739563
1 7 0.5756250023841858,0.6498749852180481,0.714062511920929
2 7 0.6234375238418579,0.6813750267028809,0.7278125286102295
3 7 0.6403124928474426,0.6784374713897705,0.7262499928474426
4 7 0.6681249737739563,0.6888124346733093,0.714062511920929
5 7 0.5418750047683716,0.5807499885559082,0.6265624761581421
6 7 0.2631250023841858,0.4398749768733978,0.5496875047683716
0 8 0.21406249701976776,0.21962499618530273,0.23125000298023224
1 8 0.2199999988079071,0.2305624932050705,0.2549999952316284
2 8 0.2212499976158142,0.2432500123977661,0.2750000059604645
3 8 0.22093750536441803,0.23149999976158142,0.25562500953674316
4 8 0.21656249463558197,0.22225001454353333,0.23000000417232513
5 8 0.21312500536441803,0.21700000762939453,0.22968749701976776
6 8 0.21062499284744263,0.21849998831748962,0.2303124964237213
0 9 0.2121874988079071,0.21949999034404755,0.23250000178813934
1 9 0.21281249821186066,0.21893751621246338,0.2306250035762787
2 9 0.21437500417232513,0.21674999594688416,0.22968749701976776
3 9 0.21531249582767487,0.2175000011920929,0.22750000655651093
4 9 0.20937499403953552,0.21537499129772186,0.22562499344348907
5 9 0.2043749988079071,0.21418750286102295,0.2253125011920929
6 9 0.1940625011920929,0.20775000751018524,0.2150000035762787
};

\nextgroupplot[yticklabels={0.3,0.09,0.027,0.008,0.002,7.3e-4,2.2e-4,6.6e-5,2.0e-5,5.9e-6}, xticklabels={}] 
\addplot[
matrix plot,
mesh/cols=7,
mesh/color input=explicit,
]
table[meta=rgb] {
x y rgb
0 0 0.263671875,0.2775390625,0.2919921875
1 0 0.25390625,0.2681640625,0.291015625
2 0 0.2568359375,0.27265625,0.2978515625
3 0 0.2568359375,0.2712890625,0.30078125
4 0 0.2666015625,0.2763671875,0.298828125
5 0 0.2646484375,0.2716796875,0.3134765625
6 0 0.2705078125,0.28125,0.298828125
0 1 0.2646484375,0.2775390625,0.3076171875
1 1 0.2744140625,0.2875,0.2978515625
2 1 0.267578125,0.275390625,0.3037109375
3 1 0.2685546875,0.2759765625,0.2978515625
4 1 0.2685546875,0.278515625,0.3056640625
5 1 0.26953125,0.28125,0.31640625
6 1 0.267578125,0.28046875,0.2978515625
0 2 0.267578125,0.2751953125,0.298828125
1 2 0.2724609375,0.283984375,0.3076171875
2 2 0.267578125,0.280859375,0.3037109375
3 2 0.26953125,0.2818359375,0.3037109375
4 2 0.2646484375,0.2802734375,0.3154296875
5 2 0.2646484375,0.2767578125,0.3076171875
6 2 0.2685546875,0.2751953125,0.29296875
0 3 0.263671875,0.2751953125,0.3037109375
1 3 0.267578125,0.280859375,0.2978515625
2 3 0.265625,0.2798828125,0.3154296875
3 3 0.2666015625,0.281640625,0.3037109375
4 3 0.2666015625,0.2798828125,0.3076171875
5 3 0.2578125,0.273828125,0.306640625
6 3 0.2724609375,0.279296875,0.3046875
0 4 0.26953125,0.2880859375,0.3056640625
1 4 0.2685546875,0.2806640625,0.30859375
2 4 0.271484375,0.2794921875,0.30859375
3 4 0.279296875,0.2857421875,0.3046875
4 4 0.2734375,0.282421875,0.3037109375
5 4 0.26953125,0.27890625,0.302734375
6 4 0.2763671875,0.2890625,0.306640625
0 5 0.2783203125,0.2921875,0.3095703125
1 5 0.29296875,0.40859375,0.7099609375
2 5 0.6357421875,0.7111328125,0.8154296875
3 5 0.26953125,0.2833984375,0.2978515625
4 5 0.267578125,0.280859375,0.2998046875
5 5 0.271484375,0.2814453125,0.30078125
6 5 0.26953125,0.2779296875,0.2998046875
0 6 0.271484375,0.283203125,0.3037109375
1 6 0.298828125,0.3634765625,0.576171875
2 6 0.6240234375,0.7544921875,0.8896484375
3 6 0.35546875,0.767578125,0.900390625
4 6 0.2958984375,0.5646484375,0.8076171875
5 6 0.279296875,0.479296875,0.73828125
6 6 0.265625,0.283203125,0.3056640625
0 7 0.265625,0.27421875,0.2958984375
1 7 0.265625,0.2837890625,0.3125
2 7 0.2705078125,0.2802734375,0.306640625
3 7 0.2763671875,0.2857421875,0.3115234375
4 7 0.2744140625,0.2890625,0.3251953125
5 7 0.2734375,0.2822265625,0.3056640625
6 7 0.2783203125,0.2884765625,0.314453125
0 8 0.267578125,0.2775390625,0.3076171875
1 8 0.2724609375,0.282421875,0.296875
2 8 0.2705078125,0.2822265625,0.30078125
3 8 0.265625,0.2767578125,0.3037109375
4 8 0.2724609375,0.2833984375,0.3271484375
5 8 0.26953125,0.2796875,0.30078125
6 8 0.2734375,0.2818359375,0.322265625
0 9 0.2705078125,0.2765625,0.294921875
1 9 0.2685546875,0.2810546875,0.2958984375
2 9 0.2685546875,0.2779296875,0.2998046875
3 9 0.2626953125,0.278125,0.2998046875
4 9 0.2744140625,0.2822265625,0.298828125
5 9 0.2666015625,0.2783203125,0.3076171875
6 9 0.271484375,0.2814453125,0.3017578125
};
\addplot[
only marks,
mark=x,
]
table {
x y
3 6
};

\nextgroupplot[yticklabels={}, xticklabels={}] 
\addplot[
matrix plot,
mesh/cols=7,
mesh/color input=explicit,
]
table[meta=rgb] {
x y rgb
0 0 0.2607421875,0.271875,0.298828125
1 0 0.265625,0.275,0.2900390625
2 0 0.2705078125,0.27734375,0.2958984375
3 0 0.2529296875,0.262890625,0.2724609375
4 0 0.0517578125,0.1509765625,0.2626953125
5 0 0.115234375,0.173828125,0.279296875
6 0 0.072265625,0.1765625,0.279296875
0 1 0.26953125,0.2794921875,0.298828125
1 1 0.263671875,0.276171875,0.2939453125
2 1 0.2578125,0.2712890625,0.2900390625
3 1 0.265625,0.276953125,0.3017578125
4 1 0.19140625,0.24921875,0.28125
5 1 0.142578125,0.183984375,0.310546875
6 1 0.0556640625,0.16953125,0.2919921875
0 2 0.2802734375,0.3130859375,0.4384765625
1 2 0.2763671875,0.306640625,0.3779296875
2 2 0.2890625,0.3166015625,0.4384765625
3 2 0.271484375,0.3150390625,0.48046875
4 2 0.26953125,0.286328125,0.3134765625
5 2 0.2724609375,0.280078125,0.3173828125
6 2 0.25,0.265625,0.29296875
0 3 0.298828125,0.442578125,0.67578125
1 3 0.2958984375,0.488671875,0.626953125
2 3 0.2890625,0.371875,0.568359375
3 3 0.28125,0.3880859375,0.5380859375
4 3 0.2861328125,0.37890625,0.5166015625
5 3 0.2763671875,0.305859375,0.419921875
6 3 0.265625,0.2796875,0.302734375
0 4 0.2763671875,0.4083984375,0.6923828125
1 4 0.2802734375,0.48515625,0.6904296875
2 4 0.283203125,0.3876953125,0.6494140625
3 4 0.287109375,0.51171875,0.6884765625
4 4 0.2880859375,0.3697265625,0.5029296875
5 4 0.2734375,0.318359375,0.447265625
6 4 0.2685546875,0.276953125,0.302734375
0 5 0.2744140625,0.289453125,0.302734375
1 5 0.291015625,0.361328125,0.419921875
2 5 0.2822265625,0.315234375,0.3642578125
3 5 0.2890625,0.3080078125,0.34375
4 5 0.2802734375,0.2888671875,0.310546875
5 5 0.2744140625,0.283984375,0.3037109375
6 5 0.26171875,0.2748046875,0.2978515625
0 6 0.2744140625,0.280859375,0.302734375
1 6 0.275390625,0.287890625,0.3056640625
2 6 0.2705078125,0.28046875,0.3056640625
3 6 0.2734375,0.2818359375,0.302734375
4 6 0.2705078125,0.279296875,0.3056640625
5 6 0.2626953125,0.2787109375,0.318359375
6 6 0.263671875,0.2822265625,0.3046875
0 7 0.267578125,0.2810546875,0.30859375
1 7 0.2646484375,0.276953125,0.30859375
2 7 0.2685546875,0.28125,0.302734375
3 7 0.2685546875,0.2787109375,0.31640625
4 7 0.263671875,0.2759765625,0.3037109375
5 7 0.267578125,0.2787109375,0.3056640625
6 7 0.2685546875,0.2771484375,0.29296875
0 8 0.2666015625,0.2751953125,0.302734375
1 8 0.25390625,0.27109375,0.296875
2 8 0.263671875,0.2748046875,0.3076171875
3 8 0.263671875,0.275390625,0.2998046875
4 8 0.25390625,0.2705078125,0.3076171875
5 8 0.26171875,0.2732421875,0.3017578125
6 8 0.255859375,0.2740234375,0.2998046875
0 9 0.267578125,0.2765625,0.2998046875
1 9 0.2646484375,0.2794921875,0.3046875
2 9 0.259765625,0.27265625,0.314453125
3 9 0.2587890625,0.2712890625,0.3017578125
4 9 0.2724609375,0.276953125,0.296875
5 9 0.2587890625,0.2677734375,0.2998046875
6 9 0.255859375,0.269140625,0.2900390625
};
\addplot[
only marks,
mark=x,
]
table {
x y
3 4
};

\nextgroupplot[yticklabels={}, xticklabels={}] 
\addplot[
matrix plot,
mesh/cols=7,
mesh/color input=explicit,
]
table[meta=rgb] {
x y rgb
0 0 0.2607421875,0.279296875,0.294921875
1 0 0.271484375,0.2796875,0.3193359375
2 0 0.2685546875,0.2740234375,0.2890625
3 0 0.2626953125,0.3232421875,0.421875
4 0 0.2509765625,0.3408203125,0.3740234375
5 0 0.109375,0.18046875,0.37109375
6 0 0.091796875,0.174609375,0.2841796875
0 1 0.2666015625,0.2779296875,0.3125
1 1 0.2646484375,0.2755859375,0.2919921875
2 1 0.263671875,0.31171875,0.3740234375
3 1 0.265625,0.43515625,0.708984375
4 1 0.6513671875,0.6853515625,0.7529296875
5 1 0.267578125,0.4853515625,0.6533203125
6 1 0.2412109375,0.2615234375,0.2890625
0 2 0.2734375,0.294921875,0.3740234375
1 2 0.275390625,0.28828125,0.3232421875
2 2 0.271484375,0.301171875,0.35546875
3 2 0.2802734375,0.3951171875,0.63671875
4 2 0.27734375,0.425390625,0.666015625
5 2 0.2861328125,0.5046875,0.880859375
6 2 0.2724609375,0.36640625,0.6396484375
0 3 0.279296875,0.32265625,0.396484375
1 3 0.2822265625,0.2931640625,0.328125
2 3 0.2763671875,0.3083984375,0.373046875
3 3 0.2763671875,0.2873046875,0.30859375
4 3 0.279296875,0.299609375,0.369140625
5 3 0.2763671875,0.32421875,0.5078125
6 3 0.2890625,0.37109375,0.470703125
0 4 0.2861328125,0.3234375,0.3603515625
1 4 0.2978515625,0.3263671875,0.3818359375
2 4 0.27734375,0.3544921875,0.4833984375
3 4 0.2958984375,0.3423828125,0.427734375
4 4 0.2734375,0.29375,0.3544921875
5 4 0.271484375,0.2923828125,0.37890625
6 4 0.2734375,0.2984375,0.4140625
0 5 0.64453125,0.6646484375,0.7001953125
1 5 0.3681640625,0.5630859375,0.671875
2 5 0.3046875,0.544921875,0.8134765625
3 5 0.2978515625,0.61484375,0.884765625
4 5 0.27734375,0.491015625,0.8740234375
5 5 0.283203125,0.420703125,0.861328125
6 5 0.267578125,0.5232421875,0.904296875
0 6 0.69140625,0.70078125,0.7373046875
1 6 0.8046875,0.832421875,0.8642578125
2 6 0.8818359375,0.9083984375,0.935546875
3 6 0.6689453125,0.845703125,0.916015625
4 6 0.6455078125,0.827734375,0.9228515625
5 6 0.361328125,0.6908203125,0.8779296875
6 6 0.2919921875,0.5763671875,0.7939453125
0 7 0.56640625,0.618359375,0.6591796875
1 7 0.625,0.6607421875,0.7470703125
2 7 0.6142578125,0.7158203125,0.802734375
3 7 0.6162109375,0.6888671875,0.8125
4 7 0.6162109375,0.7173828125,0.8017578125
5 7 0.5810546875,0.63984375,0.7607421875
6 7 0.296875,0.4623046875,0.609375
0 8 0.26953125,0.28515625,0.318359375
1 8 0.271484375,0.2880859375,0.3115234375
2 8 0.2783203125,0.2859375,0.310546875
3 8 0.28125,0.287890625,0.310546875
4 8 0.271484375,0.2935546875,0.3115234375
5 8 0.2734375,0.282421875,0.3037109375
6 8 0.2646484375,0.287109375,0.3193359375
0 9 0.263671875,0.2802734375,0.3017578125
1 9 0.2685546875,0.2818359375,0.314453125
2 9 0.2744140625,0.2814453125,0.302734375
3 9 0.2734375,0.2841796875,0.30078125
4 9 0.26171875,0.277734375,0.3056640625
5 9 0.2626953125,0.2794921875,0.294921875
6 9 0.2666015625,0.276953125,0.298828125
};
\addplot[
only marks,
mark=x,
]
table {
x y
2 6
};

\nextgroupplot[yticklabels={}, xticklabels={}] 
\addplot[
matrix plot,
mesh/cols=7,
mesh/color input=explicit,
]
table[meta=rgb] {
x y rgb
0 0 0.259765625,0.2763671875,0.2978515625
1 0 0.2646484375,0.277734375,0.294921875
2 0 0.265625,0.280078125,0.2939453125
3 0 0.2490234375,0.2564453125,0.2861328125
4 0 0.2333984375,0.2634765625,0.2861328125
5 0 0.083984375,0.1552734375,0.2626953125
6 0 0.0888671875,0.1365234375,0.30078125
0 1 0.2646484375,0.2728515625,0.2958984375
1 1 0.259765625,0.282421875,0.314453125
2 1 0.2705078125,0.278515625,0.3076171875
3 1 0.267578125,0.2779296875,0.294921875
4 1 0.2626953125,0.2767578125,0.3271484375
5 1 0.2626953125,0.278515625,0.306640625
6 1 0.251953125,0.2626953125,0.2822265625
0 2 0.2705078125,0.28125,0.306640625
1 2 0.2724609375,0.2869140625,0.3046875
2 2 0.2763671875,0.2818359375,0.314453125
3 2 0.271484375,0.3115234375,0.4365234375
4 2 0.27734375,0.2869140625,0.3154296875
5 2 0.279296875,0.2955078125,0.3466796875
6 2 0.2705078125,0.2853515625,0.3095703125
0 3 0.2763671875,0.2857421875,0.3037109375
1 3 0.28515625,0.295703125,0.3095703125
2 3 0.2763671875,0.28515625,0.306640625
3 3 0.2822265625,0.2921875,0.318359375
4 3 0.2822265625,0.290625,0.326171875
5 3 0.275390625,0.283984375,0.30859375
6 3 0.267578125,0.2818359375,0.3037109375
0 4 0.2763671875,0.290234375,0.3203125
1 4 0.27734375,0.295703125,0.3486328125
2 4 0.279296875,0.2955078125,0.3212890625
3 4 0.2685546875,0.289453125,0.328125
4 4 0.2763671875,0.287890625,0.310546875
5 4 0.2763671875,0.2880859375,0.3046875
6 4 0.271484375,0.283203125,0.306640625
0 5 0.2763671875,0.28515625,0.31640625
1 5 0.2763671875,0.2896484375,0.310546875
2 5 0.26953125,0.2845703125,0.3046875
3 5 0.275390625,0.2857421875,0.3125
4 5 0.2705078125,0.28125,0.30078125
5 5 0.2744140625,0.2939453125,0.330078125
6 5 0.26953125,0.2873046875,0.3154296875
0 6 0.2744140625,0.2873046875,0.3076171875
1 6 0.2705078125,0.2783203125,0.2998046875
2 6 0.2744140625,0.2892578125,0.3125
3 6 0.2822265625,0.292578125,0.3134765625
4 6 0.2734375,0.2833984375,0.3037109375
5 6 0.271484375,0.2861328125,0.3115234375
6 6 0.2744140625,0.284375,0.3193359375
0 7 0.2705078125,0.2857421875,0.3115234375
1 7 0.2734375,0.2884765625,0.30859375
2 7 0.2783203125,0.29140625,0.31640625
3 7 0.2724609375,0.2861328125,0.3056640625
4 7 0.2734375,0.286328125,0.310546875
5 7 0.2705078125,0.28359375,0.3125
6 7 0.271484375,0.2818359375,0.30859375
0 8 0.271484375,0.2798828125,0.3076171875
1 8 0.2685546875,0.2751953125,0.298828125
2 8 0.2685546875,0.2818359375,0.3037109375
3 8 0.27734375,0.28828125,0.3232421875
4 8 0.2783203125,0.294921875,0.337890625
5 8 0.279296875,0.2939453125,0.33984375
6 8 0.2734375,0.2845703125,0.3154296875
0 9 0.2705078125,0.27578125,0.2939453125
1 9 0.265625,0.278515625,0.2958984375
2 9 0.2705078125,0.2787109375,0.3076171875
3 9 0.2685546875,0.28359375,0.3203125
4 9 0.275390625,0.2833984375,0.30859375
5 9 0.2724609375,0.2876953125,0.30859375
6 9 0.2763671875,0.2814453125,0.30078125
};
\addplot[
only marks,
mark=x,
]
table {
x y
3 2
};

\nextgroupplot[yticklabels={}, xticklabels={}] 
\addplot[
matrix plot,
mesh/cols=7,
mesh/color input=explicit,
]
table[meta=rgb] {
x y rgb
0 0 0.2578125,0.272265625,0.3037109375
1 0 0.2568359375,0.275390625,0.298828125
2 0 0.26171875,0.2798828125,0.291015625
3 0 0.259765625,0.2748046875,0.29296875
4 0 0.2646484375,0.2765625,0.296875
5 0 0.2607421875,0.2744140625,0.302734375
6 0 0.248046875,0.265234375,0.296875
0 1 0.2744140625,0.27734375,0.296875
1 1 0.2802734375,0.2890625,0.3076171875
2 1 0.2744140625,0.283203125,0.3095703125
3 1 0.2646484375,0.2841796875,0.3037109375
4 1 0.263671875,0.280859375,0.30078125
5 1 0.263671875,0.2751953125,0.296875
6 1 0.2666015625,0.2759765625,0.3017578125
0 2 0.271484375,0.2814453125,0.298828125
1 2 0.26953125,0.2783203125,0.30859375
2 2 0.2705078125,0.28671875,0.30859375
3 2 0.2783203125,0.28828125,0.306640625
4 2 0.2705078125,0.2859375,0.3056640625
5 2 0.283203125,0.2923828125,0.318359375
6 2 0.2724609375,0.2822265625,0.3154296875
0 3 0.2744140625,0.28203125,0.3369140625
1 3 0.275390625,0.2865234375,0.310546875
2 3 0.2783203125,0.287109375,0.30859375
3 3 0.2705078125,0.2791015625,0.3056640625
4 3 0.2705078125,0.27890625,0.3115234375
5 3 0.2607421875,0.2765625,0.2998046875
6 3 0.2685546875,0.2798828125,0.2958984375
0 4 0.271484375,0.2841796875,0.314453125
1 4 0.2705078125,0.28203125,0.3076171875
2 4 0.2666015625,0.276953125,0.314453125
3 4 0.2734375,0.2791015625,0.298828125
4 4 0.271484375,0.2857421875,0.302734375
5 4 0.2724609375,0.2779296875,0.302734375
6 4 0.2626953125,0.2744140625,0.2998046875
0 5 0.2763671875,0.2818359375,0.3056640625
1 5 0.263671875,0.280078125,0.3056640625
2 5 0.2705078125,0.2853515625,0.3046875
3 5 0.26953125,0.273828125,0.294921875
4 5 0.2724609375,0.28046875,0.294921875
5 5 0.2607421875,0.2755859375,0.302734375
6 5 0.26171875,0.275,0.2998046875
0 6 0.2646484375,0.274609375,0.298828125
1 6 0.2666015625,0.2787109375,0.3046875
2 6 0.26953125,0.28203125,0.3212890625
3 6 0.2646484375,0.2765625,0.3017578125
4 6 0.26953125,0.283203125,0.3046875
5 6 0.267578125,0.275,0.3115234375
6 6 0.2646484375,0.27421875,0.29296875
0 7 0.2646484375,0.2826171875,0.3212890625
1 7 0.265625,0.2771484375,0.306640625
2 7 0.2587890625,0.2701171875,0.3076171875
3 7 0.2705078125,0.2748046875,0.296875
4 7 0.271484375,0.277734375,0.302734375
5 7 0.2607421875,0.2705078125,0.2978515625
6 7 0.2626953125,0.275,0.3125
0 8 0.2705078125,0.282421875,0.310546875
1 8 0.2685546875,0.275390625,0.3037109375
2 8 0.2646484375,0.2783203125,0.3017578125
3 8 0.25390625,0.271875,0.294921875
4 8 0.2578125,0.2685546875,0.3076171875
5 8 0.1767578125,0.249609375,0.296875
6 8 0.25390625,0.266796875,0.30078125
0 9 0.2666015625,0.2841796875,0.302734375
1 9 0.2666015625,0.2861328125,0.3076171875
2 9 0.25,0.273046875,0.3056640625
3 9 0.189453125,0.2517578125,0.3076171875
4 9 0.2587890625,0.278125,0.302734375
5 9 0.1865234375,0.2419921875,0.2919921875
6 9 0.1416015625,0.2205078125,0.2919921875
};
\addplot[
only marks,
mark=x,
]
table {
x y
5 2
};

\nextgroupplot[yticklabels={}, xticklabels={}] 
\addplot[
matrix plot,
mesh/cols=7,
mesh/color input=explicit,
]
table[meta=rgb] {
x y rgb
0 0 0.455078125,0.5037109375,0.5751953125
1 0 0.6611328125,0.81328125,0.8994140625
2 0 0.8994140625,0.91796875,0.9375
3 0 0.9140625,0.9271484375,0.9384765625
4 0 0.90234375,0.9158203125,0.9326171875
5 0 0.9111328125,0.9232421875,0.9404296875
6 0 0.8955078125,0.9140625,0.931640625
0 1 0.4228515625,0.5556640625,0.6962890625
1 1 0.8251953125,0.847265625,0.8720703125
2 1 0.841796875,0.8763671875,0.900390625
3 1 0.9013671875,0.9103515625,0.92578125
4 1 0.9091796875,0.9203125,0.9462890625
5 1 0.9130859375,0.922265625,0.93359375
6 1 0.9091796875,0.9263671875,0.947265625
0 2 0.6376953125,0.669140625,0.697265625
1 2 0.6552734375,0.762109375,0.9384765625
2 2 0.5673828125,0.64921875,0.73828125
3 2 0.6142578125,0.6865234375,0.74609375
4 2 0.6796875,0.6974609375,0.7138671875
5 2 0.6552734375,0.6755859375,0.701171875
6 2 0.798828125,0.8576171875,0.8984375
0 3 0.7333984375,0.7509765625,0.78125
1 3 0.8154296875,0.8560546875,0.9091796875
2 3 0.6962890625,0.759765625,0.88671875
3 3 0.6982421875,0.711328125,0.732421875
4 3 0.681640625,0.703125,0.7333984375
5 3 0.6845703125,0.697265625,0.7158203125
6 3 0.673828125,0.687109375,0.7021484375
0 4 0.6845703125,0.6966796875,0.71875
1 4 0.6943359375,0.705859375,0.7392578125
2 4 0.6845703125,0.6947265625,0.7255859375
3 4 0.701171875,0.70703125,0.7236328125
4 4 0.6884765625,0.69765625,0.7275390625
5 4 0.6708984375,0.6880859375,0.7177734375
6 4 0.671875,0.680859375,0.7080078125
0 5 0.6748046875,0.69609375,0.7197265625
1 5 0.6630859375,0.68046875,0.701171875
2 5 0.66015625,0.673828125,0.701171875
3 5 0.671875,0.6826171875,0.7021484375
4 5 0.6689453125,0.6787109375,0.6982421875
5 5 0.65625,0.674609375,0.6943359375
6 5 0.658203125,0.6716796875,0.7041015625
0 6 0.6376953125,0.658984375,0.689453125
1 6 0.6455078125,0.6568359375,0.6845703125
2 6 0.62109375,0.6333984375,0.66015625
3 6 0.6318359375,0.6421875,0.66015625
4 6 0.61328125,0.628125,0.65625
5 6 0.599609375,0.62109375,0.6376953125
6 6 0.603515625,0.614453125,0.6455078125
0 7 0.3916015625,0.4388671875,0.4775390625
1 7 0.4892578125,0.50625,0.56640625
2 7 0.4833984375,0.512109375,0.5634765625
3 7 0.4765625,0.49453125,0.5439453125
4 7 0.490234375,0.51015625,0.5361328125
5 7 0.44140625,0.461328125,0.501953125
6 7 0.298828125,0.394140625,0.4609375
0 8 0.2626953125,0.2765625,0.302734375
1 8 0.2763671875,0.2919921875,0.314453125
2 8 0.2841796875,0.301171875,0.3330078125
3 8 0.26953125,0.2806640625,0.302734375
4 8 0.2587890625,0.2759765625,0.3115234375
5 8 0.2666015625,0.2748046875,0.30859375
6 8 0.267578125,0.2763671875,0.2978515625
0 9 0.2666015625,0.28046875,0.296875
1 9 0.2685546875,0.2783203125,0.3076171875
2 9 0.265625,0.275390625,0.3046875
3 9 0.26953125,0.2783203125,0.2958984375
4 9 0.2607421875,0.271484375,0.2958984375
5 9 0.2626953125,0.270703125,0.279296875
6 9 0.2353515625,0.2400390625,0.26953125
};
\addplot[
only marks,
mark=x,
]
table {
x y
3 0
};

\end{groupplot}
\end{tikzpicture}